\documentclass{article}

\usepackage{microtype}
\usepackage{graphicx}
\usepackage{subcaption}
\usepackage{booktabs} % for professional tables

 \usepackage{longtable,array}

\usepackage{hyperref}

\usepackage[preprint]{icml2026}

\usepackage{amsmath}
\usepackage{amssymb}
\usepackage{mathtools}
\usepackage{amsthm}

\usepackage[utf8]{inputenc} % allow utf-8 input
\usepackage[T1]{fontenc}    % use 8-bit T1 fonts
\usepackage{url}            % simple URL typesetting
\usepackage{nicefrac}       % compact symbols for 1/2, etc.
\usepackage{xcolor}         % colors

\usepackage{array}
\usepackage{amsfonts}
\usepackage{thmtools, thm-restate}
\usepackage{algorithm, algorithmic}
\usepackage{mathalfa}

\usepackage[capitalize,noabbrev]{cleveref}

\theoremstyle{plain}
\newtheorem{theorem}{Theorem}[section]

\newtheorem{lemma}[theorem]{Lemma}
\newtheorem{corollary}[theorem]{Corollary}
\theoremstyle{definition}
\newtheorem{definition}[theorem]{Definition}
\newtheorem{assumption}[theorem]{Assumption}
\theoremstyle{remark}

\DeclareMathAlphabet{\dutchcal}{U}{dutchcal}{m}{n}

\newcommand\norm[1]{\lVert#1\rVert}

\newcommand\set[1]{\left\{ #1 \right\}}

\DeclareMathOperator{\E}{\mathbb{E}}

\usepackage[textsize=tiny]{todonotes}

\icmltitlerunning{Decoupled Continuous-Time Actor–Critic via Hamiltonian Value Flow}

\begin{document}

\twocolumn[
    \icmltitle{Decoupled Continuous-Time Reinforcement Learning via Hamiltonian Flow}



  \icmlsetsymbol{equal}{*}

  \begin{icmlauthorlist}
    \icmlauthor{Minh Nguyen}{comp,sch}
  \end{icmlauthorlist}

  \icmlaffiliation{comp}{Google}
  \icmlaffiliation{sch}{University of Texas at Austin}

  \icmlcorrespondingauthor{Minh Nguyen}{minhpnguyen@utexas.edu}
  \icmlkeywords{Machine Learning, ICML}

  \vskip 0.3in
]



\printAffiliationsAndNotice{}  % no special notice (required even if empty)

\begin{abstract}
Many real-world control problems, ranging from finance to robotics, evolve in continuous time with non-uniform, event-driven decisions. Standard discrete-time reinforcement learning (RL), based on fixed-step Bellman updates, struggles in this setting: as time gaps shrink, the $Q$-function collapses to the value function $V$, eliminating action ranking. Existing continuous-time methods reintroduce action information via an advantage-rate function $q$. However, they enforce optimality through complicated martingale losses or orthogonality constraints, which are sensitive to the choice of test processes. These approaches entangle $V$ and $q$ into a large, complex optimization problem that is difficult to train reliably. To address these limitations, we propose a novel decoupled continuous-time actor-critic algorithm with alternating updates: $q$ is learned from diffusion generators on $V$, and $V$ is updated via a Hamiltonian-based value flow that remains informative under infinitesimal time steps, where standard max/softmax backups fail. Theoretically, we prove rigorous convergence via new probabilistic arguments, sidestepping the challenge that generator-based Hamiltonians lack Bellman-style contraction under the sup-norm. Empirically, our method outperforms prior continuous-time and leading discrete-time baselines across challenging continuous-control benchmarks and a real-world trading task, achieving 21\% profit over a single quarter$-$nearly doubling the second-best method.
\end{abstract}

\section{Introduction}\label{sec:intro}
Many learning and control problems in scientific computing, robotics, and finance evolve in continuous time with continuous state and action spaces. In markets, decisions are often made only when new signals arrive or positions need rebalancing, so holding times vary even when prices are observed per minute. Robots face the same pattern: often they can move with coarse updates, but when terrain becomes unstable or contacts change, they need small corrective steps to keep balance and recover. Similarly, autonomous vehicles take larger control steps on empty roads but update more frequently near intersections or in dense traffic. Across these domains, we need RL methods that learn reliably with both irregular decision times and small step sizes.

Standard reinforcement learning algorithms, however, are built on a discrete-time Markov decision process (MDP) abstraction with a fixed step size $\Delta t$. This mismatch creates practical and theoretical issues when decisions occur at irregular times and when the data mixes multiple time scales. With large $\Delta t$, the model can miss important dynamics between decisions; with very small $\Delta t$, Bellman targets change little from one step to the next, weakening learning signals and making training sensitive to noise. More fundamentally, the fixed-step Bellman recursion does not naturally account for non-uniform holding times, so the resulting updates can depend heavily on the choice of discretization rather than the underlying continuous-time control problem.

To address these issues, recent continuous-time RL work recasts learning as stochastic control. Building on the controlled-diffusion formulation of \citet{wang2020}, \citet{Jia2022-yg} introduce the advantage-rate function $q(x,a)$, which restores state--action discrimination lost when the discrete-time $Q$-function collapses to $V$ as $\Delta t\to 0$. This $q$-function underlies both value-based methods with martingale constraints \citep{Jia2022-yg} and policy-gradient variants such as entropy-regularized PPO (CPPO) \citep{CPPO}. These formulations provide a foundation for learning under irregular decision times and small timesteps, without relying on a fixed discretization.

However, a central technical challenge remains. Under diffusion dynamics, minimizing the usual squared temporal-difference (TD) error does not directly enforce the dynamic programming principle: in continuous time, such losses can primarily control variance terms induced by Brownian noise rather than true Bellman errors. Prior work addresses this by imposing martingale-based objectives or martingale orthogonality conditions that characterize the optimal solution \citep{Jia2022-yg}. While theoretically motivated, these modifications make optimization substantially harder in practice. Martingale objectives couple $q$ and $V$ into a single functional optimization problem, which is difficult to simplify into standard TD-style updates and can be challenging to optimize with modern function approximation and off-policy training. This motivates seeking continuous-time learning rules that keep the stochastic-control foundation, but admit a more modular and scalable optimization structure.

In this paper, we resolve these difficulties by \textbf{decoupling} the learning of $q$ and $V$, replacing a single coupled objective with a sequence of simple updates. With $V$ fixed, the correct $q$ is locally determined by the controlled generator. It can then be learned from trajectory data via standard regression, treating diffusion noise as estimation noise that averages out over minibatches. The remaining challenge is updating $V$ from an \textbf{instantaneous rate} $q$. Unlike discrete-time RL, where one can apply a backup such as $V(x)=\max_a Q(x,a)$ (or a soft analogue), this step becomes ill-posed in continuous time. We instead view the update equation as defining a \textbf{flow} over value functions. This leads to an explicit Picard-style update driven by the entropy-regularized Hamiltonian, with a step size that indexes \textbf{algorithmic progress} rather than an environment discretization. We establish convergence of this value-flow iteration under broad conditions. We further obtain a scalable actor--critic implementation through the simple critic decomposition $Q \approx V+q$, keeping the method practical while staying aligned with the continuous-time stochastic control formulation.

\textbf{Contribution.} Our contributions are threefold.

(1) We propose a continuous-time RL framework that separates learning the advantage-rate signal from updating the value function, replacing martingale-coupled objectives with a simpler iterative scheme. This leads to a practical implementation compatible with standard off-policy training.

(2) We prove rigorous convergence guarantees using a probabilistic analysis technique that avoids the heavy functional-analytic machinery typically required for generator-based operators. This approach also enables us to control the approximation error directly from sampled stochastic trajectories in continuous-time settings.

(3) We provide an empirical study beyond fixed-step simulated SDE benchmarks, evaluating continuous-time RL on high-dimensional control tasks with irregular decision times and a minute-based trading environment. Our method achieves consistently stronger performance than existing continuous-time approaches and discrete-time baselines.

\textbf{Organization.} In \cref{sec:preliminaries}, we introduce the continuous-time reinforcement learning formulation, together with key terminologies needed in our framework setup. In \cref{sec:formulation}, we discuss limitations of prior martingale-based approaches and present our novel decoupled framework together with the resulting actor--critic algorithms. In \cref{sec:theory}, we establish theoretical and convergence guarantees for the value function, the advantage-rate function, and the full algorithm. Finally, \cref{sec:experiments} evaluates our methods on continuous-control benchmarks and a real-world trading task under small and irregular time steps. Extended related work appears in Appendix~\ref{sec:appendix:related_works}. Background material is summarized in \ref{sec:appendix:preliminaries}. Full proofs are deferred to Appendices~\ref{sec:appendix:value_convergence}, \ref{sec:appendix:q_convergence}, \ref{sec:appendix:algorithm_convergence}, \ref{sec:appendix:random_time}, and \ref{sec:appendix:regret_bound}. Additional experimental details and reproducibility information are provided in Appendices~\ref{sec:appendix:experiment_details} and \ref{sec:appendix:further_experiment_details}.

\section{Preliminaries}\label{sec:preliminaries}

\subsection{Discrete-time MDPs to continuous-time control}\label{sec:prelim:mdp-to-sde}
The usual Markov decision process (MDP) discretized with a fixed time step $\Delta>0$ has the form:
\begin{equation}
  X_{t+\Delta}^{\Delta}
  =
  X_t^{\Delta}
  +
  f^{\Delta}(X_t^{\Delta}, a_t),
  \qquad
  a_t \sim \pi_{\alpha}(\cdot\mid X_t^{\Delta})
  \label{eq:discrete-transition}
\end{equation}
where $t\in\{0,\Delta,2\Delta,\dots\}$, $X_t^{\Delta}\in\mathbb{R}^d$, and $\pi_{\alpha}$ denotes an
exploration policy. Through classical limit arguments, \citet{munos06b} shows that if the increment is approximately linear in $\Delta$, i.e, $f^{\Delta}(x,a) = f(x,a)\,\Delta + o(\Delta)$, then, as $\Delta\to 0$, the discrete process converges to a deterministic trajectory $x_t$ governed by the ODE:
\begin{equation}
  \dot{x}_t = f_{\alpha}(x_t),
  \qquad
  f_{\alpha}(x) := \mathbb{E}_{a\sim\pi_{\alpha}(\cdot\mid x)}\big[f(x,a)\big]
  \label{eq:ode-limit}
\end{equation}

where action randomization averages into the drift. To model stochasticity beyond randomized actions, we augment the transition with a mean-zero perturbation,
\begin{equation}
  X_{t+\Delta}^{\Delta}
  =
  X_t^{\Delta}
  +
  f^{\Delta}(X_t^{\Delta}, a_t)
  +
  \xi_t^{\Delta}
  \label{eq:discrete-transition-noise}
\end{equation}
Under diffusion scaling (e.g., $\mathrm{Var}(\xi_t^{\Delta})=\mathcal{O}(\Delta)$), the accumulated fluctuations converge to Brownian motion (Donsker’s invariance principle; \citealp{Fleming2006-wa}), and the continuous-time limit becomes a controlled diffusion. This motivates the Itô SDE model used throughout the paper (formalized next; details in Appendix~\ref{sec:appendix:further_theoretical_details}).

\subsection{Stochastic control formulation}\label{sec:preliminaries:sde}
Motivated by the diffusion limit above, we model continuous-time RL as controlled stochastic dynamics: the state process $(X_t)_{t \ge 0} \subset \mathbb{R}^d$ evolves according to It\^o SDE:
\begin{equation}
  dX_t = b(X_t, a_t)\,dt + \sigma(X_t,a_t)\,dW_t,
  \qquad X_0 \sim \mu
  \label{eq:sde-physical}
\end{equation}
where $W$ is an $n$-dimensional Brownian motion, $a_t\in\mathcal{A}$ is the control (action), and
$b:\mathbb{R}^d\times\mathcal{A}\to\mathbb{R}^d, \,\sigma:\mathbb{R}^d\times\mathcal{A}\to\mathbb{R}^{d\times n}$ are measurable drift/diffusion functions. We restrict to Markov feedback policies $a_t \sim \pi(\cdot\mid X_t)$, so exploration is implemented through randomized control.

Given discount $\beta>0$, running reward $r:\mathbb{R}^d\times\mathcal{A}\to\mathbb{R}$, and temperature $\alpha \ge 0$, we consider the entropy-regularized objective:
\begin{equation}
\begin{aligned}
  V^{(\alpha)}(x; \pi) := &\mathbb{E}\Bigg[
  \int_0^\infty e^{-\beta t} \Big(r(X_t,a_t) \\
  &\quad - \alpha \log \pi(a_t\mid X_t) \Big)\,dt\,\Big|\, X_0 = x \Bigg]
\end{aligned}
  \label{eq:reg-objective-physical}
\end{equation}
with the optimal (entropy-regularized) value function $V^{(\alpha)}(x) := \sup_{\pi} V^{(\alpha)}(x;\pi)$.

\subsection{Generators and Hamiltonians}\label{sec:preliminaries:hamiltonian}
Two standard objects from stochastic control are central to our framework: the controlled generator and the (hard/soft) Hamiltonian. For a fixed action $a\in\mathcal{A}$, define the controlled (infinitesimal) generator $L^a$ acts on $\varphi\in C_b^2(\mathbb{R}^d)$ as:
\begin{equation}
\begin{aligned}
  (L^a \varphi)(x)
  := &\ b(x,a)\cdot\nabla \varphi(x) \\
   &+ \frac{1}{2}\,\mathrm{Tr}\!\bigl(
      \sigma(x,a)\sigma(x,a)^\top \nabla^2\varphi(x)
     \bigr)
\end{aligned}
  \label{eq:generator-hard}
\end{equation}
We view $L^a \varphi$ as continuous-time analogue of one-step Bellman drift. Using $L^a$, define the action-wise Hamiltonian:
\begin{equation}
  H_a\bigl(V\bigr)(x)
  :=
  (L^a V)(x)
  + r(x,a) - \beta V(x)
  \label{eq:Ha-soft}
\end{equation}

The corresponding hard Hamiltonian ($\alpha=0$) is:
\begin{equation}
  H^{(0)}(V)(x)
  :=
  \sup_{a\in\mathcal{A}} H_a\bigl(V\bigr)(x)
  \label{eq:hard-hamiltonian}
\end{equation}
and the soft (entropy-regularized) Hamiltonian for $\alpha>0$ is:
\begin{equation}
H^{(\alpha)}(V)(x)
:= \alpha \log \int_{\mathcal{A}} \exp\Big(
\frac{1}{\alpha} H_a\bigl(V\bigr)(x) \Big)\,da
\label{eq:soft-hamiltonian}
\end{equation}

As $\alpha\to 0$, $H^{(\alpha)}$ recovers $H^{(0)}$. We include the discount term $-\beta V(x)$ inside the Hamiltonian for a compact form. Under standard assumptions, the optimal value function $V^{(\alpha)}$ satisfies the HJB equation \cite{Fleming2006-wa, tang2021}:
\begin{equation}
  H^{(\alpha)}(V^{(\alpha)})(x) = 0,
  \qquad x\in\mathbb{R}^d
  \label{eq:hjb}
\end{equation}

\subsection{Instantaneous $q$-function in continuous time}\label{sec:preliminaries:q}
In discrete-time MDPs, the state--action value function $Q$ provides a direct action-selection signal, e.g.,
through $V(x)=\max_a Q(x,a)$. In continuous time, however, the standard $Q$-function collapses to $V$ as the time step shrinks to zero, i.e, $\lim_{\Delta \to 0} Q^{\Delta}(x, a) = V(x)$, regardless of action $a$. To restore action dependence, \citet{Jia2022-yg} introduce the \emph{advantage rate} (instantaneous) $q$-function. For a policy $\pi$ with value function $V^{(\alpha)}(\cdot;\pi)$, define:
\begin{equation}
  q^{(\alpha)}(x,a;\pi)
  :=
  \lim_{u\to 0}
  \frac{Q_{u}^{(\alpha)}(x,a;\pi) - V^{(\alpha)}(x;\pi)}{u} 
\end{equation}
where $Q_{u}^{(\alpha)}(x,a;\pi)$ is the return obtained by executing action $a$ for a short duration $u$ starting from state $x$, and then following $\pi$ afterwards. With $V = V^{(\alpha)}(:; \pi)$, standard small-time expansion yields:
\begin{equation}
Q_{u}^{(\alpha)}(x,a;\pi)= V(x) + H_a(V)(x)\,u + o(u)
\label{eq:small-time-Q-expansion}
\end{equation}
As a result, $q$-function can be also expressed as:
\begin{equation}
  q_V(x, a) = q^{(\alpha)}(x,a;\pi) = H_a(V)(x)
  \label{eq:q-def}
\end{equation}
Unlike the degenerate continuous-time limit of $Q$, the function $q(x,a)$ retains action dependence and provides the correct local improvement signal.

\textbf{Notations.} From here on, we write $V^{(\alpha)}$ for the \textbf{optimal} value function, and use $V$ for a generic value function (e.g., under a non-optimal policy or an approximation) in place of $V^{(\alpha)}(x; \pi)$. By \cref{eq:q-def}, the instantaneous $q$-function is determined by the corresponding $V$, so we write $q = q_V$ instead of $q^{(\alpha)}(x,a;\pi)$ so that $q_V(x, a) = H_a(V)(x)$. In particular, the optimal $q$-function is denoted as $q_{V^{(\alpha)}}$.

\section{Formulation}\label{sec:formulation}
\subsection{Challenges in continuous-time RL}\label{sec:formulation:challenges}
The instantaneous $q$-function introduced in \cref{sec:preliminaries:q} resolves the degeneracy of the
discrete-time $Q$-function in the infinitesimal limit. A second obstacle remains: under diffusion dynamics, the usual squared TD regression is no longer aligned with the continuous-time dynamic programming principle. Consider the TD error with step size $u>0$,
\begin{equation}
  \delta_u
  \approx
  e^{-\beta u} V(X_{t+u})
  + u\,r(X_t,a_t)
  - V(X_t)
  \label{eq:td-error-discrete}
\end{equation}
In a discrete-time MDP, minimizing $\E[\delta_u^2]$ enforces Bellman consistency. In continuous time,
however, the discounted value increment contains not only a drift term but also a stochastic martingale
term. In particular, with It\^o lemma, rewriting $\delta_u$ in differential form as $u\to 0$ yields:
\begin{equation}
\begin{aligned}
 & d\,\big(e^{-\beta t}V(X_t)\big) + e^{-\beta t} r(X_t, a_t)dt \\
 & =
  e^{-\beta t}\Big(
    (L^{a_t}V)(X_t) - \beta V(X_t) + r(X_t, a_t)
  \Big)\,dt \\
  & \qquad \qquad \qquad + e^{-\beta t}\nabla V(X_t)\,\sigma(X_t,a_t)\,dW_t
\end{aligned}
  \label{eq:ito-decomp}
\end{equation}
As $u\to 0$, these two terms scale differently: the drift is $O(u)$, while the stochastic increment is
$O(\sqrt{u})$. Hence, $\delta_u^2$ is dominated by martingale variation rather than the Bellman drift. To handle this mismatch, rather than simply reuse discrete-time TD updates, \citet{Jia2022-yg} characterize continuous-time RL through martingales.
More specifically, for $V$ and $q$ to be optimal, the process
\begin{equation}
M_t :=
e^{-\beta t}V(X_t)
+ \int_{t_0}^t e^{-\beta u}\big(r(X_u,a_u)-q(X_u,a_u)\big)\,du
\label{eq:martingale-characterization}
\end{equation}
must be a martingale for any starting time $t_0$.

A direct way to enforce this is a martingale-loss objective, which leads to a nested long-horizon path
functional. To avoid this complexity, \citet{Jia2022-yg} instead propose enforcing martingality through orthogonality tests, i.e., requiring $\E\!\left[\int_0^T \omega_t\, dM_t\right]=0$ for suitable filtration-adapted test processes $\omega_t$. In theory,
this condition holds for any $\omega_t$; in practice, \citet{Jia2022-yg} note that numerical performance
requires a careful choice of test family, and common gradient-based constructions are largely heuristic.
Both martingale-loss and orthogonality enforcement then couple $V$ and $q$ into a single optimization
problem, making them hard to reduce to standard TD-style updates. Explicit details of these approaches are given in Appendix~\ref{sec:appendix:benchmarks}.

Beyond these value-based martingale formulations, continuous-time policy-gradient methods such as CPPO \citep{CPPO} are among the few alternatives, but they also rely on the same martingale-enforcement formulas. This leaves the current continuous-time RL approaches constrained by a common bottleneck, which is reflected in our experiments: even in settings with irregular timesteps, where continuous-time methods should be more suitable, these approaches still underperform their
discrete-time counterparts (see \cref{sec:experiments}).

\subsection{Continuous-time actor--critic via Hamiltonian flow}\label{sec:formulation:decoupled}
Instead of enforcing martingale constraints through a single coupled optimization over $(V,q)$, we break learning into a sequence of smaller, tractable subproblems that iteratively update $q$ and $V$:
\begin{equation}
  V_0 \;\Rightarrow\; q_0 \;\Rightarrow\; V_1 \;\Rightarrow\; q_1 \;\Rightarrow\; \cdots
  \label{eq:decoupled-iteration}
\end{equation}
The guiding principle is simple: \emph{learning $q$ is easy if $V$ is fixed}, and the remaining difficulty is \emph{updating $V$ from an instantaneous rate} without degeneracy.

\paragraph{Step 1: Updating $q$ from a fixed $V$.}
For a fixed value function $V$, from \cref{eq:q-def}, the $q$-function can be determined locally by the controlled diffusion generator $L^a$:
\begin{equation}
  q_V(x,a) = r(x,a) + (L^aV)(x) - \beta V(x).
  \label{eq:q-control}
\end{equation}
This characterization is exact, but it is not directly model-free since $(L^aV)(x)$ depends on unknown dynamics.

\begin{algorithm}[t]
  \caption{Continuous-time soft actor-critic (CT-SAC)}
  \label{algo:ct_sac}
  \begin{algorithmic}[1]
    \STATE \textbf{Input:} discount $\beta$, temperature $\alpha$, Euler step $\tau>0$.
    \STATE Initialize parameters $\theta$ and $\phi$ for $Q_\theta$ and policy $\pi_{\phi}$.
    \FOR{$k = 0,1,2,\dots$}
      \STATE Collect transitions $(x, a, r, x', u)$, where $x = X_t$, $u$ is the time step increment, $x' = X_{t+u}$, $a = a_t$ is the action, and $r = r(X_t, a_t)$ from rollouts under current policy $\pi_\phi$ and store in an off-policy buffer. $u$ can vary across samples.
      \STATE \textbf{Update $Q_{\theta} = q_{\theta} + V_{\theta}$:} Fit $Q_{\theta}$ using mini-batch samples $(x, a, r, x', u)$ with targets being the left hand-side of \cref{eq:main_Q_update}.
      \STATE \textbf{Update $\pi_{\phi}$}: towards $\pi_\phi(a\mid x) \propto
        \exp\big(Q_{\theta}(x,a)/\alpha\big)$
    \ENDFOR
  \end{algorithmic}
\end{algorithm}

\begin{figure*}[ht]
    \centering
    \includegraphics[scale=0.45]{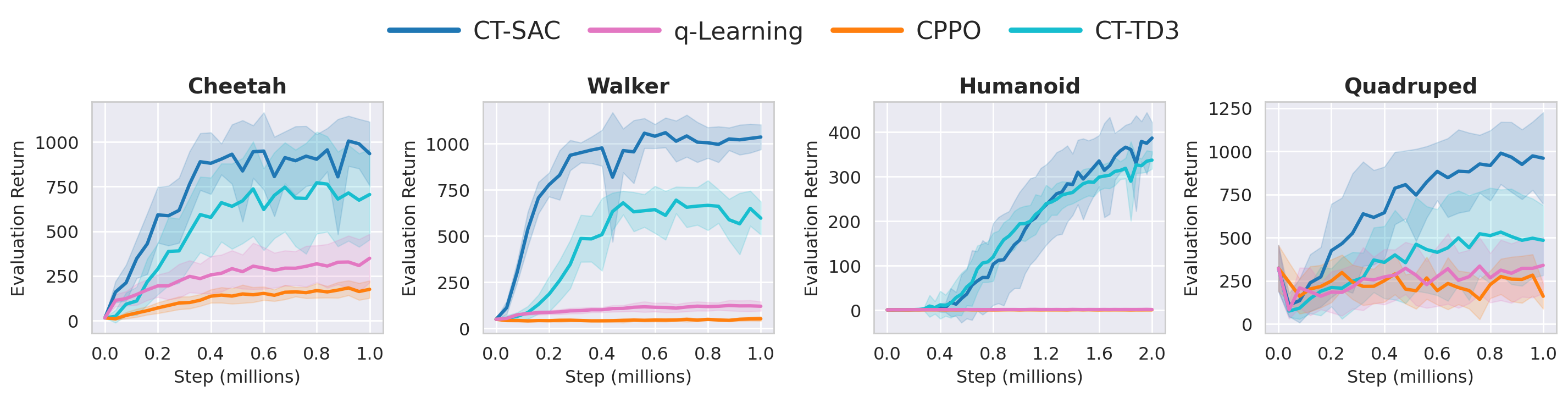}
    \caption{Evaluation returns for 4 continuous-time algorithms including our CT-SAC and CT-TD3 on control tasks over 12 seeds.}
    \label{fig:continuous_algorithms_control_tasks_plot}
\end{figure*}

\paragraph{Step 2: Updating $V$ from $q$ is not a discrete-time max step.}
In discrete-time RL, one typically recovers $V$ via a hard or soft maximization of $Q$:
\begin{equation}
\begin{aligned}
  V(x) &= \max_{a} Q(x,a), \text{ or} \\
  V(x) &= \alpha\log\int_{\mathcal{A}}\exp\big(Q(x,a)/\alpha\big)\,da
\end{aligned}
  \label{eq:dt-soft-backup}
\end{equation}
In continuous time, however, $q$ is an instantaneous \emph{rate}, not a value. If we relate a
small-step return by $Q_u(x,a)\approx V(x)+u\,q(x,a)$ as $u\to 0$, then the soft backup degenerates:
\begin{equation}
\begin{aligned}
V(x)
&=\alpha\log\int_{\mathcal{A}}\exp\big((V(x)+u\,q(x,a))/\alpha\big)\,da \\
&=V(x)+\alpha\log\int_{\mathcal{A}}\exp\big(u\,q(x,a)/\alpha\big)\,da
\end{aligned}
\label{eq:degenerate-backup}
\end{equation}
As $u\to 0$, this becomes a trivial update $V_{\mathrm{new}}(x)\to V(x)$, while enforcing equality
yields a meaningless normalization condition. This reflects a unit mismatch: $V$ behaves like a displacement (a value), while $q$ behaves like a velocity (a derivative), so direct max/softmax backups are ill-posed in the infinitesimal setting.

To extract a non-degenerate update from $q$ in the infinitesimal setting, we propose a novel approach that interprets the update not as a discrete backup but as a \emph{flow} indexed by an iteration parameter $\tau$. As a first step, consider the soft backup increment at small $u$:
\begin{equation}
  V_u(x)
  :=
  V(x)
  + \alpha\log\int_{\mathcal{A}}\exp\big(u\,q(x,a)/\alpha\big)\,da
  \label{eq:Vu-def}
\end{equation}
The increment vanishes as $u\to 0$. However, Jensen inequality implies an overestimation for $V_u$:
\begin{equation}
  V_u(x)-V(x)
  \;\le\;
  \alpha u\,\log\int_{\mathcal{A}}\exp\!\left(q(x,a)/\alpha\right)\,da
  \label{eq:Vu-bound}
\end{equation}
Crucially, this overestimation allows us to derive an update gradient flow by considering the ratio $[V_u(x) - V(x)] / u$:
\begin{equation}
  \frac{d}{d\tau}V_\tau(x)
  =
  \mathcal{Q}^{(\alpha)}(q_{V_\tau})(x).
  \label{eq:V-flow-q}
\end{equation}
where the aggregation operator $\mathcal{Q}$ acts on $q$ as:
\begin{equation}
\begin{aligned}
  \mathcal{Q}^{(\alpha)}(q)(x)
  &:=
  \alpha\log\int_{\mathcal{A}}\exp\big(q(x,a)/\alpha\big)\,da,\\
  \mathcal{Q}^{(0)}(q)(x)&:=\sup_{a\in\mathcal{A}}q(x,a).
\end{aligned}
  \label{eq:Q-operator}
\end{equation}

Under our decoupled scheme, $q_\tau = q_{V_{\tau}}$ is defined solely by $V_\tau$, and substituting this into
\cref{eq:Q-operator} yields exactly the Hamiltonian operator:
\begin{equation}
  \mathcal{Q}^{(\alpha)}(q_V)(x)
  =
  H^{(\alpha)}(V)(x).
  \label{eq:Q-equals-H}
\end{equation}
Therefore the flow becomes the Hamiltonian-driven ODE:
\begin{equation}
  \frac{d}{d\tau}V_\tau(x)=H^{(\alpha)}(V_\tau)(x),
  \label{eq:V-flow-H}
\end{equation}
and we have the following forward-Euler step with step size $\tau>0$ that is called the \textbf{Picard--Hamiltonian iteration}:
\begin{equation}
  V^{\mathrm{new}}(x)
  =
  V(x)+\tau\,H^{(\alpha)}(V)(x)
  =
  T_\tau^{(\alpha)}(V)(x).
  \label{eq:picard-update}
\end{equation}
In \cref{sec:theory}, we establish convergence of this \textit{ideal} iteration to $V^{(\alpha)}$ as $\tau\to 0$.

\textbf{Model-free $q$-estimation induces discretized Hamiltonian flows.} The remaining gap is that $(L^aV)(x)$ in \cref{eq:q-control} is not available in model-free RL. Instead, we estimate the instantaneous rate via its small-time representation: for a holding time $u>0$,
\begin{equation}
\begin{aligned}
  q_V(x,a) &= \lim_{u\to 0} q_V^u(x,a),\\
  q_V^u(x,a)
  &:= \frac{
    e^{-\beta u}\,\E\!\left[V(X_{t+u})\mid X_t=x,a_t=a\right] - V(x)
  }{u}\\
  &+ r(x,a).
\end{aligned}
  \label{eq:q-limit-estimator}
\end{equation}

This replaces the generator by a short-horizon value increment and converges to $q_V$ as $u\to 0$. To reduce discretization bias, we further apply Richardson extrapolation:
\begin{equation}\label{eq:def-qVutilde-main}
  \tilde q_V^u(x,a)
  := 2\,q_V^{u/2}(x,a) - q_V^u(x,a).
\end{equation}
These approximations induce ``discretized'' Hamiltonians:
\begin{equation}
  H_u^{(\alpha)}(V):=\mathcal{Q}^{(\alpha)}(q_V^u),
  \qquad
  \tilde H_u^{(\alpha)}(V):=\mathcal{Q}^{(\alpha)}(\tilde q_V^u),
  \label{eq:Hu-def}
\end{equation}
and the corresponding value iterations
\begin{equation}
  V_{k+1}=V_k+\tau\,H_u^{(\alpha)}(V_k),
  \qquad
  V_{k+1}=V_k+\tau\,\tilde H_u^{(\alpha)}(V_k).
  \label{eq:Hu-iterations}
\end{equation}
In \cref{sec:theory}, we prove convergence of these discretized flows (including the Richardson variant) and show that $\tilde q_{V_k}^u$ converges to the optimal $q_{V^{(\alpha)}}$ as $(\tau,u)\to 0$.

\subsection{Final algorithm: a single critic via $Q \approx V+q$}\label{sec:formulation:final}
The decoupled iteration above is conceptually clean: learn $q$ given $V$, then update $V$ via the Hamiltonian flow. However, maintaining two separate function approximators is inconvenient in practice. We therefore
represent both quantities with a single critic network:
\begin{equation}
  Q_k(x,a) := V_k(x) + q_k(x,a).
  \label{eq:Q-decomposition}
\end{equation}
This is a numerical parameterization: while $V$ and $q$ have different roles, the decomposition lets one network carry both state-value and advantage-rate information. Near optimality, $q_k(x,a)$ is typically
small for actions drawn from the improved policy, so $Q_k(x,a)\approx V_k(x)$ on-policy, making a single critic sufficient. With this representation, the decoupled update yields a practical sample-based critic target. By defining $\tilde{Q}_k(x,a):=Q_k(x,a)-\alpha\log\pi(a\mid x)$, the critic update takes the form:
\begin{equation}
\begin{aligned}
Q_{k+1}&(x, a)
=
(1-\tau)Q_k(x, a)
+ \tau \mathbb{E}_{a}\big[\tilde{Q}_k(x, a)\big] \\
&+ \tau r(x, a)
+ \tau \frac{\gamma\,\mathbb{E}_{a'}\big[\tilde{Q}_k(x', a')\big]
- \mathbb{E}_a\big[\tilde{Q}_k(x, a)\big]}{u}
\end{aligned}
\label{eq:main_Q_update}
\end{equation}
where discount $\gamma = e^{-\beta u}$, and $u$ is the (possibly irregular) transition duration in the sampled tuple $(x,a,r,x',u)$. By the variational form of the soft Hamiltonian (Appendix~\ref{sec:appendix:further_theoretical_details}), the maximizer over $\pi(\cdot\mid x)$ is a Boltzmann policy, $\pi_V^{(\alpha)}(a\mid x)\propto \exp\!\big(q_V(x,a)/\alpha\big)$. Since $Q=V+q$ differs from $q$ only by an $a$-independent shift, we implement the actor update equivalently as $\pi(a\mid x)\propto \exp\!\big(Q(x,a)/\alpha\big)$.

Combining these steps gives a continuous-time soft actor--critic method (CT-SAC; see \cref{algo:ct_sac}). A deterministic variant based on the hard operator gives a continuous-time analogue of TD3 (CT-TD3; see \cref{algo:ct_td3}).

\begin{figure*}[ht]
    \centering
    \includegraphics[scale=0.45]{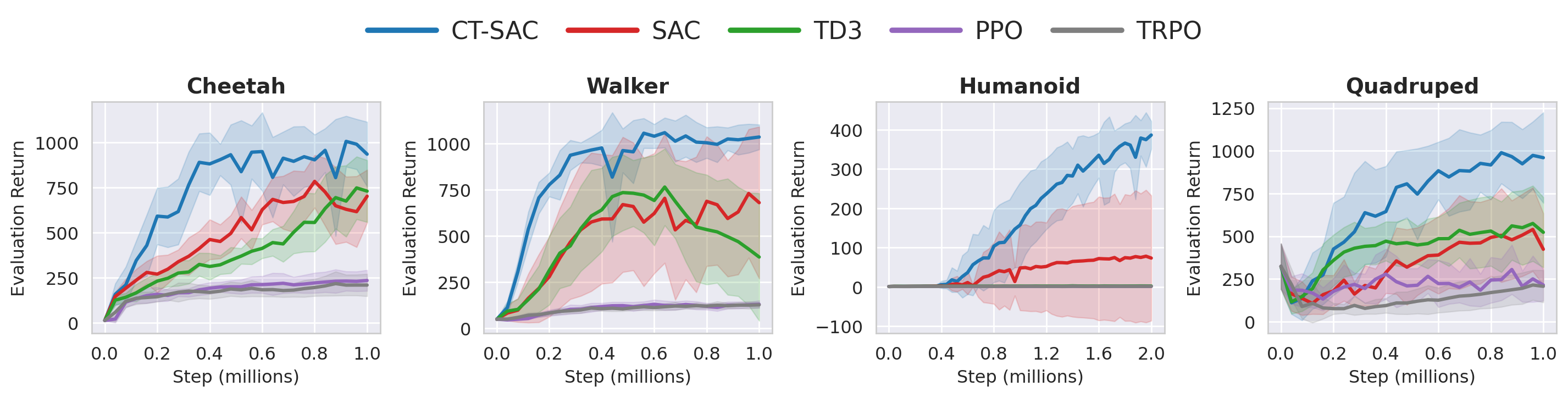}
    \caption{Evaluation returns for our CT-SAC against 4 discrete-time algorithms on control tasks over 12 seeds.}
    \label{fig:discrete_algorithms_control_tasks_plot}
\end{figure*}

\section{Theoretical analysis}\label{sec:theory}
This section provides rigorous guarantees for the decoupled continuous-time learning framework in \cref{sec:formulation}. Our analysis proceeds in three layers. We first study the \emph{ideal} Picard--Hamiltonian iteration $V_{k+1}=T_\tau^{(\alpha)}(V_k)=V_k+\tau H^{(\alpha)}(V_k)$, which corresponds to the analytic rate $q_V$. We then analyze the \emph{model-free discretized flows} obtained by replacing $q_V$ with its finite-horizon approximation $q_V^u$, and its Richardson-corrected version $\tilde q_V^u$, yielding updates based on $H_u^{(\alpha)}$ and $\tilde H_u^{(\alpha)}$. Finally, we lift these results to the single-critic implementation of CT-SAC/CT-TD3.

Recall that $H^{(\alpha)}$ contains the infinitesimal generator term $L^aV$. Thus, unlike discrete-time Bellman operators, $T_\tau^{(\alpha)}$ is not a \textbf{contraction} mapping on $(C_b,\|\cdot\|_\infty)$ and purely analytic approaches often require Sobolev-type regularity with heavy analysis machinery. Instead, we compare $T_\tau^{(\alpha)}$ to a probabilistic dynamic-programming operator that is a sup-norm contraction and treat $T_\tau^{(\alpha)}$ as a controlled local approximation. We define the entropy-regularized dynamic-programming semigroup $\{\Phi_\tau^{(\alpha)}\}_{\tau\ge0}$ by:
\begin{equation}
\begin{aligned}
  (\Phi_\tau^{(\alpha)}& V)(x)
  := \sup_{\pi}\,\mathbb{E}\Bigg[
    \int_0^\tau e^{-\beta t}\Big(r(X_t,a_t) \\
    &- \alpha \log \pi(a_t\mid X_t)\Big)\,dt
    + e^{-\beta \tau} V(X_\tau)
    \,\Big|\, X_0 = x
  \Bigg]
\end{aligned}
  \label{eq:semigroup}
\end{equation}

The semigroup $\Phi_\tau^{(\alpha)}$ is value-only and serves as our contraction reference; it does not by itself provide the action-dependent signal needed for actor--critic updates. We begin with convergence of the ideal Picard--Hamiltonian iteration.

\begin{restatable}[Value Function Convergence]{theorem}{ValueFunctionConvergence}
\label{thm:value_convergence}
Fix $\alpha\ge 0$ and $\tau>0$. Under Assumptions \ref{assumption:dynamics}, \ref{assumption:reward}, and \ref{assumption:test-functions}, let $V_k$ being the value function after $k$ iterations $V_{k+1} := T_\tau^{(\alpha)}(V_k)$. Then, $V_k$ converges to $V^{(\alpha)}$ as $\tau \to 0$, and for all $k\ge 0$,
\begin{equation}\label{eq:value_fct_bound}
\|V_k - V^{(\alpha)}\|_\infty
\le C_0 e^{-\beta\tau k}
   + \frac{C_1\tau^{3/2}}{1-e^{-\beta\tau}}.
\end{equation}
Here $C_0$ and $C_1$ are constants independent of $k$.
\end{restatable}

\paragraph{Proof sketch.}
We compare the approximate update $V_{k+1}=T_\tau^{(\alpha)}(V_k)$ with the semigroup iteration
\begin{equation}
  W_{k+1}:=\Phi_\tau^{(\alpha)}(W_k),\qquad C_0:=\|W_0-V^{(\alpha)}\|_\infty .
\end{equation}
The proof proceeds in 2 steps. In \textbf{step 1}, through a short-time It\^{o} expansion, we prove the one-step operator error $\|\Phi_\tau^{(\alpha)}(V)-T_\tau^{(\alpha)}(V)\|_\infty
\le C_\alpha\,\tau^{3/2}$. Then, in \textbf{step 2}, with $E_k:=V_k-W_k$, we recursively prove $\|E_k\|_\infty \le \frac{C_\alpha\,\tau^{3/2}}{1-e^{-\beta\tau}}$ through the decomposition:
\begin{equation*}
\begin{aligned}
  E_{k+1}
  &= T_\tau^{(\alpha)}(V_k) - \Phi_\tau^{(\alpha)}(W_k)\\
  &= \bigl(T_\tau^{(\alpha)}(V_k) - \Phi_\tau^{(\alpha)}(V_k)\bigr)
    + \bigl(\Phi_\tau^{(\alpha)}(V_k) - \Phi_\tau^{(\alpha)}(W_k)\bigr)
\end{aligned}
\end{equation*}
Hence, $\norm{E_{k+1}}_{\infty} \le C_\alpha \tau^{3/2} + e^{-\beta\tau}\,\norm{E_k}_{\infty}$.

Finally, $\Phi_\tau$ is a contraction operator so that $W_k$ converges exponentially to $V^{(\alpha)}$: $\|W_k-V^{(\alpha)}\|_\infty\le e^{-\beta\tau k}C_0$. The triangle inequality then gives the desired bound \eqref{eq:value_fct_bound}. \hfill$\square$

We extend the result to the model-free estimator $q_V^u$ (and $\tilde q_V^u$). We state the theorem for fixed $u$, and defer the random-increment case $U$ to the Appendix~\ref{sec:appendix:random_time}.

\begin{restatable}[Convergence with $q$-function]{theorem}{qFunctionConvergence}\label{thm:q_convergence}
Let $V_k$ be the resulting function defined by Richardson Picard-Hamiltonian iterations $V_{k+1}:= V_k + \tau\,\tilde H_u^{(\alpha)}(V_k)$. Under  Assumption~\ref{assumption:dynamics-smooth} and \ref{assumption:value-C4}, for $u > 0$, there exist constants $C_1,C_2, L$, and $\rho>0$ such that for all $k\ge 0$:
\begin{equation}\label{eq:V-conv-rich-main}
  \|V_k - V^{(\alpha)}\|_\infty
  \;\le\;
  e^{-\rho\tau k}\,\|V_0 - V^{(\alpha)}\|_\infty
  + \frac{C_1(\tau^2 + \tau u^2)}{1-e^{-\rho\tau}}.
\end{equation}
And, for $q$-function estimation, we obtain the error:
\begin{equation}\label{eq:q-conv-rich-main}
  \|\tilde q_{V_k}^u - q_{V^{(\alpha)}}\|_\infty
  \;\le\;
  \frac{L}{u}\,\|V_k - V^{(\alpha)}\|_\infty
  + C_2\,u^2
\end{equation}
\end{restatable}

In particular, if $\tau, u \to 0$ and $\tau/u \to 0$, then the value estimation $V_k$ and the Richardson-version of $q$-function estimates converges to the optimal values $V^{(\alpha)}$ and $q_{V^{(\alpha)}}$.

\paragraph{Proof sketch.} Now we need to find the error when $q$ is estimate by finite fraction $q^u$. We use a similar proof as in \cref{thm:value_convergence} to bound value function with a stronger expansion for $\Phi_{\tau}$. The main difference now is the new error between $H^{(\alpha)}$ and $\tilde H^{(\alpha)}_u$. Luckily, the estimate $\|\tilde q_V^u - q_V\|_\infty \le C u^2$ allows the Hamiltonian bound $\|\tilde H_u^{(\alpha)}(V) - H^{(\alpha)}(V)\|_\infty \le C u^2$ so that the one-step error $\|\Phi_\tau^{(\alpha)}(V) - (V+\tau \tilde H_u^{(\alpha)}(V))\|_\infty$ is now bounded by $C_1(\tau^2 + \tau u^2)$. 

For $q$-convergence, we use the Lipschitz bound $\|\tilde q_V^u-\tilde q_W^u\|_\infty \le \frac{L}{u}\|V-W\|_\infty$, and transfer the work to $V$. \hfill$\square$

Finally, we translate these guarantees to the single-critic update $Q_k=V_k+q_k$ underlying CT-SAC/CT-TD3.

\begin{restatable}[Theoretical Algorithm Convergence]{theorem}{TheoreticalAlgorithmConvergence}\label{thm:algorithm-convergence}
Consider the exact iterative update of $Q_k$ through \cref{eq:main_Q_update} with no optimization or statistical error. Then $Q_k$ converge to the optimal value $Q^{(\alpha)} = V^{(\alpha)} + q_{V^{(\alpha)}}$ as $k \to \infty$.
\end{restatable}

\begin{restatable}[Algorithm Convergence]{corollary}
{AlgorithmConvergence}\label{corollary:algorithm_statistical-convergence}
Suppose at each iteration $k$, $n_k$ samples are used in the gradient update. Then, for any pair of $(\epsilon, \delta) > 0$, there exists an iteration $L$ together with the sequence $n_1, n_2, \cdots, n_L$ so that with the probability of at least $\delta$, $\E\left[|Q_L(X, A) - Q^{(\alpha)}(X, A)|^2\right] < \epsilon$ for random state-action pair $(X, A)$. Here $Q^{(\alpha)} = V^{(\alpha)} + q_{V^{(\alpha)}}$ is the optimal function, and $Q_L$ is the learned function from \cref{algo:ct_sac}.
\end{restatable}

\paragraph{Proof sketch.}
We construct the corresponding sequences $(V_k,q_k)$ induced by the exact update and prove by induction that the decomposition \eqref{eq:Q-decomposition} is preserved, i.e., $Q_k = V_k + q_k$ for all $k$. The convergence then follows by \cref{thm:q_convergence}. For statistical part, let $\hat{Q}_k$ denote the learned critic with finite samples of the $k^{th}$ exact iteration of $Q_k$. Assume the update mapping is $\mathcal{F}$, then $Q_{k+1}=\mathcal{F}(Q_k)$. More importantly $\hat{Q}_{k+1}$ is obtained by regression against $\mathcal{F}(\hat{Q}_k)$ with statistical error $\mathrm{stat}_k:= \norm{\hat{Q}_{k+1} - \mathcal{F}(\hat Q_k)}$. For $\epsilon_k = \norm{\hat Q_k - Q_k}$, one can obtain $\epsilon_{k+1} \;\le\;\mathrm{stat}_k + \Big(1+\frac{\tau}{u}\Big)\epsilon_k$ through the Lipschitzness of update mapping $\mathcal{F}$. For large enough number of samples, $\mathrm{stat}_k \to 0$, yielding the converge of $\epsilon_k$, and therefore, of the error w.r.t. optimal function $Q^{(\alpha)}$. \hfill$\square$

More refined control of sample complexity is given in \cref{sec:appendix:regret_bound} through a norm-transfer argument. Finally, refer to Appendices \ref{sec:appendix:value_convergence}, \ref{sec:appendix:q_convergence}, and \ref{sec:appendix:algorithm_convergence} for the full detailed proofs of Theorems \ref{thm:value_convergence}, \ref{thm:q_convergence}, and \ref{thm:algorithm-convergence} respectively.

\section{Experiments}\label{sec:experiments}
Moving beyond fixed-step simulated SDE benchmarks, we test whether continuous-time RL provides gains with irregular decision times and very small time steps on standard control benchmarks and a real trading task.

\subsection{Tasks and baselines}\label{sec:exp_setup}

\textbf{Irregular-time setup.} Across all tasks, the agent acts at event times $t_0 < t_1 < \cdots$ with holding times $u_k := t_{k+1}-t_k$, where $u_k$ \textbf{mixes between small and large timesteps} within the same rollout, so trajectories contain many short ``micro-steps'' interleaved with occasional larger jumps. Concretely, we implement this by sampling $u_k$ around the nominal control step $u_{\mathrm{nom}}$. For example, for \texttt{Cheetah} we allow holding times as small as $\approx u_{\mathrm{nom}}/5$ and as large as $\approx 3u_{\mathrm{nom}}$. This setting is strictly harder than fixed-step training because trajectories are not time-aligned and learning must generalize beyond a single step size.

\begin{figure}[ht]
    \centering
    \includegraphics[scale=0.45]{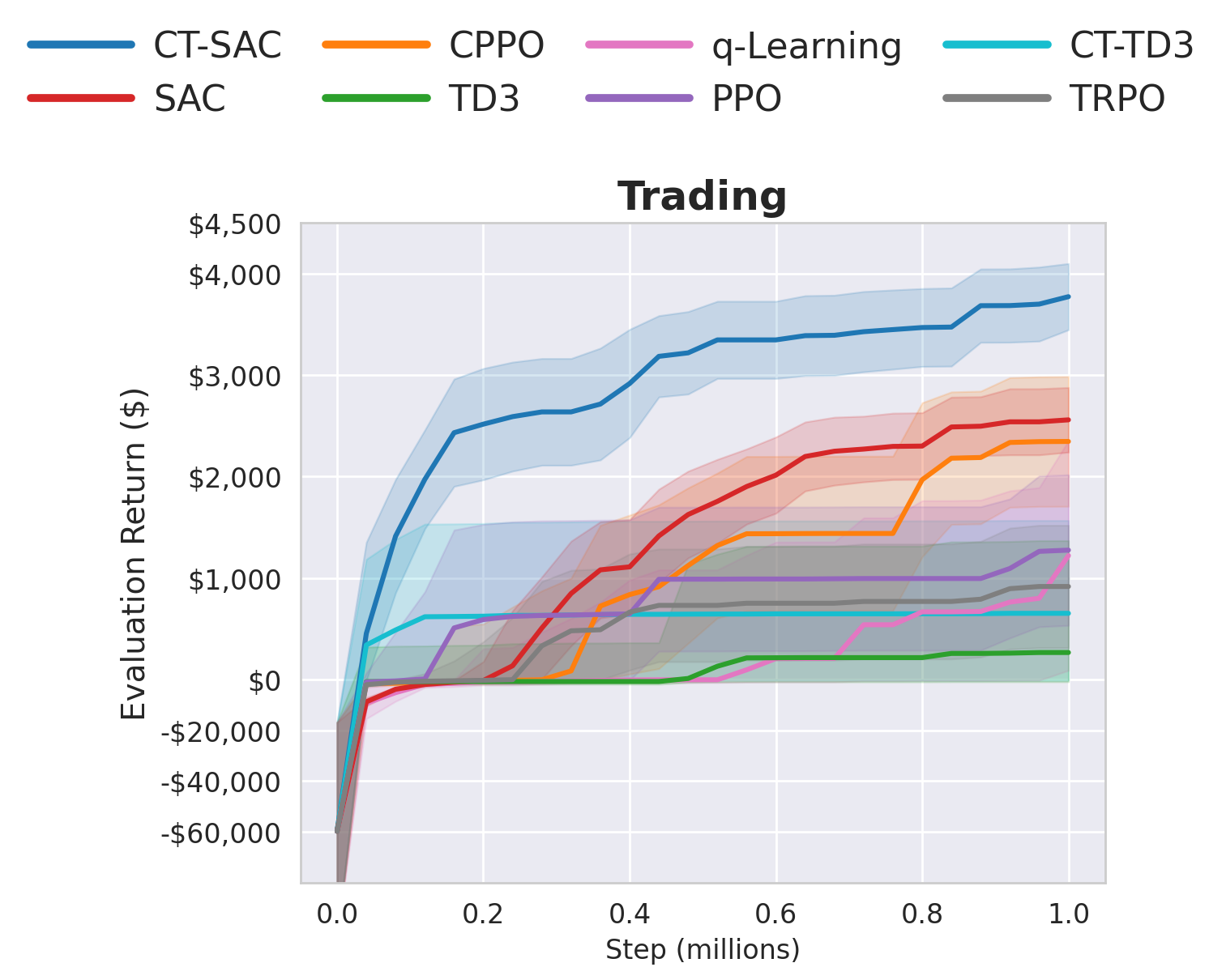}
    \caption{Two-week evaluation returns for CT-SAC and CT-TD3 (Ours) against 6 other algorithms on trading tasks.}
    \label{fig:trading_plot}
\end{figure}

\textbf{Control tasks.} We use DeepMind Control Suite \cite{dmcontrol} for high-dimensional continuous control. Irregular timing is particularly important for humanoid and quadruped locomotion, where stability relies on timely corrective actions.

\textbf{Real-world trading.} We additionally evaluate on a trading environment constructed from minute-resolution market data (Alpaca API). The universe contains four industry sectors with five large-cap tickers each, and each episode spans two weeks of trading (390 minutes/day, 5 days/week). The agent makes decisions at irregular intervals with holding times $u\in[1,20]$ minutes, using a state built from multi-scale historical price features (minutes to days). The action is to buy/sell a bounded quantity within an industry group, with position limits and transaction costs. The reward is the realized profit over the subsequent holding interval, so episode return equals total two-week P\&L. To avoid temporal leakage, we train on Q3 2023 -- Q2 2025 and evaluate on Q3 2025 using a fixed set of episodes so all methods face the same market paths. Details are provided in Appendix~\ref{sec:appendix:experiment_details}.

\textbf{Baselines.} We benchmark against two baseline groups: continuous-time and discrete-time. For continuous-time RL, we include representative advantage-rate methods with martingale enforcement ($q$-Learning) \citep{Jia2022-yg}, as well as a continuous-time policy-gradient baseline (CPPO) \citep{CPPO}. To our knowledge, most existing continuous-time deep RL algorithms fall into these two families. Because these methods can be optimization-sensitive, we additionally apply stabilization techniques such as gradient clipping to improve their performance. For discrete-time RL, we compare against widely used state-of-the-art algorithms, SAC, TD3, TRPO, and PPO, implemented in Stable-Baselines3 \cite{stable-baselines3}.

\textbf{Evaluation.} We train each method for 1M environment steps on control tasks (2M on \texttt{humanoid}). We tune hyperparameters over the search space (detailed in Appendix~\ref{sec:appendix:experiment_details}) and select the best configuration for each method. We additionally keep the top 3 configurations for ablations, and report the best one in the main figures. Evaluation is performed periodically using 10 episodes per checkpoint for control tasks and 30 episodes for trading. We report mean learning curves and aggregate metrics over 12 random seeds.

\textbf{Runtime \& complexity}. Details are given in Appendix~\ref{sec:appendix:experiment_details}.

\subsection{Results}\label{sec:exp_results}
\textbf{Continuous-time baselines.} \cref{fig:continuous_algorithms_control_tasks_plot} compares CT-SAC and CT-TD3 with prior continuous-time baselines across the control suite. Our methods achieve the strongest overall performance with consistently stable learning on both high-dimensional locomotion and trading tasks. These results highlight a limitation of existing continuous-time methods, where optimizing coupled martingale objectives for $V$ and $q$ is challenging. Our approach instead breaks learning into simpler subproblems, improving optimization reliability.

\textbf{Discrete-time baselines.} Our method CT-SAC also outperforms leading discrete-time algorithms (see Figure~\ref{fig:discrete_algorithms_control_tasks_plot} and ~\ref{fig:trading_plot}). When step sizes vary across samples and between training and evaluation, standard discrete-time updates degrade, while our continuous-time-derived update remains robust to non-uniform timestamps. Across both control and trading, our method improves early, with learning curves rise within the first fraction of steps and then stabilize, while discrete-time methods typically require many more samples to accelerate. A possible explanation is that under variable step sizes, discrete-time critics may only learn an order-preserving surrogate rather than the true value function across time scales, which delays robust improvement. CT-TD3 tends to underperform CT-SAC but still improves on its discrete-time counterpart, underscoring the value of entropy regularization and the need for theory improvement. 

\textbf{Performance on trading.} While on standard continuous-control benchmarks, CPPO is consistently weaker than its discrete-time counterparts, for trading task, CPPO becomes competitive on a similar scale to strong off-policy baselines such as SAC. This suggests that continuous-time modeling can matter more in trading under irregular decision times. Still, our method achieves the strongest overall performance. Despite using only price-derived features (no news or fundamentals data), our policy learns profitable strategies and attains the best evaluation returns among all baselines. Starting from \$100{,}000, CT-SAC achieves about \$3{,}500 profit every two weeks on average (see \cref{fig:trading_plot}). Across episodes of Q3 2025, it achieves a total profit of roughly \$21{,}000, nearly doubling the second-best method. This highlights that continuous-time RL can be effective for real-world trading when the learning objective and optimization are designed to remain stable under irregular decision intervals and noisy dynamics. 

Finally, we conduct statistical testing on 12 seeds and find that improvements are significant across tasks.

\textbf{Ablation study.} Using the top 3 hyperparameter configurations identified during tuning for each method, we run ablations across all tasks. The results show that our tuned baselines are competitive, and that our method remains top performer even under non-optimal hyperparameter settings, indicating robustness to hyperparameter variation (Table~\ref{table:ablation_table}).

\textbf{Increment modeling}. We also evaluate reward-shaping variants for discrete-time RL. These variants yield performance similar to their standard counterparts and consistently rank behind CT-SAC (Figure~\ref{fig:increment_modeling_control_tasks_plot} and \ref{fig:increment_modeling_trading_task_plot}).

\section{Conclusion}\label{sec:conclusion}
We introduce a continuous-time reinforcement learning framework for irregular decision times. The core idea is a simplified learning procedure: estimate an instantaneous advantage-rate signal from rollouts under a fixed value predictor, then update the value predictor via a Hamiltonian flow. This avoids the coupled martingale objectives used in prior continuous-time methods. A single-critic implementation makes the approach compatible with standard off-policy actor–critic training and stable under variable step sizes. On the theory side, we establish convergence guarantees despite the presence of infinitesimal generators, where discrete-time contraction arguments no longer apply. Our analysis uses a probabilistic technique that links the flow update to the dynamic-programming semigroup and extends naturally to the model-free setting with finite time increments. Empirically, we evaluate our method under an irregular-time setup with many short steps mixed with large jumps, on high-dimensional control benchmarks and a real-world trading task. Across these settings, our method consistently outperforms existing continuous-time baselines and strong discrete-time algorithms.

\section*{Reproducibility Statement}
We provide the complete code in the codebase Github link: \texttt{https://github.com/mpnguyen2/ct-rl}. Please see Appendices \ref{sec:appendix:experiment_details} and \ref{sec:appendix:further_experiment_details} for further details.

\section*{Impact Statement}
This paper presents work whose goal is to advance the field of Machine Learning by developing reinforcement learning methods for continuous-time settings with irregular decision intervals. The broader impacts are expected to be similar to those of many general advances in reinforcement learning: the methods may help improve robustness and reproducibility in time-irregular data and control problems, while we do not anticipate immediate or unusual ethical concerns beyond standard considerations for responsible use of learned decision-making systems.

\bibliography{citation}
\bibliographystyle{icml2026}

%%%%%%%%%%%%%%%%%%%%%%%%%%%%%%%%%%%%%%%%%%%%%%%%%%%%%%%%%%%%%%%%%%%%%%%%%%%%%%%
%%%%%%%%%%%%%%%%%%%%%%%%%%%%%%%%%%%%%%%%%%%%%%%%%%%%%%%%%%%%%%%%%%%%%%%%%%%%%%%
% APPENDIX
%%%%%%%%%%%%%%%%%%%%%%%%%%%%%%%%%%%%%%%%%%%%%%%%%%%%%%%%%%%%%%%%%%%%%%%%%%%%%%%
%%%%%%%%%%%%%%%%%%%%%%%%%%%%%%%%%%%%%%%%%%%%%%%%%%%%%%%%%%%%%%%%%%%%%%%%%%%%%%%
\appendix
\onecolumn
\clearpage

% --- re-enable ToC writing for appendix only (ICML disables it globally) ---
\makeatletter
\@ifpackageloaded{hyperref}{%
  \def\addcontentsline#1#2#3{%
    \addtocontents{#1}{\protect\contentsline{#2}{#3}{\thepage}{\@currentHref}}%
  }%
}{%
  \def\addcontentsline#1#2#3{%
    \addtocontents{#1}{\protect\contentsline{#2}{#3}{\thepage}}%
  }%
}
\makeatother
% -------------------------------------------------------------------------

\renewcommand{\contentsname}{Appendix Table of Contents}
\setcounter{tocdepth}{2} % sections only
\tableofcontents
\clearpage

\section{Related works}\label{sec:appendix:related_works}
\paragraph{Discrete-time reinforcement learning.}
Most effective deep reinforcement learning for continuous control can be broadly grouped into (i) off-policy actor--critic methods and (ii) on-policy policy-gradient methods \cite{Sutton2018-fn}, with widely used representatives including Soft Actor--Critic (SAC) \cite{SAC}, Twin Delayed Deep Deterministic Policy Gradient (TD3) \cite{TD3}, Trust Region Policy Optimization (TRPO) \cite{TRPO}, and Proximal Policy Optimization (PPO) \cite{PPO}, while many other practical algorithms can be viewed as variants, hybrids, or refinements of these cornerstones. SAC is particularly powerful due to its entropy-regularized objective, which encourages exploration while yielding stable policy improvement and strong empirical performance in high-dimensional continuous-action domains. In contrast to SAC's stochastic policy, TD3 is a deterministic actor--critic method that can be viewed as a strengthened successor of DDPG \cite{ddpg}, mitigating overestimation bias via clipped double critics and delayed policy updates. On the policy-gradient side, TRPO is notable for its principled trust-region formulation and theoretical guarantees, whereas PPO is widely adopted as an efficient and scalable approximation that performs robustly across a wide range of tasks. Despite their success, all of these methods are fundamentally formulated on fixed-step Markov decision processes (MDPs). When the decision interval $\Delta t$ becomes small or irregular, the Bellman recursion becomes increasingly sensitive to discretization \cite{Tallec, Jia2022-yg}. Additionally, learning signals can degenerate in the settings mentioned in \cref{sec:intro}, motivating new approaches that are more suitable for continuous-time problems.

\paragraph{Early continuous-time RL and ODE-based approaches.}
Continuous-time reinforcement learning has long been studied from a control-theoretic perspective, primarily through the Hamilton--Jacobi--Bellman (HJB) equation \cite{Fleming2006-wa}. Early work by \citet{Baird}, \citet{Doya2000-do}, and \citet{munos06b} explored continuous-time value learning, policy gradients, and advantage-based updates as alternatives to fixed-step MDP formulations. However, these early approaches are typically restricted to deterministic dynamical systems or require additional structure beyond standard model-free RL. For instance, they need access to the closed-form transition dynamics or their derivatives to construct path-wise or value-gradient updates. Subsequent developments \cite{sutton_cont_rl, Ainsworth2020-vw, Yildiz2021-wc} further refine continuous-time analogues of actor--critic or policy-gradient methods through closed-form formulae. Nonetheless, these methods also inherit similar limitations, remaining either model-based, tied to deterministic dynamics, or relying on assumptions that are difficult to meet in general stochastic environments under irregular decision times. A central limitation is that the continuous-time problem demands a stochastic-control formulation, which should be associated with stochastic differential equations (SDEs) rather than deterministic ODE, as discussed in \cref{sec:prelim:mdp-to-sde}. Without a framework that works well under SDEs, na\"ive temporal-difference (TD) learning can lead to minimizing the variance terms induced by Brownian noise rather than true dynamic-programming errors, a phenomenon highlighted in \cite{Jia_policy_eval, Jia2022-yg} and analyzed further in \cref{sec:formulation:challenges}. Moreover, as the timestep shrinks, the discrete-time $Q$-function collapses to the value function, eliminating meaningful action ranking. This collapse is empirically suggested by \citet{Tallec}, who demonstrate the sensitivity of $Q$-learning-based methods to time discretization. Nevertheless, that work still operates in a discrete setting with fixed time step, rather than a continuous-time setup under stochastic dynamics.

\paragraph{Stochastic-control formulations and recent continuous-time RL.} A more recent novel line of work \cite{wang2020, Jia_policy_eval, Jia_policy_gradient, Jia2022-yg, CPPO} formulates reinforcement learning directly as stochastic control problem under diffusion dynamics. \citet{wang2020} introduce an entropy-regularized continuous-time framework that models exploration through relaxed controls and establishes a stochastic HJB characterization. \citet{tang2021} further prove important convergence results for the associated stochastic HJB. Building on this foundation, \citet{Jia2022-yg} show that the conventional state--action value function does not survive the continuous-time limit and propose the advantage-rate (or ``little-$q$'') function as the correct infinitesimal analogue. This quantity is closely linked to the Hamiltonian and admits characterization via martingale conditions, enabling continuous-time analogues of policy evaluation and improvement. Crucially, such martingale conditions arise precisely to address the failure modes of naïve TD updates under diffusion noise, as discussed in \cref{sec:formulation:challenges}. Within the same stochastic-control framework, \cite{CPPO} further develops policy optimization methods, deriving continuous-time counterparts of TRPO and PPO based on occupation measures and performance-difference formulas. Complementary theoretical work, including regret analyses \cite{tang2024regret} and the study of distributional quantities such as the action-gap \cite{action_gap}, further enriches this emerging line of work and clarifies the theoretical behavior of continuous-time learning. In terms of new algorithmic contributions, recent works include \cite{cont_q_score_match}, which follows the martingale-based characterization in \cite{Jia2022-yg}, and \cite{score_as_action}, which extends policy-gradient methods \cite{CPPO} using diffusion-based parameterizations. These works, however, specialize the policy class through diffusion modeling or apply similar martingale objectives to other tasks, rather than proposing a fundamentally new learning approach to address the complexity of martingale characterization. In contrast, our work focuses on developing a new decoupled learning scheme that preserves the continuous-time theoretical foundation while yielding a more practical optimization structure. As shown in \cref{sec:theory} and \cref{sec:experiments}, this decoupling leads to an efficient algorithm with theoretical guarantees and strong empirical performance on advanced control benchmarks and a real-world trading task under small and irregular time steps.

\section{Experiment details}\label{sec:appendix:experiment_details}
\subsection{Task details}
\label{sec:appendix:task_details}

\subsubsection{Trading}
\textbf{Market data and episode construction.}
We consider a diversified equity universe consisting of four major industry sectors, each containing five liquid, large-cap tickers (20 total). We use minute-resolution OHLCV bars and define one episode as two weeks of trading, i.e., $390$ minutes per day over $5$ trading days per week. Decisions are made at \emph{irregular} times: after each action, the environment samples a holding time $u \in [1,20]$ minutes and advances the market by $u$ minutes before the next decision.

\textbf{State, action, and reward.}
The state summarizes recent market behavior through multi-scale historical features (e.g., rolling windows spanning minutes, hours, and days) computed from the minute bars. At each decision time, the agent chooses a bounded buy/sell quantity within a sector-specific group, subject to position limits and proportional transaction costs. Rewards are defined as realized profit over the subsequent holding interval; consequently, the episode return equals the total P\&L accumulated over the two-week horizon.

\textbf{Train--test split.}
To avoid temporal leakage, we train on a contiguous historical window from Q3 2023 through Q2 2025 and evaluate on a disjoint quarter (Q3 2025). Evaluation uses a \emph{fixed set} of episodes spanning the full quarter, so all methods are tested on the same market-path distribution and performance differences cannot be attributed to favorable evaluation windows. Minute-level data are obtained from the Alpaca API~\cite{alpaca_api}. For reproducibility, we provide preprocessed datasets (\texttt{train.npz} and \texttt{eval.npz}) under \texttt{data/trading/processed\_data/} along with scripts and instructions to regenerate the minute-resolution features from raw data.

\subsubsection{Control tasks}
We evaluate on four standard continuous-control benchmarks from the DeepMind Control Suite (DMC)~\cite{dmcontrol}: \texttt{cheetah-run}, \texttt{walker-run}, \texttt{humanoid-walk}, and \texttt{quadruped-run}. DMC provides a physics-based simulator with direct control over integration and control rates, which makes it well-suited for our irregular-time setting. In DMC, rewards are normalized to $[0,1]$; thus the maximum achievable return is approximately proportional to the episode duration.

\subsubsection{Irregular-time configuration}
A central goal of this work is to evaluate learning under \emph{intertwined} time scales, where trajectories contain a mixture of very small and very large time steps within the same episode.
For each task, we report the finest time scale $\Delta t_{\text{physics}}$, the minimum and maximum decision increments $(\Delta t_{\min}, \Delta t_{\max})$, and the maximum episode duration $T_{\max}$, along with the fractions of \emph{small}, \emph{large}, and \emph{average} steps and the range of number of decision steps per episode. Control-task time is measured in seconds, while trading time is measured in minutes (See \cref{table:irregular_dt_config}).

\begin{table}[h]
\centering
\caption{Irregular-time sampling configuration per task: finest time scale $\Delta t_{\text{physics}}$, decision-step bounds $(\Delta t_{\min},\Delta t_{\max})$, maximum episode duration $T_{\max}$, step-type proportions (small/large/average), and the step range per episode.}

\label{table:irregular_dt_config}
\setlength{\tabcolsep}{4.5pt}
\renewcommand{\arraystretch}{1.12}
\begin{tabular}{l c c c c c c c c}
\toprule
\textbf{Task} &
$\boldsymbol{\Delta t_{\text{physics}}}$ &
$\boldsymbol{\Delta t_{\min}}$ &
$\boldsymbol{\Delta t_{\max}}$ &
\textbf{\% small} &
\textbf{\% large} &
\textbf{\% avg} &
$\boldsymbol{T_{\max}}$ &
\textbf{\# of steps range} \\
\midrule
\textbf{Cheetah}   & 0.0020 & 0.002 & 0.030 & 89.1\% & 9.9\% & 1.0\%  & 10   & 1200-2000 \\
\textbf{Walker}    & 0.0025 & 0.005 & 0.075 & 89.1\% & 9.9\% & 1.0\%  & 25   & 1200-2000 \\
\textbf{Humanoid} & 0.0050 & 0.010 & 0.040 & 40.0\% & 40.0\% & 20.0\% & 25   & 800-1000 \\
\textbf{Quadruped} & 0.0050 & 0.005 & 0.050 & 89.1\% & 9.9\% & 1.0\%  & 20   & 1200-2000 \\
\textbf{Trading}       & 1 & 1     & 11    & 89.1\% & 9.9\% & 1.0\%  & 4000 & 1600-2100 \\
\bottomrule
\end{tabular}
\end{table}

\subsection{Ablation study tables}
We report hyperparameter ablations in \cref{table:ablation_table}. Across all environments, \textbf{CT-SAC is consistently the best-performing method among the tested configurations}, and the same trend holds for \textbf{CT-TD3 vs.\ TD3}. To keep the main ablation table compact, we show the \textbf{top/second/third} configurations per algorithm and task; the exact hyperparameter settings corresponding to these entries are provided in the hyperparameter dictionary in \cref{sec:appendix:further_experiment_details}.

\begin{table}[H]
\centering
\caption{Hyperparameter ablations across all tasks. For each algorithm and task, we report the returns of the \textbf{top}, \textbf{second}, and \textbf{third} configurations from the hyperparameter tuning process.}
\label{table:ablation_table}
\small
\setlength{\tabcolsep}{3.5pt}
\renewcommand{\arraystretch}{1.15}

\begin{subtable}[t]{\textwidth}
\centering
\caption{CT-SAC versus discrete-time algorithms SAC and TD3}
{\scriptsize
\setlength{\tabcolsep}{4pt}
\renewcommand{\arraystretch}{1.06}
\resizebox{0.85\textwidth}{!}{%
\begin{tabular}{l ccc ccc ccc}
\toprule
 & \multicolumn{3}{c}{\textbf{CT-SAC}} & \multicolumn{3}{c}{\textbf{SAC}} & \multicolumn{3}{c}{\textbf{TD3}} \\
\cmidrule(lr){2-4}\cmidrule(lr){5-7}\cmidrule(lr){8-10}
\textbf{Task} & \textbf{top} & \textbf{second} & \textbf{third} & \textbf{top} & \textbf{second} & \textbf{third} & \textbf{top} & \textbf{second} & \textbf{third} \\
\midrule
\textbf{Cheetah} & \textcolor{red}{\textbf{934.76}} & \textcolor{blue}{\textbf{863.45}} & \textcolor{green!50!black}{\textbf{807.19}} & 701.88 & 680.30 & 605.63 & 730.12 & 694.15 & 692.81 \\
\textbf{Walker} & \textcolor{red}{\textbf{1035.52}} & \textcolor{blue}{\textbf{928.26}} & \textcolor{green!50!black}{\textbf{817.28}} & 680.08 & 560.98 & 522.85 & 385.73 & 286.09 & 191.74 \\
\textbf{Humanoid} & \textcolor{red}{\textbf{386.75}} & \textcolor{blue}{\textbf{379.75}} & \textcolor{green!50!black}{\textbf{371.75}} & 73.39 & 39.28 & 2.12 & 2.28 & 2.24 & 2.11 \\
\textbf{Quadruped} & \textcolor{red}{\textbf{959.75}} & \textcolor{blue}{\textbf{958.44}} & \textcolor{green!50!black}{\textbf{829.41}} & 423.33 & 314.77 & 284.00 & 522.63 & 518.37 & 464.95 \\
\textbf{Trading} & \textcolor{red}{\textbf{37.72}} & \textcolor{blue}{\textbf{33.09}} & \textcolor{green!50!black}{\textbf{31.98}} & 25.59 & 21.04 & 13.77 & 2.67 & -5.46 & -9.46 \\
\bottomrule
\end{tabular}}}

\end{subtable}

\vspace{7pt}

\begin{subtable}[t]{\textwidth}
\centering
\caption{CT-SAC versus discrete-time algorithms TRPO and PPO}
{\scriptsize
\setlength{\tabcolsep}{4pt}
\renewcommand{\arraystretch}{1.06}
\resizebox{0.85\textwidth}{!}{%
\begin{tabular}{l ccc ccc ccc}
\toprule
 & \multicolumn{3}{c}{\textbf{CT-SAC}} & \multicolumn{3}{c}{\textbf{TRPO}} & \multicolumn{3}{c}{\textbf{PPO}} \\
\cmidrule(lr){2-4}\cmidrule(lr){5-7}\cmidrule(lr){8-10}
\textbf{Task} & \textbf{top} & \textbf{second} & \textbf{third} & \textbf{top} & \textbf{second} & \textbf{third} & \textbf{top} & \textbf{second} & \textbf{third} \\
\midrule
\textbf{Cheetah} & \textcolor{red}{\textbf{934.76}} & \textcolor{blue}{\textbf{863.45}} & \textcolor{green!50!black}{\textbf{807.19}} & 210.10 & 190.51 & 170.73 & 235.61 & 208.42 & 206.70 \\
\textbf{Walker} & \textcolor{red}{\textbf{1035.52}} & \textcolor{blue}{\textbf{928.26}} & \textcolor{green!50!black}{\textbf{817.28}} & 126.73 & 119.95 & 114.90 & 129.94 & 120.28 & 118.08 \\
\textbf{Humanoid} & \textcolor{red}{\textbf{386.75}} & \textcolor{blue}{\textbf{379.75}} & \textcolor{green!50!black}{\textbf{371.75}} & 1.39 & 1.33 & 1.30 & 1.18 & 1.17 & 1.11 \\
\textbf{Quadruped} & \textcolor{red}{\textbf{959.75}} & \textcolor{blue}{\textbf{958.44}} & \textcolor{green!50!black}{\textbf{829.41}} & 206.24 & 178.01 & 135.29 & 214.13 & 208.47 & 202.72 \\
\textbf{Trading} & \textcolor{red}{\textbf{37.72}} & \textcolor{blue}{\textbf{33.09}} & \textcolor{green!50!black}{\textbf{31.98}} & 9.17 & -3.36 & -4.77 & 12.76 & 11.94 & 4.19 \\
\bottomrule
\end{tabular}}}

\end{subtable}

\vspace{7pt}

\begin{subtable}[t]{\textwidth}
\centering
\caption{CT-SAC versus continuous-time algorithms CPPO, q-Learning, and CT-TD3}
\resizebox{0.93\textwidth}{!}{%
\begin{tabular}{l ccc ccc ccc ccc}
\toprule
 & \multicolumn{3}{c}{\textbf{CT-SAC}} & \multicolumn{3}{c}{\textbf{CPPO}} & \multicolumn{3}{c}{\textbf{q-Learning}} & \multicolumn{3}{c}{\textbf{CT-TD3}} \\
\cmidrule(lr){2-4}\cmidrule(lr){5-7}\cmidrule(lr){8-10}\cmidrule(lr){11-13}
\textbf{Task} & \textbf{top} & \textbf{second} & \textbf{third} & \textbf{top} & \textbf{second} & \textbf{third} & \textbf{top} & \textbf{second} & \textbf{third} & \textbf{top} & \textbf{second} & \textbf{third} \\
\midrule
\textbf{Cheetah} & \textcolor{red}{\textbf{934.76}} & \textcolor{blue}{\textbf{863.45}} & \textcolor{green!50!black}{\textbf{807.19}} & 174.50 & 168.64 & 145.92 & 348.61 & 328.69 & 272.00 & 705.86 & 457.24 & 268.61 \\
\textbf{Walker} & \textcolor{red}{\textbf{1035.52}} & \textcolor{blue}{\textbf{928.26}} & \textcolor{green!50!black}{\textbf{817.28}} & 51.77 & 44.24 & 42.04 & 119.78 & 94.79 & 74.67 & 596.50 & 578.19 & 567.22 \\
\textbf{Humanoid} & \textcolor{red}{\textbf{386.75}} & \textcolor{blue}{\textbf{379.75}} & \textcolor{green!50!black}{\textbf{371.75}} & 1.16 & 1.08 & 1.04 & 1.81 & 1.59 & 1.28 & 337.23 & 326.61 & 295.94 \\
\textbf{Quadruped} & \textcolor{red}{\textbf{959.75}} & \textcolor{blue}{\textbf{958.44}} & \textcolor{green!50!black}{\textbf{829.41}} & 160.19 & 159.80 & 159.05 & 339.49 & 299.32 & 216.92 & 484.25 & 455.86 & 345.87 \\
\textbf{Trading} & \textcolor{red}{\textbf{37.72}} & \textcolor{blue}{\textbf{33.09}} & \textcolor{green!50!black}{\textbf{31.98}} & 23.46 & 23.37 & 22.32 & 12.22 & 2.37 & -39.94 & 6.53 & -0.11 & -6.22 \\
\bottomrule
\end{tabular}}
\end{subtable}
\end{table}

\subsection{Reward shaping with increment modeling}
We also evaluate a reward-shaping variant (\emph{increment modeling}) defined as $r_t^{\text{new}} = \Delta_t\, r_t$, to test whether scaling rewards by the observed time increment improves learning under non-uniform time steps. Across control tasks and the trading task, this modification does not improve performance. In most cases, it is worse or statistically indistinguishable from the unshaped baseline, suggesting that naive time-scaling can distort the learning signal and increase variance rather than stabilizing training (See \cref{fig:increment_modeling_control_tasks_plot,fig:increment_modeling_trading_task_plot}).

\begin{figure*}[ht]
    \centering
    \includegraphics[scale=0.45]{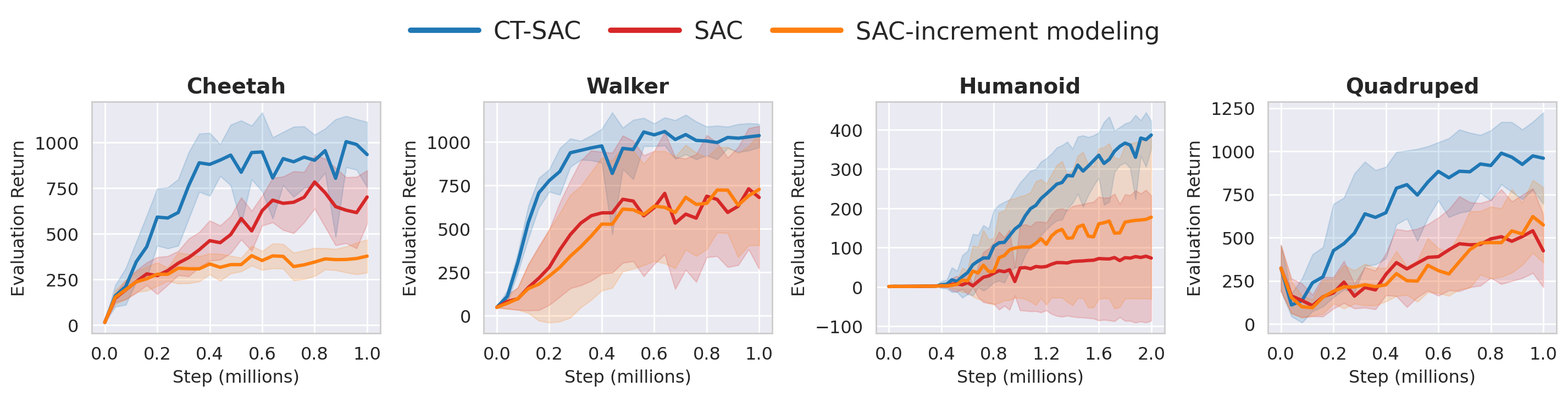}
    \caption{Evaluation returns for CT-SAC against SAC and its reward-shaping version on control tasks over 12 seeds.}
    \label{fig:increment_modeling_control_tasks_plot}
\end{figure*}

\begin{figure*}[ht]
    \centering
    \includegraphics[scale=0.45]{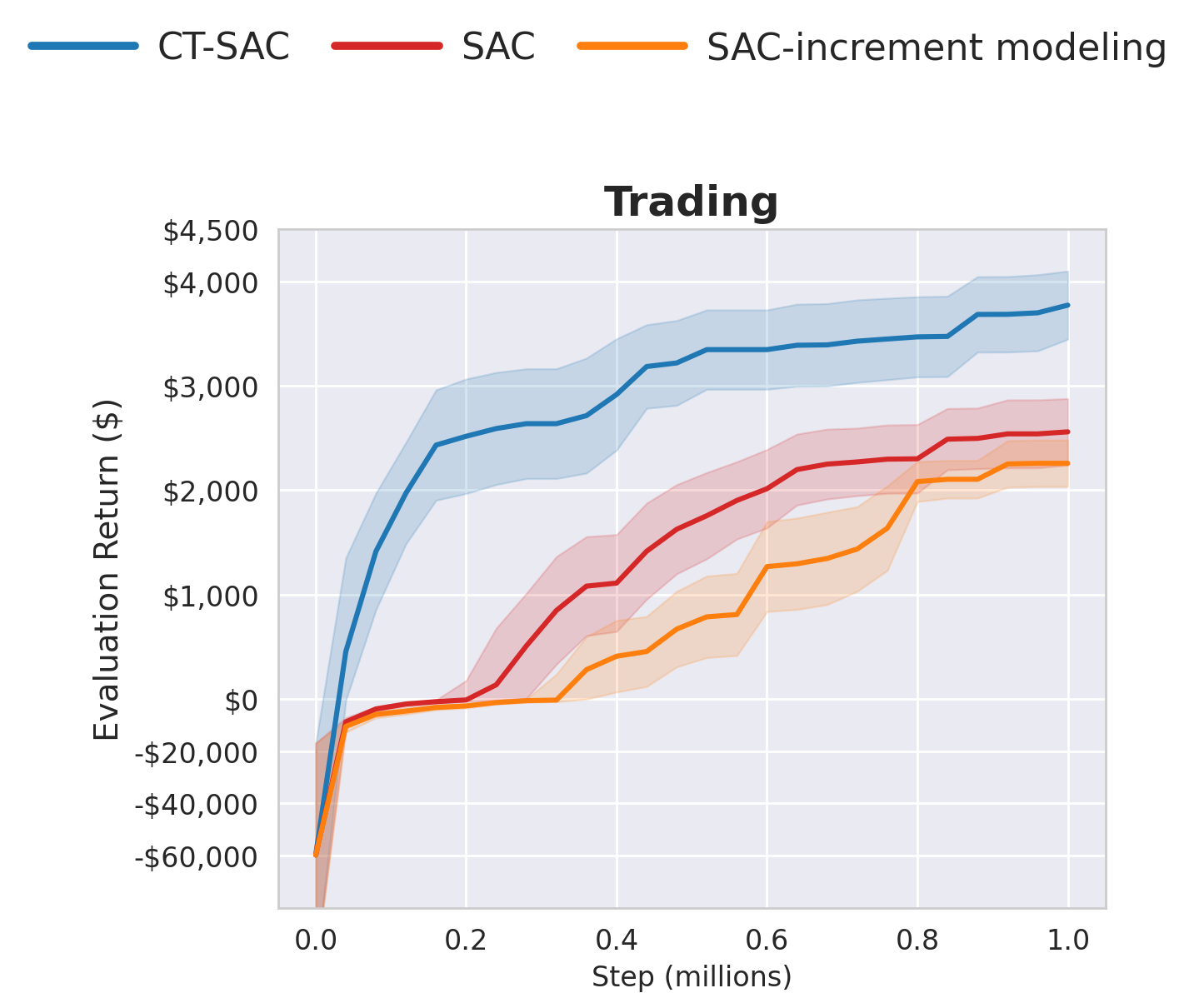}
    \caption{Evaluation returns for CT-SAC against SAC and its reward-shaping version on trading task over 12 seeds.}
    \label{fig:increment_modeling_trading_task_plot}
\end{figure*}

\subsection{Performance of all algorithms}
To facilitate comparison across all eight algorithms, \cref{fig:all_algorithms_control_tasks} reports the \textbf{mean evaluation return only} (without variance bands). This view highlights the overall ranking and learning dynamics without clutter; variance plots and full per-seed statistics are already provided in the main paper and in the appendix.

\begin{figure*}[ht]
    \centering
    \includegraphics[scale=0.45]{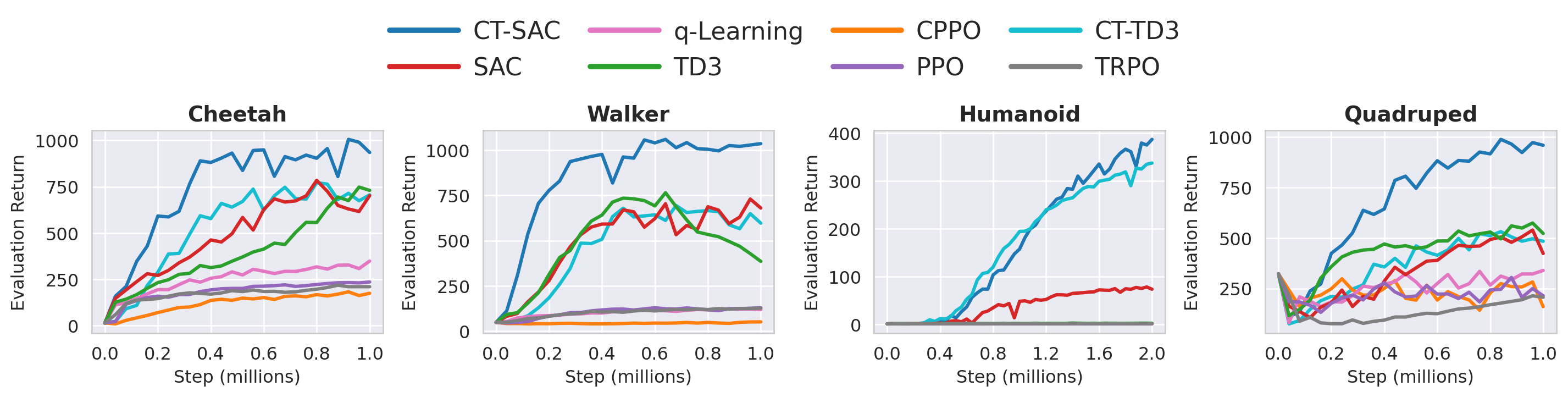}
    \caption{Evaluation returns for all algorithms on the control task over 12 seeds.}
    \label{fig:all_algorithms_control_tasks}
\end{figure*}

\subsection{Runtime and time complexity}
\textbf{Runtime.} We summarize the runtime in \cref{table:runtime}, reported as mean $\pm$ std over 12 seeds for the \textbf{top} configuration of each algorithm. Overall, \textbf{CT-SAC has runtime comparable to SAC}, and \textbf{CT-TD3 is comparable to TD3}, indicating that our continuous-time modifications do not materially increase training cost relative to their discrete-time counterparts.

\textbf{Time complexity.} CT-SAC follows the same training structure as SAC: sampling rollouts and performing batched critic/actor updates at a fixed update schedule. Let $N$ denote the number of environment interaction steps (or sampled decision times), $f$ the update frequency (updates every $f$ steps), and $d$ the number of gradient updates per update event. Then the total number of update steps scales as $\mathcal{O}(N d / f)$. Per update, CT-SAC and SAC share the same dominant costs (forward/backward passes through the actor/critic networks). CT-SAC additionally includes time operations (e.g., dividing by $\Delta_t$ or using $\Delta_t$ in targets), which may increase constants slightly but does not change asymptotic complexity.

\begin{table*}[ht]
\centering
\caption{Training runtime (hours) for each algorithm on each task, reported as mean$\pm$std over $n{=}12$ seeds.}
\label{table:runtime}
\small
\setlength{\tabcolsep}{4pt}
\renewcommand{\arraystretch}{1.12}
\resizebox{\textwidth}{!}{%
\begin{tabular}{l c c c c c c c c}
\toprule
\textbf{Task} & \textbf{CT-SAC} & \textbf{SAC} & \textbf{TD3} & \textbf{TRPO} & \textbf{PPO} & \textbf{CPPO} & \textbf{q-Learning} & \textbf{CT-TD3} \\
\midrule
\textbf{Cheetah}   & $10.62 \pm 0.53$ & $6.29 \pm 0.04$ & $3.98 \pm 0.05$ & $0.44 \pm 0.07$ & $\mathbf{0.29 \pm 0.07}$ & $0.57 \pm 0.06$ & $9.12 \pm 0.17$ & $4.60 \pm 0.04$ \\
\textbf{Walker}    & $9.69 \pm 0.83$  & $8.90 \pm 0.54$ & $4.57 \pm 0.24$ & $\mathbf{0.40 \pm 0.00}$ & $0.41 \pm 0.07$ & $0.61 \pm 0.04$ & $5.59 \pm 0.08$ & $4.68 \pm 0.09$ \\
\textbf{Humanoid} & $8.30 \pm 0.50$  & $6.02 \pm 0.09$ & $3.87 \pm 0.06$ & $0.82 \pm 0.01$ & $\mathbf{0.48 \pm 0.02}$ & $1.01 \pm 0.02$ & $5.07 \pm 0.14$ & $5.32 \pm 0.16$ \\
\textbf{Quadruped} & $8.21 \pm 0.41$  & $6.49 \pm 0.04$ & $4.66 \pm 0.04$ & $0.40 \pm 0.02$ & $0.41 \pm 0.01$ & $\mathbf{0.36 \pm 0.01}$ & $6.87 \pm 0.32$ & $2.65 \pm 1.91$ \\
\textbf{Trading}       & $8.42 \pm 0.11$  & $7.42 \pm 0.07$ & $4.90 \pm 0.08$ & $0.81 \pm 0.00$ & $\mathbf{0.52 \pm 0.01}$ & $0.89 \pm 0.17$ & $6.38 \pm 0.10$ & $5.44 \pm 0.08$ \\
\bottomrule
\end{tabular}}
\end{table*}

\subsection{Hyperparameters tuning}
We perform hyperparameter search to select the best configuration for each algorithm--task pair. Because the full sweep spans 8 algorithms $\times$ 5 tasks with distinct search spaces (in total roughly 40 algorithm--task search families), we place the complete search ranges and the selected \textbf{top/second/third} settings in the codebase under the \texttt{out/final\_reports/hyperparam\_spaces} folder.

\subsection{Statistical testing}
We report statistical comparisons across 12 independent seeds. We use the same seeds across all algorithms and tasks. We run two complementary tests using \texttt{scipy.stats}: Welch's t-test and a paired test. Across tasks, these tests support that \textbf{CT-SAC outperforms the baselines} with statistically significant differences in the majority of settings. The full set of $p$-values and other details are included in \texttt{out/final\_reports/significant\_testing.csv} (available in the codebase repository).

\subsection{Codebase and reproducibility}
Our codebase is available at the GitHub link: \texttt{https://github.com/mpnguyen2/ct-rl}. Additional plots/media/tables are under \texttt{out/final\_reports} folder. 

Because training logs and saved checkpoints are large ($\sim$6\,GB), we provide them via the SwissTransfer links, given in the codebase: \texttt{https://github.com/mpnguyen2/ct-rl}. A step-by-step reproduction guide (exact commands, dependencies, seeds, and evaluation protocol) is documented in \texttt{README.md}, including how to regenerate the tables and figures from raw logs.

\subsection{Generalization from irregular-time training to regular-time evaluation}
We further evaluate whether policies trained under \emph{irregular} decision times transfer to the standard \emph{regular}-time evaluation setting. This test checks whether the learned policies capture the underlying control structure rather than overfitting to a particular irregular sampling pattern. \cref{table:mean_return_summary_control} reports mean evaluation return over 12 seeds on the control tasks. CT-SAC achieve the best performance on all tasks, followed closely by CT-TD3. This indicates strong generalization from complex, noisy, intertwined irregular-time trajectories to the regular-time setting.

\begin{table}[h]
\centering
\caption{Regular-time evaluation of policies trained under irregular time steps. We report mean return over 12 seeds on control tasks.}
\label{table:mean_return_summary_control}
\small
\setlength{\tabcolsep}{5pt}
\renewcommand{\arraystretch}{1.12}
\begin{tabular}{l c c c c c c c c}
\toprule
\textbf{Task} & \textbf{CT-SAC} & \textbf{SAC} & \textbf{TD3} & \textbf{PPO} & \textbf{TRPO} & \textbf{q-Learning} & \textbf{CPPO} & \textbf{CT-TD3} \\
\midrule
\textbf{Cheetah}   &
\textcolor{red}{\textbf{605.05}} &
\textcolor{green!50!black}{\textbf{455.74}} &
404.18 & 124.45 & 104.46 & 208.04 & 102.03 &
\textcolor{blue}{\textbf{513.39}} \\
\textbf{Walker}    &
\textcolor{red}{\textbf{610.03}} &
403.44 &
\textcolor{green!50!black}{\textbf{509.23}} &
73.06 & 64.55 & 70.94 & 25.04 &
\textcolor{blue}{\textbf{512.30}} \\
\textbf{Humanoid}  &
\textcolor{red}{\textbf{404.64}} &
\textcolor{green!50!black}{\textbf{77.14}} &
2.27 & 1.34 & 1.35 & 1.90 & 1.42 &
\textcolor{blue}{\textbf{356.90}} \\
\textbf{Quadruped} &
\textcolor{red}{\textbf{530.13}} &
\textcolor{blue}{\textbf{328.30}} &
\textcolor{green!50!black}{\textbf{325.87}} &
159.52 & 121.84 & 192.86 & 125.30 & 284.23 \\
\bottomrule
\end{tabular}
\end{table}

Finally, we provide additional visualization for the trading task (\cref{fig:irregular_trading}) and for control tasks under both the irregular setting (Figures~\ref{fig:irregular_ctsac_vs_continuous_last} and~\ref{fig:irregular_ctsac_vs_discrete_last}) and the regular setting (Figures~\ref{fig:regular_ctsac_vs_continuous_last} and~\ref{fig:regular_ctsac_vs_discrete_last}). Further implementation and evaluation details are provided in \cref{sec:appendix:further_experiment_details}.

\section{Preliminaries for theoretical analysis}\label{sec:appendix:preliminaries}
\paragraph{State, action, and function space.}
Let $\mathcal{S} \subseteq \mathbb{R}^d$ be the state space and let $\mathcal{A}$ be a compact metric action space.
For value functions, we work on the Banach space $\mathcal{D}:=C_b(\mathbb{R}^d)$ of bounded continuous functions equipped with the sup-norm
$\|V\|_\infty := \sup_{x\in\mathbb{R}^d} |V(x)|$.

For convenience, we recall all notations and definitions introduced in \cref{sec:preliminaries}, \cref{sec:formulation}, and \cref{sec:theory} in the main paper.

\subsection{Notation}
\label{sec:appendix:notation}

\setlength{\LTpre}{6pt}
\setlength{\LTpost}{6pt}
\small
\begin{longtable}{@{\hspace{0.05\textwidth}}>{\raggedright\arraybackslash}m{0.15\textwidth}p{0.66\textwidth}@{}}

\caption{Notations used throughout the paper.}
\label{table:notation}\\
\toprule
\textbf{Symbol} & \textbf{Meaning} \\
\midrule
\endfirsthead

\toprule
\textbf{Symbol} & \textbf{Meaning} \\
\midrule
\endhead

\bottomrule
\endfoot

$V^{(\alpha)}(x)$ & Optimal value function with temperature $\alpha$. \\
$V(x)$ & Generic value function (not necessarily optimal). \\
$q_V(x,a)$ & Advantage-rate $q$-function associated with $V$ at state--action $(x,a)$. \\
$q_V^{u}(x,a)$ & Finite-time approximation of $q_V(x,a)$ using time increment $u$. \\
$\tilde q_V^{u}(x,a)$ & Richardson-interpolated version of $q_V^{u}(x,a)$. \\
$L^{a}$ & Infinitesimal generator under fixed action $a$. \\
$H_a(V)$ & Action-wise Hamiltonian at $a$, defined via $L^{a}V$. \\
$H^{(0)}(V)$ & Hard Hamiltonian (standard max aggregation over actions). \\
$H^{(\alpha)}(V)$ & Soft Hamiltonian with temperature $\alpha$. \\
$H^{(\alpha)}_{u}(V)$ & Soft Hamiltonian using the finite-time approximation $q_V^{u}$. \\
$\tilde H^{(\alpha)}_{u}(V)$ & Soft Hamiltonian using the Richardson approximation $\tilde q_V^{u}$. \\
$\mathcal{Q}^{(\alpha)}$ & Soft aggregation operator acting on $q$-functions (temperature $\alpha$). \\
$Q(x,a)$ & Numerical $Q$-function defined by $Q := V + q_V$. \\
$T_{\tau}^{(\alpha)}(V)$ & Picard--Hamiltonian operator: $T_{\tau}^{(\alpha)}(V):= V+\tau H^{(\alpha)}(V)$. \\
$\Phi_{\tau}^{(\alpha)}$ & Probabilistic dynamic-programming semigroup (step size $\tau$, temperature $\alpha$). \\
$X_t$ & State process at time $t$. \\
$a_t$ & Action at time $t$. \\
$W_t$ & Brownian motion at time $t$. \\
$b,\sigma$ & Drift and diffusion coefficients of the controlled SDE. \\
$\tilde b,\tilde\sigma,\tilde X_t$ & Policy-averaged drift/diffusion and the induced averaged state process. \\
$L^{\pi}$ & Policy-averaged generator: $L^{\pi}=\mathbb{E}_{a\sim\pi(\cdot|x)}[L^{a}]$. \\
$P_t^{a}$ & Markov semigroup of the diffusion under fixed action $a$. \\
$\mathcal{F}_{\tau,u}$ & Single-critic update operator for the $Q$-function (step size $\tau$, increment $u$). \\
$\mathcal{F}^{\mathrm{Rich}}_{\tau,u}$ & Richardson version of the single-critic update operator $\mathcal{F}_{\tau,u}$. \\
$\bar H_{U}^{(\alpha)}(V)$ & Random-increment Hamiltonian: $\bar H_{U}^{(\alpha)}(V)=\mathbb{E}[H_{U}^{(\alpha)}(V)]$ for random $U$. \\
$\tilde T_{\tau}^{(\alpha)}(V)$ & Random-time Picard--Hamiltonian operator corresponding to $\bar H_{U}^{(\alpha)}(V)$. \\
\end{longtable}

\subsection{Controlled diffusion dynamics}\label{sec:appendix-dynamics}
Given a progressively measurable control process $(a_t)_{t\ge 0}$ taking values in $\mathcal{A}$,
we consider the controlled stochastic differential equation (SDE)
\begin{equation}\label{eq:controlled-sde}
  dX_t = b(X_t,a_t)\,dt + \sigma(X_t,a_t)\,dW_t,
  \qquad X_0=x\in\mathbb{R}^d,
\end{equation}
where $W_t$ is a $d$-dimensional Brownian motion and
$b:\mathbb{R}^d\times\mathcal{A}\to\mathbb{R}^d$ and
$\sigma:\mathbb{R}^d\times\mathcal{A}\to\mathbb{R}^{d\times d}$ are measurable.

\begin{assumption}[Dynamics regularity]\label{assumption:dynamics}
The drift $b$ and diffusion $\sigma$ satisfy:
\begin{enumerate}
  \item \textbf{Global Lipschitz and linear growth.}
  There exists $C>0$ such that for all $x,y\in\mathbb{R}^d$ and $a\in\mathcal{A}$,
  \[
    \|b(x,a)-b(y,a)\| + \|\sigma(x,a)-\sigma(y,a)\|
    \le C\|x-y\|,
    \qquad
    \|b(x,a)\| + \|\sigma(x,a)\| \le C\,(1+\|x\|).
  \]
  \item \textbf{Uniform nondegeneracy and boundedness.}
  There exist $0<\underline{\sigma}\le\overline{\sigma}<\infty$ such that
  \[
    \underline{\sigma}^2 I
    \;\preceq\;
    \sigma(x,a)\sigma(x,a)^\top
    \;\preceq\;
    \overline{\sigma}^2 I
    \qquad \forall x\in\mathbb{R}^d,\ \forall a\in\mathcal{A}.
  \]
\end{enumerate}
\end{assumption}

\begin{assumption}[Reward and policy regularity]\label{assumption:reward}
The running reward $r:\mathbb{R}^d\times\mathcal{A}\to\mathbb{R}$ satisfies:
\begin{enumerate}
  \item $r(\cdot,a)$ is bounded and is globally Lipschitz uniformly in $a$, i.e., there exist $R_{\max}<\infty$  and $L_r<\infty$ such that for all $x,y$ and all $a$,
  \[
    |r(x,a)|\le R_{\max}, \qquad |r(x,a)-r(y,a)| \le L_r\|x-y\|
  \]
  \item For the entropy term, admissible Markov policies $\pi(\cdot\mid x)$ satisfy
  that $\log\pi(a\mid x)$ is well-defined and uniformly integrable under the
  induced state-action distribution.
\end{enumerate}
\end{assumption}

\begin{assumption}[Regularity of value functions]\label{assumption:test-functions}
There exists a class $\mathcal{D}\subset C_b^2(\mathbb{R}^d)$ of value functions such that for every $V\in\mathcal{D}$:
\begin{enumerate}
  \item For each fixed $a\in\mathcal{A}$, the generator $L^a V$ is globally Lipschitz in $x$,
  uniformly in $a$.
  \item For $\alpha>0$, we restrict to policies $\pi$ for which
  $x\mapsto \log\pi(a\mid x)$ is globally Lipschitz uniformly in $a$.
\end{enumerate}
\end{assumption}

Throughout this work, $\mathcal{D}$ refers to the domain of all value functions. Under Assumptions~\ref{assumption:dynamics}--\ref{assumption:test-functions}, the SDE \eqref{eq:controlled-sde} admits a unique strong solution for any admissible control,
and the entropy-regularized HJB equation admits a unique bounded viscosity solution
$V^{(\alpha)}\in\mathcal{D}$ \cite{Fleming2006-wa, wang2020}.

\subsection{Policies and value functions}\label{sec:appendix-values}
Fix a discount $\beta>0$ and an entropy coefficient $\alpha\ge 0$.
A Markov policy $\pi$ maps each state $x$ to a probability distribution $\pi(\cdot\mid x)$ on $\mathcal{A}$,
and we write $\E_x^\pi[\cdot]$ for expectation under \eqref{eq:controlled-sde} when
$a_t\sim \pi(\cdot\mid X_t)$.

Recall that the entropy-regularized objective of policy $\pi$ is:
\begin{equation}\label{eq:entropy-regularized-value}
  V^{(\alpha)}(x;\pi)
  :=
  \E_x^\pi\!\left[
    \int_0^\infty e^{-\beta t}
      \Big(r(X_t,a_t)-\alpha\log\pi(a_t\mid X_t)\Big)\,dt
  \right],
\end{equation}
and the optimal value function is $V^{(\alpha)}(x):=\sup_\pi V^{(\alpha)}(x;\pi)$. The case $\alpha=0$ recovers the standard (non-regularized) value function. Additionally, for notational convenience, if $\pi$ and $x$ is clear from the contex, we simply use $\E[.]$.

\subsection{Infinitesimal generator and $q$-functions}\label{sec:appendix-generator}
Recall the following notations from \cref{sec:preliminaries} and \cref{sec:formulation}. For a fixed action $a\in\mathcal{A}$, the controlled infinitesimal generator $L^a$ acting on $\varphi\in C_b^2(\mathbb{R}^d)$ is:
\begin{equation}\label{eq:generator-fixed-action}
  (L^a \varphi)(x)
  :=
  b(x,a)\cdot\nabla \varphi(x)
  + \frac{1}{2}\,\mathrm{Tr}\!\Big(
      \sigma(x,a)\sigma(x,a)^\top \nabla^2\varphi(x)
    \Big).
\end{equation}
Given $V\in\mathcal{D}$, the $q$-function has the form:
\begin{equation}\label{eq:def-qV}
  q_V(x,a)
  := r(x,a) + (L^a V)(x) - \beta V(x).
\end{equation}

For a short horizon $u>0$, the finite time estimation for $q$ under constant action $a$ is:
\begin{equation}\label{eq:def-qVu}
  q_V^u(x,a):=\frac{e^{-\beta u}\,\E_x^a[V(X_u)]- V(x)}{u} + r(x, a),
\end{equation}
where $\E_x^a[\cdot]$ denotes expectation for the SDE \eqref{eq:controlled-sde} with constant control $a_t\equiv a$.

\subsection{Dynkin's formula and small-time moments}\label{sec:appendix-dynkin-moments}

\begin{lemma}[Dynkin's formula for controlled diffusions]\label{lemma:dynkin-formula}
Let $(X_t)_{t\ge 0}$ solve \eqref{eq:controlled-sde} under a progressively measurable control $(a_t)_{t\ge 0}$
and let $\tau$ be a bounded stopping time. Then for any $f\in C_b^2(\mathbb{R}^d)$,
\begin{equation}\label{eq:dynkin-formula}
  \E\big[f(X_\tau)\big]
  =
  f(x)
  + \E\!\left[\int_0^\tau (L^{a_t}f)(X_t)\,dt\right],
\end{equation}
where $(L^{a_t}f)(X_t)$ is understood pointwise using \eqref{eq:generator-fixed-action}. Additionally, under the diffusion without control: $dX_t = b(X_t)\,dt + \sigma(X_t)\,dW_t$, \cref{eq:dynkin-formula} also holds with the (uncontrolled) infinitesimal generator $L\varphi(x):=b(x)\cdot\nabla \varphi(x) + \frac{1}{2}\,\mathrm{Tr}\!\Big( \sigma(x)\sigma(x)^\top \nabla^2\varphi(x) \Big).$ This can easily be seen by treating drift $b(x)$ as $b(x, a_0)$, and diffusion $\sigma(x)$ as $\sigma(x, a_0)$ for a fixed constant action $a_0$.
\end{lemma}

\begin{proof}
This is a direct consequence of It\^o's formula applied to $f(X_t)$ and taking expectations \cite{Fleming2006-wa}.
\end{proof}

We now recall a standard result in stochastic processes.
\begin{lemma}[Small-time moment bounds]\label{lemma:moment-bound}
Under Assumption~\ref{assumption:dynamics}, for $X_t$ in \cref{eq:controlled-sde}, there exists a constant $C>0$ such that
for all $t\in(0,1]$ and for any admissible control process $(a_s)_{0\le s\le t}$,
\begin{equation}\label{eq:moment-bound-mean}
  \E_x\big[\|X_t-x\|\big] \le C\sqrt{t},
\end{equation}
and moreover
\begin{equation}\label{eq:moment-bound-sup}
  \E_x\!\left[\sup_{0\le s\le t}\|X_s-x\|^2\right] \le C t,
  \qquad
  \E_x\!\left[\sup_{0\le s\le t}\|X_s-x\|\right] \le C\sqrt{t}.
\end{equation}
The constant $C$ depends only on the constants in Assumption~\ref{assumption:dynamics} and $d$,
and is independent of the choice of control.
\end{lemma}

\begin{proof}[Proof sketch]
We have $ X_t - x = \int_0^t b(X_s,a_t)\,du + \int_0^s \sigma(X_t,a_t)\,dW_u$. Because of the linear growth, we get  $\|b(x,a)\|^2 \le K(1+\|x\|^2)$ and $\|\sigma(x,a)\|_F^2 \le K(1+\|x\|^2)$ for some constant $K$. From here, we bound the second moment through It\^o isometry and Cauchy-Schwarz inequalities:
\[
\begin{aligned}
\mathbb{E}\|X_t - x\|^2
&\le 2\mathbb{E}\Big\|\int_0^t b(X_s,a_s)\,ds\Big\|^2
   + 2\mathbb{E}\Big\|\int_0^t \sigma(X_s,a_s)\,dW_s\Big\|^2 \\
&\le 2t \int_0^t \mathbb{E}\|b(X_s,a_s)\|^2 ds
   + 2 \int_0^t \mathbb{E}\|\sigma(X_s,a_s)\|_F^2 ds \\
&\le 2t \int_0^t K(1 + \norm{X_s}^2) ds + 2 \int_0^t K(1 + \norm{X_s}^2) ds \le C_1 t,
\end{aligned}
\]
as we can bound $\norm{X_s}^2$ easily through an argument using Gronwall's lemma. Hence, by Cauchy-Schwarz again, we obtain the desired bound. For the supremum, note that with
$A_s:=\int_0^s b(X_u,a_u)\,du$ and $M_s:=\int_0^s \sigma(X_u,a_u)\,dW_u$, we have
\[
  \sup_{0\le s\le t}\|X_s-x\|^2
  \le 2\sup_{s\le t}\|A_s\|^2 + 2\sup_{s\le t}\|M_s\|^2,
\]
We can then apply the Burkholder--Davis--Gundy inequality to the martingale part $M_t$ and use similar estimates to obtain desired bound.
\end{proof}

\subsection{Stochastic control under relaxed (policy-averaged) dynamics}\label{sec:appendix-relaxed-dynamics}
When the action is sampled from a Markov policy $\pi(\cdot\mid X_t)$, it is convenient to use an equivalent diffusion with drift and covariance averaged under $\pi$. For a fixed Markov policy $\pi$, define:
\begin{equation}\label{eq:relaxed-drift}
  \tilde b(x;\pi) := \int_{\mathcal{A}} b(x,a)\,\pi(da\mid x),
\end{equation}

Here $\pi(da|x)$ is the measured induced by the policy $\pi(.|x)$ on $\mathcal{A}$. We will use this notation rather than $\int_{a \in \pi(.|x)} b(x, a) da$ from now on, as it is more aligned with the measure-theoretic nature of the theoretical details. We also define $\tilde\sigma(x;\pi)$ via the covariance identity
\begin{equation}\label{eq:relaxed-diffusion}
  \tilde\sigma(x;\pi)\tilde\sigma(x;\pi)^\top
  :=
  \int_{\mathcal{A}} \sigma(x,a)\sigma(x,a)^\top\,\pi(da\mid x).
\end{equation}
Under Assumption~\ref{assumption:dynamics} and the mild Lipschitz regularity of $\pi$ in
Assumption~\ref{assumption:test-functions}, there exists a filtered probability space supporting a Brownian motion
$\tilde W$ and a process $(\tilde X_t)_{t\ge 0}$ solving
\begin{equation}\label{eq:sde-relaxed}
  d\tilde X_t = \tilde b(\tilde X_t;\pi)\,dt + \tilde\sigma(\tilde X_t;\pi)\,d\tilde W_t,
  \qquad \tilde X_0=x,
\end{equation}
such that, for each $t\ge 0$, the law of $\tilde X_t$ coincides with the law of $X_t$ under the original controlled system \eqref{eq:controlled-sde} with $a_t\sim \pi(\cdot\mid X_t)$ \cite{wang2020,tang2021}. Hence distributional quantities (value functions, occupancy measures, etc.) can be computed using \eqref{eq:sde-relaxed}.

For $\varphi\in C_b^2(\mathbb{R}^d)$, we also define the relaxed infinitesimal generator with two equivalent representations:
\begin{equation}\label{eq:generator-relaxed}
  (L^\pi \varphi)(x)
  :=
  \int_{\mathcal{A}} (L^a\varphi)(x)\,\pi(da\mid x)
  =
  \tilde b(x;\pi)\cdot\nabla\varphi(x)
  + \frac12\,\mathrm{Tr}\!\Big(
      \tilde\sigma(x;\pi)\tilde\sigma(x;\pi)^\top \nabla^2\varphi(x)
    \Big).
\end{equation}

Finally, we may write the policy value \eqref{eq:entropy-regularized-value} equivalently as
\begin{equation}\label{eq:value-relaxed}
  V^{(\alpha)}(x;\pi)
  =
  \E\!\left[
    \int_0^\infty e^{-\beta t}\,
      \tilde r^{(\alpha)}(\tilde X_t;\pi)\,dt
    \,\Big|\, \tilde X_0=x
  \right],
\end{equation}
where the relaxed running reward is
\begin{equation}\label{eq:relaxed-reward}
  \tilde r^{(\alpha)}(x;\pi)
  :=
  \int_{\mathcal{A}}
    \Big(r(x,a)-\alpha\log\pi(a\mid x)\Big)\,\pi(da\mid x).
\end{equation}

Note that the drift function $\tilde b$ defined by averaging over policy randomness is similar to the result obtained by \cite{munos06b} that we mentioned in \cref{sec:prelim:mdp-to-sde}. Here we also have the diffusion (Brownian motion) part. See \cite{wang2020,tang2021,Jia2022-yg} for more details on the policy-averaged SDE.

Lastly, we would like to state an useful lemma that will be used throughout Appendices~\ref{sec:appendix:value_convergence}, \ref{sec:appendix:q_convergence}, and \ref{sec:appendix:algorithm_convergence}.
\begin{lemma}[Entropy-regularized expectation via a KL shift]\label{lemma:entropy-KL-identity}
Fix $\alpha>0$ and let $q:\mathcal{X}\times\mathcal{A}\to\mathbb{R}$ be any measurable function. For each $x\in\mathcal{X}$ recall the functional $\mathcal{Q}$ on $q$ (see \cref{sec:formulation:decoupled}) that:
\[
  \mathcal{Q}^{(\alpha)}(q)(x)=\alpha\log\left(\int_{\mathcal{A}}\exp\!\Big(q(x,a)/\alpha\Big)\,da\right) = \alpha \log Z_q(x),
\]
with partition function $Z_q$, and the Boltzmann distribution $\pi_q$:
\[
  \pi_q(a\mid x):=\frac{\exp(q(x,a)/\alpha)}{Z_q(x)}, \quad Z_q(x):= \int_{\mathcal{A}}\exp\!\Big(q(x,a)/\alpha\Big)\,da.
\]
Then for any policy $\pi(\cdot\mid x)$,
\begin{equation}\label{eq:entropy-KL-identity}
\alpha\log Z_q(x) - \alpha\,\mathrm{KL}\!\big(\pi(\cdot\mid x)\,\|\,\pi_q(\cdot\mid x)\big)
=
\E_{a\sim\pi(\cdot\mid x)}\!\big[q(x,a)-\alpha\log\pi(a\mid x)\big].
\end{equation}
In particular,
\begin{equation}\label{eq:gibbs-variational}
\mathcal{Q}^{(\alpha)}(q)(x) = \alpha\log Z_q(x) = \sup_{\pi(\cdot\mid x)}
\E_{a\sim\pi(\cdot\mid x)}\!\big[q(x,a)-\alpha\log\pi(a\mid x)\big],
\end{equation}
and the supremum is achieved uniquely by $\pi=\pi_q$.
\end{lemma}

\begin{proof}
By definition,
\[
\mathrm{KL}(\pi\|\pi_q)
=
\int_{\mathcal{A}}\pi(a\mid x)\log\frac{\pi(a\mid x)}{\pi_q(a\mid x)}\,da
=
\int_{\mathcal{A}}\pi(a\mid x)\log\pi(a\mid x)\,da
-
\int_{\mathcal{A}}\pi(a\mid x)\log\pi_q(a\mid x)\,da.
\]
Since $\log\pi_q(a\mid x)=\tfrac{1}{\alpha}q(x,a)-\log Z_q(x)$, we obtain
\[
\mathrm{KL}(\pi\|\pi_q)
=
\int_{\mathcal{A}}\pi(a\mid x)\log\pi(a\mid x)\,da
-\frac{1}{\alpha}\E_{\pi}[q(x,a)]
+\log Z_q(x).
\]
Rearranging terms gives us the equation \eqref{eq:entropy-KL-identity}.
Finally, \eqref{eq:gibbs-variational} follows because $\mathrm{KL}(\pi\|\pi_q)\ge 0$ with equality iff
$\pi=\pi_q$.
\end{proof}
\section{Proof of value function convergence (\cref{thm:value_convergence})}\label{sec:appendix:value_convergence}

This appendix proves the convergence result for the \textbf{ideal} Picard--Hamiltonian value iteration $V_{k+1}=T_\tau^{(\alpha)}(V_k)$ in \cref{thm:value_convergence}. The key idea is to compare the Picard update to the short-horizon dynamic-programming operator $\Phi_\tau^{(\alpha)}$, which is a strict contraction in the sup-norm.

\subsection{Dynamic-programming operator and contraction}\label{sec:appendix-dp-contraction}

For $\tau>0$ and $\alpha\ge 0$, recall the entropy-regularized dynamic-programming operator $\Phi_\tau^{(\alpha)}:\mathcal{D}\to \mathcal{D}$ defined earlier:
\begin{equation}\label{eq:appendix-Phi-definition}
  (\Phi_\tau^{(\alpha)}V)(x)
  :=
  \sup_{\pi}\,
  \E_x^\pi\!\left[
    \int_0^\tau e^{-\beta t}
      \Big(r(X_t,a_t)-\alpha\log\pi(a_t\mid X_t)\Big)\,dt
    + e^{-\beta\tau}V(X_\tau)
  \right],
\end{equation}
where $X$ follows the controlled SDE \eqref{eq:controlled-sde} and $a_t\sim\pi(\cdot\mid X_t)$.

\begin{lemma}[Contraction of $\Phi_\tau^{(\alpha)}$]\label{lemma:Phi-contraction}
Assume Assumption~\ref{assumption:dynamics}. Then for all $\tau>0$ and $\alpha\ge 0$,
\begin{equation}\label{eq:appendix-Phi-contraction}
  \|\Phi_\tau^{(\alpha)}(V)-\Phi_\tau^{(\alpha)}(W)\|_\infty
  \le e^{-\beta\tau}\,\|V-W\|_\infty,
  \qquad \forall V,W\in X.
\end{equation}
Consequently, $\Phi_\tau^{(\alpha)}$ admits a unique fixed point $V^{(\alpha)}\in X$ and
\begin{equation}\label{eq:appendix-Phi-iterate-bound}
  \|(\Phi_\tau^{(\alpha)})^{(k)}(V_0)-V^{(\alpha)}\|_\infty
  \le e^{-\beta\tau k}\,\|V_0-V^{(\alpha)}\|_\infty,
  \qquad \forall k\ge 0.
\end{equation}
\end{lemma}

\begin{proof}
Fix $x\in\mathbb{R}^d$ and any policy $\pi$. Since the running reward terms cancel,
\[
\begin{aligned}
  &\left|
  \E_x^\pi\!\left[\int_0^\tau e^{-\beta t}(r-\alpha\log\pi)\,dt+e^{-\beta\tau}V(X_\tau)\right]
  -
  \E_x^\pi\!\left[\int_0^\tau e^{-\beta t}(r-\alpha\log\pi)\,dt+e^{-\beta\tau}W(X_\tau)\right]
  \right| \\
  &= e^{-\beta\tau}\,\Big|\E_x^\pi\big[V(X_\tau)-W(X_\tau)\big]\Big|
  \le e^{-\beta\tau}\,\|V-W\|_\infty.
\end{aligned}
\]
so $\Phi_\tau^{(\alpha)}$ is a contraction on bounded functions with modulus $e^{-\beta\tau}<1$. By the Banach fixed-point theorem, it admits a unique fixed point $V^\star$, and $(\Phi_\tau^{(\alpha)})^{(k)} (V_0)$ converge to $V^\star$ at a geometric rate. Finally, the entropy-regularized value function $V^{(\alpha)}$ satisfies the dynamic programming principle, hence $V^{(\alpha)}$ is a fixed point of $\Phi_\tau^{(\alpha)}$. By uniqueness of the fixed point, $V^\star = V^{(\alpha)}$, and \cref{eq:appendix-Phi-iterate-bound} follows trivially.
\end{proof}

\subsection{Picard--Hamiltonian operator}\label{sec:appendix-picard-operator}
Recall the infinitesimal rate $q_V$ from \eqref{eq:def-qV}:
\[
  q_V(x,a)=r(x,a)+(L^aV)(x)-\beta V(x),
\]
and the (soft/hard) Hamiltonian $H^{(\alpha)}(V):\mathbb{R}^d\to\mathbb{R}$:
\begin{equation}\label{eq:appendix-Hamiltonian-definition}
  H^{(\alpha)}(V)(x)
  :=
  \begin{cases}
  \alpha \log\displaystyle\int_{\mathcal{A}}
    \exp\!\Big(\frac{q_V(x,a)}{\alpha}\Big)\,da,
  & \alpha>0,\\[2.0ex]
  \displaystyle\sup_{a\in\mathcal{A}} q_V(x,a),
  & \alpha=0,
  \end{cases}
\end{equation}
The \textbf{ideal} Picard--Hamiltonian operator has the form:
\begin{equation}\label{eq:appendix-Ttau-definition}
  (T_\tau^{(\alpha)}V)(x) := V(x) + \tau\,H^{(\alpha)}(V)(x).
\end{equation}

\subsection{Small-time expansions}\label{sec:appendix-small-time-expansions}
We first give a small-time expansion for a \emph{fixed action} $a$ and then for a \emph{fixed policy} $\pi$. The policy expansion is the missing ingredient needed to justify the per-step mismatch bound for $\alpha>0$.

\begin{lemma}[Small-time expansion, fixed action]\label{lemma:fixed-action-expansion}
Assume Assumptions~\ref{assumption:dynamics}, \ref{assumption:reward}, and \ref{assumption:test-functions}.
Fix $a\in\mathcal{A}$ and $V\in\mathcal{D}$. Define
\begin{equation}\label{eq:appendix-Phi-fixed-action}
  (\Phi_\tau^{a}V)(x)
  :=
  \E_x^a\!\left[
    \int_0^\tau e^{-\beta t} r(X_t,a)\,dt
    + e^{-\beta\tau}V(X_\tau)
  \right],
\end{equation}
where $X_t$ follows \eqref{eq:controlled-sde} with constant control $a_t\equiv a$.
Then for all sufficiently small $\tau\in(0,1]$,
\begin{equation}\label{eq:appendix-fixed-action-expansion}
  (\Phi_\tau^{a}V)(x)
  =
  V(x)
  + \tau\,q_V(x,a)
  + R_\tau^{a}(V)(x),
\end{equation}
and the remainder satisfies
\begin{equation}\label{eq:appendix-fixed-action-remainder}
  \sup_{x\in\mathbb{R}^d,\,a\in\mathcal{A}}|R_\tau^{a}(V)(x)|
  \le C\,\tau^{3/2},
\end{equation}
where $C$ depends on bounds/Lipschitz constants of $b,\sigma,r,L^aV$ and $\|V\|_{C^2}$
but is independent of $\tau$.
\end{lemma}

\begin{proof}
We treat the terminal and running reward parts separately.

\textbf{Terminal term.} By Dynkin's formula (Lemma~\ref{lemma:dynkin-formula}) under constant action $a$,
\[
  \E_x^a[V(X_\tau)] - V(x)
  =
  \int_0^\tau \E_x^a\big[(L^aV)(X_t)\big]\,dt.
\]
Add and then subtract $(L^aV)(x)$ inside the integral gives:
\[
\begin{aligned}
  \E_x^a[V(X_\tau)]
  &= V(x) + \tau(L^aV)(x)
     + \int_0^\tau \E_x^a\!\Big[(L^aV)(X_t)-(L^aV)(x)\Big]\,dt.
\end{aligned}
\]
Since $L^aV$ is globally Lipschitz in $x$ (Assumption~\ref{assumption:test-functions}), 
Lemma~\ref{lemma:moment-bound} then implies that:
\[
  \Big|\int_0^\tau \E_x^a\!\big[(L^aV)(X_t)-(L^aV)(x)\big]\,dt\Big|
  \le C\int_0^\tau \sqrt{t}\,dt
  = \mathcal{O}(\tau^{3/2}).
\]
Therefore,
\[
  \E_x^a[e^{-\beta\tau}V(X_\tau)]
  = V(x) + \tau\big((L^aV)(x)-\beta V(x)\big) + \mathcal{O}(\tau^{3/2}),
\]
using $e^{-\beta\tau}=1-\beta\tau+\mathcal{O}(\tau^2)$ and absorbing $\mathcal{O}(\tau^2)$ into $\mathcal{O}(\tau^{3/2})$.

\textbf{Running reward term.}
Using Lipschitz continuity of $r(\cdot,a)$ and Lemma~\ref{lemma:moment-bound},
\[
\begin{aligned}
  \Big|
  \E_x^a\Big[\int_0^\tau e^{-\beta t}r(X_t,a)\,dt\Big]
  - r(x,a)\int_0^\tau e^{-\beta t}\,dt
  \Big|
  &\le \int_0^\tau e^{-\beta t}\,L_r\,\E_x^a\|X_t-x\|\,dt\\
  &\le C\int_0^\tau \sqrt{t}\,dt
  = \mathcal{O}(\tau^{3/2}).
\end{aligned}
\]
Since $\int_0^\tau e^{-\beta t}dt=\tau+\mathcal{O}(\tau^2)$, we conclude
\[
  \E_x^a\Big[\int_0^\tau e^{-\beta t}r(X_t,a)\,dt\Big]
  = \tau r(x,a) + \mathcal{O}(\tau^{3/2}).
\]
Combining the two expansions and the definition of $q$-function ($q_V(x,a)=r(x,a)+(L^aV)(x)-\beta V(x)$) yields \eqref{eq:appendix-fixed-action-expansion} and the remainder bound \eqref{eq:appendix-fixed-action-remainder}.
\end{proof}

\begin{lemma}[Small-time expansion, fixed policy via relaxed dynamics]\label{lemma:fixed-policy-expansion}
Assume Assumptions~\ref{assumption:dynamics}, \ref{assumption:reward}, and \ref{assumption:test-functions}.
Fix a Markov policy $\pi$ and $V\in\mathcal{D}$.
Let $(\tilde X_t)_{t\ge 0}$ follow the relaxed SDE \eqref{eq:sde-relaxed} and define the relaxed reward
$\tilde r^{(\alpha)}(\cdot;\pi)$ as in \eqref{eq:relaxed-reward}.
Define the fixed-policy short-horizon operator
\begin{equation}\label{eq:appendix-Phi-fixed-policy}
  (\Phi_\tau^{\pi,(\alpha)}V)(x)
  :=
  \E\!\left[
    \int_0^\tau e^{-\beta t}\,\tilde r^{(\alpha)}(\tilde X_t;\pi)\,dt
    + e^{-\beta\tau}V(\tilde X_\tau)
    \,\Big|\,\tilde X_0=x
  \right].
\end{equation}
Then for all sufficiently small $\tau\in(0,1]$,
\begin{equation}\label{eq:appendix-fixed-policy-expansion}
  (\Phi_\tau^{\pi,(\alpha)}V)(x)
  =
  V(x)
  + \tau\Big(
      \tilde r^{(\alpha)}(x;\pi)
      + (L^\pi V)(x)
      - \beta V(x)
    \Big)
  + R_\tau^{\pi,(\alpha)}(V)(x),
\end{equation}
where $L^\pi$ is the relaxed generator in \cref{eq:generator-relaxed}, and
\begin{equation}\label{eq:appendix-fixed-policy-remainder}
  \sup_{x\in\mathbb{R}^d}\,|R_\tau^{\pi,(\alpha)}(V)(x)| \le C_\alpha\,\tau^{3/2},
\end{equation}
with $C_\alpha$ independent of $\tau$ and uniform over $\pi$ in the admissible class.
\end{lemma}

\begin{proof}
We again treat the terminal and running terms separately.

\textbf{Terminal term.} Applying Dynkin's formula (Lemma~\ref{lemma:dynkin-formula}) to $V(\tilde X_t)$ under the relaxed generator $L^\pi$,
\[
  \E[V(\tilde X_\tau)] - V(x)
  =
  \int_0^\tau \E\big[(L^\pi V)(\tilde X_t)\big]\,dt.
\]
Add and subtract $(L^\pi V)(x)$:
\[
  \E[V(\tilde X_\tau)]
  = V(x) + \tau(L^\pi V)(x)
    + \int_0^\tau \E\big[(L^\pi V)(\tilde X_t)-(L^\pi V)(x)\big]\,dt.
\]
By Assumption~\ref{assumption:test-functions}, $L^\pi V$ is globally Lipschitz in $x$ (uniformly over admissible $\pi$). In addition, Lemma~\ref{lemma:moment-bound} applies to $\tilde X$ as well, and hence the last integral is $\mathcal{O}(\tau^{3/2})$ uniformly in $\pi$. Multiplying by $e^{-\beta\tau}=1-\beta\tau+\mathcal{O}(\tau^2)$ yields:
\[
  \E\big[e^{-\beta\tau}V(\tilde X_\tau)\big]
  = V(x) + \tau\big((L^\pi V)(x)-\beta V(x)\big) + \mathcal{O}(\tau^{3/2}).
\]

\textbf{Running term.}
Using Lipschitz continuity of $\tilde r^{(\alpha)}(\cdot;\pi)$ in $x$ over the admissible class,
together with Lemma~\ref{lemma:moment-bound},
\[
\begin{aligned}
  \left|
  \E\Big[\int_0^\tau e^{-\beta t}\tilde r^{(\alpha)}(\tilde X_t;\pi)\,dt\Big]
  - \tilde r^{(\alpha)}(x;\pi)\int_0^\tau e^{-\beta t}\,dt
  \right|
  &\le \int_0^\tau e^{-\beta t}\,C\,\E\|\tilde X_t-x\|\,dt \\
  &\le C\int_0^\tau \sqrt{t}\,dt
  = \mathcal{O}(\tau^{3/2}),
\end{aligned}
\]
uniformly in $\pi$. Again, because $\int_0^\tau e^{-\beta t}dt=\tau+\mathcal{O}(\tau^2)$, we obtain:
\[
  \E\Big[\int_0^\tau e^{-\beta t}\tilde r^{(\alpha)}(\tilde X_t;\pi)\,dt\Big]
  = \tau\,\tilde r^{(\alpha)}(x;\pi) + \mathcal{O}(\tau^{3/2}).
\]
Combining both parts gives \eqref{eq:appendix-fixed-policy-expansion} and \eqref{eq:appendix-fixed-policy-remainder}.
\end{proof}

\subsection{Per-step mismatch: $\Phi_\tau^{(\alpha)}$ vs.\ $T_\tau^{(\alpha)}$}\label{sec:appendix-per-step-mismatch}

\begin{lemma}[Per-step mismatch bound]\label{lemma:per-step-mismatch}
Assume Assumptions~\ref{assumption:dynamics}, \ref{assumption:reward}, and \ref{assumption:test-functions}.
Let $V\in\mathcal{D}$ and $\alpha\ge 0$.
Then for all sufficiently small $\tau\in(0,1]$,
\begin{equation}\label{eq:appendix-Phi-expansion}
  \Phi_\tau^{(\alpha)}(V)(x)
  = V(x) + \tau\,H^{(\alpha)}(V)(x) + R_\tau^{(\alpha)}(V)(x),
\end{equation}
with
\begin{equation}\label{eq:appendix-Phi-remainder-bound}
  \|R_\tau^{(\alpha)}(V)\|_\infty \le C_\alpha\,\tau^{3/2},
\end{equation}
where $C_\alpha$ is independent of $\tau$.
Consequently,
\begin{equation}\label{eq:appendix-Phi-T-mismatch}
  \|T_\tau^{(\alpha)}(V) - \Phi_\tau^{(\alpha)}(V)\|_\infty
  \le C_\alpha\,\tau^{3/2}.
\end{equation}
\end{lemma}

\begin{proof}
We treat $\alpha=0$ and $\alpha>0$ separately.

\emph{Case 1: $\alpha=0$ (hard-max).}
For fixed $a$, Lemma~\ref{lemma:fixed-action-expansion} yields
\[
  \Phi_\tau^{a}V(x)
  = V(x) + \tau q_V(x,a) + R_\tau^{a}(V)(x),
  \qquad \sup_{x,a}|R_\tau^{a}(V)(x)|\le C\tau^{3/2}.
\]
Since $\Phi_\tau^{(0)}V(x)=\sup_{a\in\mathcal{A}}\Phi_\tau^{a}V(x)$,
\[
\begin{aligned}
  \Phi_\tau^{(0)}V(x)
  &= V(x) + \tau\sup_{a\in\mathcal{A}} q_V(x,a) + \sup_{a\in\mathcal{A}}R_\tau^{a}(V)(x) \\
  &= V(x) + \tau H^{(0)}(V)(x) + R_\tau^{(0)}(V)(x),
\end{aligned}
\]
and $\|R_\tau^{(0)}(V)\|_\infty\le C\tau^{3/2}$.

\emph{Case 2: $\alpha>0$ (soft-max).}
For any fixed Markov policy $\pi$, Lemma~\ref{lemma:fixed-policy-expansion} gives
\[
  \Phi_\tau^{\pi,(\alpha)}V(x)
  =
  V(x) + \tau\Big(\tilde r^{(\alpha)}(x;\pi)+(L^\pi V)(x)-\beta V(x)\Big)
  + R_\tau^{\pi,(\alpha)}(V)(x),
\]
with $\sup_x|R_\tau^{\pi,(\alpha)}(V)(x)|\le C_\alpha\tau^{3/2}$ uniformly in $\pi$.
Taking supremum over $\pi$ and using $\Phi_\tau^{(\alpha)}V=\sup_\pi \Phi_\tau^{\pi,(\alpha)}V$ yields
\[
  \Phi_\tau^{(\alpha)}V(x)
  =
  V(x)
  + \tau\,\sup_{\pi}\Big(\tilde r^{(\alpha)}(x;\pi)+(L^\pi V)(x)-\beta V(x)\Big)
  + R_\tau^{(\alpha)}(V)(x),
\]
where $\|R_\tau^{(\alpha)}(V)\|_\infty\le C_\alpha\tau^{3/2}$.

It remains to identify the second supremum term as $H^{(\alpha)}(V)(x)$. Using the definitions of $\tilde r^{(\alpha)}$ and $L^\pi$,
\[
\begin{aligned}
  \tilde r^{(\alpha)}(x;\pi)+(L^\pi V)(x)-\beta V(x)
  &= \int_{\mathcal{A}}\Big(r(x,a)+(L^aV)(x)-\beta V(x)\Big)\,\pi(da\mid x)
     - \alpha\int_{\mathcal{A}}\log\pi(a\mid x)\,\pi(da\mid x) \\
  &= \int_{\mathcal{A}} q_V(x,a)\,\pi(da\mid x)
     - \alpha\int_{\mathcal{A}}\log\pi(a\mid x)\,\pi(da\mid x).
\end{aligned}
\]
Through \cref{lemma:entropy-KL-identity}, for each fixed $x$,
\[
\begin{aligned}
  \sup_{\pi(\cdot\mid x)}
  \left\{
    \int_{\mathcal{A}} q_V(x,a)\,\pi(da\mid x)
    - \alpha\int_{\mathcal{A}}\log\pi(a\mid x)\,\pi(da\mid x)
  \right\}
  &= \sup_{\pi(\cdot\mid x)} \E_{a\sim\pi(\cdot\mid x)}\!\big[q_V(x,a)-\alpha\log\pi(a\mid x)\big]\\
  &= \mathcal{Q}^{(\alpha)}(q_V)(x)
  = \alpha\log\int_{\mathcal{A}}\exp\!\Big(\frac{q_V(x,a)}{\alpha}\Big)\,da
\end{aligned}
\]
which is exactly $H^{(\alpha)}(V)(x)$ by \eqref{eq:appendix-Hamiltonian-definition}. Substituting gives \eqref{eq:appendix-Phi-expansion}--\eqref{eq:appendix-Phi-remainder-bound}. Finally, \eqref{eq:appendix-Phi-T-mismatch} follows immediately from the definition
$T_\tau^{(\alpha)}(V)=V+\tau H^{(\alpha)}(V)$.
\end{proof}

\subsection{Proof of \cref{thm:value_convergence}}\label{sec:appendix-proof-value-theorem}

\ValueFunctionConvergence*
\begin{proof}
Let $W_{k+1}:=\Phi_\tau^{(\alpha)}(W_k)$ be the exact dynamic-programming iterations initialized at the same $W_0:=V_0$. Define $\Delta_k := V_k - W_k$. Then
\[
\begin{aligned}
  \Delta_{k+1}
  &= T_\tau^{(\alpha)}(V_k) - \Phi_\tau^{(\alpha)}(W_k) \\
  &= \Big(T_\tau^{(\alpha)}(V_k) - \Phi_\tau^{(\alpha)}(V_k)\Big)
     + \Big(\Phi_\tau^{(\alpha)}(V_k) - \Phi_\tau^{(\alpha)}(W_k)\Big).
\end{aligned}
\]
Taking sup norms and using Lemma~\ref{lemma:per-step-mismatch} and
Lemma~\ref{lemma:Phi-contraction}, we get:
\[
  \|\Delta_{k+1}\|_\infty
  \le C_\alpha\tau^{3/2} + e^{-\beta\tau}\|\Delta_k\|_\infty.
\]
Since $\Delta_0=0$, unrolling the recursion gives
\[
  \|\Delta_k\|_\infty
  \le \frac{C_\alpha\tau^{3/2}}{1-e^{-\beta\tau}}
  \qquad \forall k\ge 0.
\]
Finally, by the triangle inequality and Lemma~\ref{lemma:Phi-contraction},
\[
\begin{aligned}
  \|V_k - V^{(\alpha)}\|_\infty
  &\le \|V_k - W_k\|_\infty + \|W_k - V^{(\alpha)}\|_\infty \\
  &\le \frac{C_\alpha\tau^{3/2}}{1-e^{-\beta\tau}}
     + e^{-\beta\tau k}\,\|W_0 - V^{(\alpha)}\|_\infty,
\end{aligned}
\]
which is the claimed bound.
\end{proof}

Note that, since $1 - e^{-\beta \tau} = \beta\tau + \mathcal{O}(\tau^2)$, the error estimate implies that $\norm{V_k - V^{(\alpha)}} \leq C_0 e^{-\beta\tau k} + C_1 \tau^{1/2}$, and as $k, \tau \to 0$, this error also goes to $0$.

\section{Proof of $q$-convergence (\cref{thm:q_convergence})}\label{sec:appendix:q_convergence}
This appendix proves the model-free convergence guarantee in \cref{thm:q_convergence}. The key challenge is that the infinitesimal generator term $L^aV$ is unavailable in model-free learning, so we replace the infinitesimal advantage-rate $q_V$ by a short-horizon estimator $q_V^u$ computed from rollouts of length $u$ (See \cref{sec:formulation:decoupled}). This introduces a \emph{modelling bias} $q_V^u-q_V$ and an additional stability issue because $q_V^u$ is $\mathcal{O}(1/u)$-Lipschitz in $V$. Richardson extrapolation removes the leading $\mathcal{O}(u)$ bias term and improves the approximation accuracy of the learning signal while preserving the same stability structure.

\subsection{Baseline modelling bias under standard assumptions}\label{sec:appendix-q-bias-sqrtu}
Before imposing higher-order smoothness (which yields $\mathcal{O}(u)$ and Richardson $\mathcal{O}(u^2)$), we record the baseline modelling-bias rate available under the standard diffusion assumptions. In this setting, one only obtains $\|q_V^u-q_V\|_\infty=\mathcal{O}(\sqrt{u})$.

\begin{lemma}[Baseline bias for $q_V^u$: $\|q_V^u-q_V\|_\infty=\mathcal{O}(\sqrt{u})$]\label{lemma:qV-u-sqrtu-bias}
Assume Assumptions~\ref{assumption:dynamics} and \ref{assumption:reward}.
Suppose $V\in C_b^2(\mathbb{R}^d)$ and $L^aV$ is globally Lipschitz in $x$ uniformly over $a\in\mathcal{A}$.
Then there exist $u_0>0$ and $C<\infty$ such that for all $u\in(0,u_0]$,
\begin{equation}\label{eq:appendix-qV-u-sqrtu-bias}
  \|q_V^u-q_V\|_\infty \le C\sqrt{u}.
\end{equation}
Consequently, for any $\alpha>0$,
\begin{equation}\label{eq:appendix-Hu-H-sqrtu-bias}
  \|H_u^{(\alpha)}(V)-H^{(\alpha)}(V)\|_\infty \le C\sqrt{u}.
\end{equation}
\end{lemma}

\begin{proof}[Proof sketch]
By Dynkin's formula,
\[
  \E_x^a[V(X_u^a)]-V(x)=\int_0^u \E_x^a[(L^aV)(X_t^a)]\,dt.
\]
Subtract $u(L^aV)(x)$ and use the Lipschitz property of $L^aV$ gives:
\[
  \Big|\E_x^a[V(X_u^a)]-V(x)-u(L^aV)(x)\Big|
  \le \int_0^u \E_x^a\!\big[\,|(L^aV)(X_t^a)-(L^aV)(x)|\,\big]dt
  \le C\int_0^u \E_x^a[\|X_t^a-x\|]dt.
\]
\cref{lemma:moment-bound} give $\E_x^a[\|X_t^a-x\|]\le C\sqrt{t}$ uniformly in $(x,a)$, and hence the right-hand side is $\mathcal{O}(u^{3/2})$. Through definitions of $q_V^u$ and $q_V$, dividing by $u$ and accounting for the discount factor $e^{-\beta u}=1 - \beta u + \mathcal{O}(u^2)$ yields $q_V^u(x,a)=q_V(x,a)+\mathcal{O}(\sqrt{u})$ uniformly in $(x,a)$, proving \eqref{eq:appendix-qV-u-sqrtu-bias}. The Hamiltonian bound \eqref{eq:appendix-Hu-H-sqrtu-bias} follows from a similar Lipschitz argument as in Lemma~\ref{lemma:logsumexp-and-H-bias}.
\end{proof}

\subsection{Smoothness assumptions for quantitative $u$-rates}\label{sec:appendix-q-assumptions}

\begin{assumption}[Smooth $C^2$ assumptions]\label{assumption:dynamics-smooth}
In addition to Assumption~\ref{assumption:dynamics}, assume:
\begin{enumerate}
  \item For each $a\in\mathcal{A}$, the drift and diffusion maps
  $x\mapsto b(x,a)$ and $x\mapsto \sigma(x,a)$ are $C^2$ with bounded derivatives
  up to order $2$, uniformly in $a$, and globally Lipschitz in $x$ uniformly in $a$.
  \item For each $a\in\mathcal{A}$, the reward map $x\mapsto r(x,a)$ is $C^2$
  with bounded derivatives up to order $2$, uniformly in $a$, and globally Lipschitz in $x$.
\end{enumerate}
\end{assumption}

\begin{assumption}[Value function smoothness]\label{assumption:value-C4}
In addition to standard assumptions~\ref{assumption:reward} and \ref{assumption:test-functions}, assume that the value function domain $\mathcal{D} \subseteq C_b^4(\mathbb{R}^d)$ so that for value function $V$, and for each $a\in\mathcal{A}$, both $L^aV$ and $(L^a)^2V$ are bounded and continuous.
\end{assumption}

\textbf{Notation.} Fix $a\in\mathcal{A}$. Let $(X_t^a)_{t\ge0}$ solve the controlled SDE
\[
  dX_t^a = b(X_t^a,a)\,dt + \sigma(X_t^a,a)\,dW_t,\qquad X_0^a=x,
\]
Recall the fixed action controlled generator $L^a$ is:
\[
  (L^a f)(x)
  :=
  b(x,a)\cdot\nabla f(x)
  +\frac12\mathrm{Tr}\!\big(\sigma(x,a)\sigma(x,a)^\top \nabla^2 f(x)\big).
\]
We additionally define the Markov semigroup
\begin{equation}\label{eq:appendix-semigroup-Pta}
  (P_t^a f)(x) := \E_x^a[f(X_t^a)].
\end{equation}

For $V\in C_b^2(\mathbb{R}^d)$, recall the $q$-function:
\begin{equation}\label{eq:appendix-qV-definition}
  q_V(x,a) := r(x,a) + (L^aV)(x) - \beta V(x).
\end{equation}
Together with its finite-step estimator $q_V^u$ with time step $u > 0$:
\begin{equation}\label{eq:appendix-qV-u-definition}
  q_V^u(x,a)
  :=
  \frac{e^{-\beta u}\,\E_x^a\!\big[V(X_u^a)\big] - V(x)}{u}
  \;+\; r(x,a).
\end{equation}
And the Richardson version $\tilde q_V^u$:
\begin{equation}\label{eq:appendix-qV-Richardson-definition}
  \tilde q_V^u(x,a) := 2\,q_V^{u/2}(x,a) - q_V^u(x,a).
\end{equation}

We now start with a series of lemmas that serve as the main ingredients for the proof of \cref{thm:q_convergence}.

\subsection{Bias expansions for $q_V^u$ and $\tilde q_V^u$}\label{sec:appendix-q-bias}

We begin with a quantitative small-time expansion of the semigroup $P_t^a$.

\begin{lemma}[Second-order semigroup expansion]\label{lemma:semigroup-second-order-expansion}
Under assumptions~\ref{assumption:dynamics-smooth} and \ref{assumption:value-C4}, for $a\in\mathcal{A}$, $V \in C_b^4(\mathbb{R}^d)$, and for all $t \in [0,1]$, we have
\begin{equation}\label{eq:appendix-semigroup-second-order-expansion}
  (P_t^aV)(x)
  = V(x) + t(L^aV)(x) + R_V(t;x,a),
\end{equation}
where the remainder satisfies
\begin{equation}\label{eq:appendix-semigroup-second-order-remainder}
  \sup_{x\in\mathbb{R}^d,\,a\in\mathcal{A}}
  |R_V(t;x,a)|
  \le \frac{t^2}{2}\sup_{a\in\mathcal{A}}\|(L^a)^2V\|_\infty
  \le C_V\,t^2,
\end{equation}
for a constant $C_V<\infty$ depending only on $V$ and the coefficients.
\end{lemma}

\begin{proof}
Fix $a$ and write $P_t:=P_t^a$ and $L:=L^a$ for notation convenience.
Dynkin's formula (\cref{lemma:dynkin-formula}) gives
\[
  (P_tV)(x)=V(x)+\int_0^t(P_sLV)(x)\,ds.
\]
Applying the same identity to $LV$ one more time yields:
\[
  (P_sLV)(x)=(LV)(x)+\int_0^s(P_rL^2V)(x)\,dr.
\]
Substitute into the first identity and use Fubini theorem, we have:
\[
  (P_tV)(x)
  =V(x)+t(LV)(x)+\int_0^t(t-r)(P_rL^2V)(x)\,dr.
\]
Define $R_V(t;x,a):=\int_0^t(t-r)(P_r^a(L^a)^2V)(x)\,dr$.
Then
\[
  |R_V(t;x,a)|\le \|(L^a)^2V\|_\infty\int_0^t(t-r)\,dr=\frac{t^2}{2}\|(L^a)^2V\|_\infty,
\]
which gives \eqref{eq:appendix-semigroup-second-order-remainder}.
\end{proof}

\begin{lemma}[First-order bias for $q_V^u$: $\|q_V^u-q_V\|_\infty=\mathcal{O}(u)$]\label{lemma:qV-u-first-order-bias}
Under Assumptions~\ref{assumption:dynamics-smooth} and \ref{assumption:value-C4}, there exist $u_0>0$ and $C_q<\infty$ such that for all $u\in(0,u_0]$,
\begin{equation}\label{eq:appendix-qV-u-first-order-bias}
  \|q_V^u-q_V\|_\infty \le C_q\,u.
\end{equation}
\end{lemma}

\begin{proof}
Fix $(x,a)$. By Lemma~\ref{lemma:semigroup-second-order-expansion},
\[
  \E_x^a[V(X_u^a)] = V(x) + u(L^aV)(x) + \mathcal{O}(u^2),
\]
uniformly in $(x,a)$. Also $e^{-\beta u}=1-\beta u+\mathcal{O}(u^2)$, hence
\[
\begin{aligned}
  e^{-\beta u}\E_x^a[V(X_u^a)] - V(x)
  &= \big(1-\beta u+\mathcal{O}(u^2)\big)\big(V(x)+u(L^aV)(x)+\mathcal{O}(u^2)\big)-V(x) \\
  &= u\big((L^aV)(x)-\beta V(x)\big) + \mathcal{O}(u^2),
\end{aligned}
\]
uniformly in $(x,a)$. Divide by $u$ and add $r(x,a)$, we obtain
\[
  q_V^u(x,a) = r(x,a) + (L^aV)(x) - \beta V(x) + \mathcal{O}(u)
  = q_V(x,a) + \mathcal{O}(u),
\]
uniformly in $(x,a)$, proving \eqref{eq:appendix-qV-u-first-order-bias}.
\end{proof}

\begin{lemma}[Richardson bias for $\tilde q_V^u$: $\|\tilde q_V^u-q_V\|_\infty=\mathcal{O}(u^2)$]\label{lemma:qV-u-Richardson-bias}
Under Assumptions~\ref{assumption:dynamics-smooth} and \ref{assumption:value-C4}, there exist $u_0>0$ and $\tilde C_q<\infty$ such that for all $u\in(0,u_0]$,
\begin{equation}\label{eq:appendix-qV-u-Richardson-bias}
  \|\tilde q_V^u-q_V\|_\infty \le \tilde C_q\,u^2.
\end{equation}
\end{lemma}

\begin{proof}
Proof of Lemma~\ref{lemma:qV-u-first-order-bias} implies the expansion of $q_V^u(x,a)$ contains only $u$ and $u^2$ terms with no $u^h,\,h \in (1, 2)$ in between:
\[
  q_V^u(x,a)=q_V(x,a)+u\,c_1(x,a)+u^2\,c_2(u;x,a),
\]
where $c_1$ is bounded and $c_2$ is uniformly bounded for $u$ small.
Evaluating at $u/2$ and forming $\tilde q_V^u=2q_V^{u/2}-q_V^u$ cancels the linear term:
\[
  \tilde q_V^u(x,a)
  = q_V(x,a)+u^2\Big(\tfrac12c_2(u/2;x,a)-c_2(u;x,a)\Big),
\]
which is $\mathcal{O}(u^2)$ uniformly in $(x,a)$.
\end{proof}

\subsection{Hamiltonians and Lipschitz lemmas}\label{sec:appendix-Hamiltonians}

Recall the (soft) Hamiltonian for $\alpha>0$,
\begin{equation}\label{eq:appendix-HV-definition}
  H^{(\alpha)}(V)(x)
  := \alpha \log \int_{\mathcal{A}} \exp\!\Big(\frac{q_V(x,a)}{\alpha}\Big)\,da,
\end{equation}
and its finite-horizon analogues on approximations $q_V^u$ and $\tilde q_V^u$ of $q$-function:
\begin{align}
  H_u^{(\alpha)}(V)(x)
  &:= \alpha \log \int_{\mathcal{A}} \exp\!\Big(\frac{q_V^u(x,a)}{\alpha}\Big)\,da,
  \label{eq:appendix-Hu-definition}\\
  \tilde H_u^{(\alpha)}(V)(x)
  &:= \alpha \log \int_{\mathcal{A}} \exp\!\Big(\frac{\tilde q_V^u(x,a)}{\alpha}\Big)\,da.
  \label{eq:appendix-Htilde-definition}
\end{align}

\begin{lemma}[Lipschitz Hamiltonian bias]\label{lemma:logsumexp-and-H-bias}
Fix $\alpha>0$ and let $z,\bar z:\mathcal{A}\to\mathbb{R}$ be bounded. Then
\begin{equation}\label{eq:appendix-logsumexp-Lipschitz}
  \Big|
  \alpha\log\!\int_{\mathcal{A}} e^{z(a)/\alpha}\,da
  -
  \alpha\log\!\int_{\mathcal{A}} e^{\bar z(a)/\alpha}\,da
  \Big|
  \le \|z-\bar z\|_\infty.
\end{equation}
Consequently, for any bounded $V$,
\begin{align}
  \|H_u^{(\alpha)}(V)-H^{(\alpha)}(V)\|_\infty
  &\le \|q_V^u-q_V\|_\infty,
  \label{eq:appendix-Hu-H-bound}\\
  \|\tilde H_u^{(\alpha)}(V)-H^{(\alpha)}(V)\|_\infty
  &\le \|\tilde q_V^u-q_V\|_\infty.
  \label{eq:appendix-Htilde-H-bound}
\end{align}
\end{lemma}

\begin{proof}
For all $a$, $z(a)\le \bar z(a)+\|z-\bar z\|_\infty$, hence
$e^{z(a)/\alpha}\le e^{\|z-\bar z\|_\infty/\alpha}e^{\bar z(a)/\alpha}$.
Integrate and take logs to get
\[
  \log\!\int e^{z/\alpha}
  \le \frac{\|z-\bar z\|_\infty}{\alpha}+\log\!\int e^{\bar z/\alpha}.
\]
Reverse the roles of $z,\bar z$ to obtain the other inequality, we can deduce \eqref{eq:appendix-logsumexp-Lipschitz}. Apply this pointwise in $x$ with $(z,\bar z)=(q_V^u(x,\cdot),q_V(x,\cdot))$ and $(z,\bar z)=(\tilde q_V^u(x,\cdot),q_V(x,\cdot))$, and then take the supremum over $x$, we obtain \eqref{eq:appendix-Hu-H-bound}--\eqref{eq:appendix-Htilde-H-bound}.
\end{proof}

Now we show that $q_V^u$ is $\mathcal{O}(1/u)$-Lipschitz in $V$
\begin{lemma}[Lipschitz dependence of $q_V^u$ on $V$]\label{lemma:qV-u-Lipschitz-in-V}
For any bounded functions $V,W$ and any $u>0$,
\begin{equation}\label{eq:appendix-qV-u-Lipschitz-in-V}
  \|q_V^u-q_W^u\|_\infty \le \frac{2}{u}\,\|V-W\|_\infty.
\end{equation}
Consequently, for the Richardson rate,
\begin{equation}\label{eq:appendix-qV-Richardson-Lipschitz-in-V}
  \|\tilde q_V^u-\tilde q_W^u\|_\infty
  \le \frac{6}{u}\,\|V-W\|_\infty.
\end{equation}
\end{lemma}

\begin{proof}
From \eqref{eq:appendix-qV-u-definition}, the $r(x,a)$ term cancels, so
\[
  q_V^u(x,a)-q_W^u(x,a)
  =
  \frac{e^{-\beta u}\E_x^a[(V-W)(X_u^a)]-(V-W)(x)}{u}.
\]
Taking absolute values and using
$|\E[(V-W)(X_u^a)]|\le \|V-W\|_\infty$ and $|(V-W)(x)|\le\|V-W\|_\infty$
gives \eqref{eq:appendix-qV-u-Lipschitz-in-V}.
For Richardson, use $\tilde q_V^u=2q_V^{u/2}-q_V^u$ and apply the first bound twice:
\[
  \|\tilde q_V^u-\tilde q_W^u\|_\infty
  \le 2\|q_V^{u/2}-q_W^{u/2}\|_\infty+\|q_V^u-q_W^u\|_\infty
  \le 2\cdot\frac{2}{u/2}\|V-W\|_\infty+\frac{2}{u}\|V-W\|_\infty
  =\frac{6}{u}\|V-W\|_\infty.
\]
\end{proof}

\subsection{One-step mismatch under stricter assumptions} \label{sec:appendix:one-step-mismatch-smooth}
Compared to \cref{lemma:per-step-mismatch} (which gives an $\mathcal{O}(\tau^{3/2})$ remainder under standard assumptions), with stricter smooth assumptions, we can upgrade the local mismatch to $\mathcal{O}(\tau^2)$. This allows the $\tau^2$ term in \cref{thm:q_convergence}.

\begin{lemma}[Smooth local expansion of $\Phi_\tau^{(\alpha)}$]\label{lemma:Phi-smooth-local-expansion}
Under Assumptions~\ref{assumption:dynamics-smooth}, and \ref{assumption:value-C4}, fix $\alpha\ge0$ and $V\in C_b^4(\mathbb{R}^d)$. Then there exist $\tau_0>0$ and $C_\Phi<\infty$ such that for all $\tau\in(0,\tau_0]$,
\begin{equation}\label{eq:appendix-Phi-smooth-local-expansion}
  \Phi_\tau^{(\alpha)}(V)
  =
  V+\tau H^{(\alpha)}(V)+R_\tau^{(\alpha)}(V),
  \qquad
  \|R_\tau^{(\alpha)}(V)\|_\infty \le C_\Phi\,\tau^2.
\end{equation}
Equivalently,
\begin{equation}\label{eq:appendix-Phi-vs-Ttau-smooth}
  \|\Phi_\tau^{(\alpha)}(V)-T_\tau^{(\alpha)}(V)\|_\infty \le C_\Phi\,\tau^2.
\end{equation}
\end{lemma}

\begin{proof}
We follow similar steps as in the proof of \cref{lemma:per-step-mismatch}. First we consider a fixed policy $\pi$, and recall from \cref{lemma:per-step-mismatch} the policy-wise operator $\Phi_{\tau}^{\pi,(\alpha)}$:
\begin{equation}\label{eq:appendix-Phi-policywise}
\begin{aligned}
  (\Phi_{\tau}^{\pi,(\alpha)}V)(x)
  :=
  \E_x^\pi\Bigg[
    \int_0^\tau e^{-\beta t}\Big(r(X_t,a_t)-\alpha\log\pi(a_t\mid X_t)\Big)\,dt
    + e^{-\beta\tau}V(X_\tau)
  \Bigg].
\end{aligned}
\end{equation}
We will show that, uniformly in $x$ and uniformly over admissible $\pi$,
\begin{equation}\label{eq:appendix-policywise-expansion}
  \Phi_{\tau}^{\pi,(\alpha)}V(x)
  =
  V(x) + \tau\,G_\alpha(V)(x,\pi) + \mathcal{O}(\tau^2),
\end{equation}
where
\begin{equation}\label{eq:appendix-Galpha-definition}
  G_\alpha(V)(x,\pi)
  :=
  \int_{\mathcal{A}} q_V(x,a)\,\pi(da\mid x)
  -
  \alpha\int_{\mathcal{A}}\log\pi(a\mid x)\,\pi(da\mid x).
\end{equation}

To see this, we again treat separately the terminal term and the running term.

\textbf{Terminal term.}
Under Assumptions~\ref{assumption:dynamics-smooth} and \ref{assumption:reward},
the relaxed coefficients $(\tilde b(\cdot;\pi),\tilde\sigma(\cdot;\pi))$ are Lipschitz,
and the relaxed generator $L^\pi$ is well-defined and has bounded coefficients.
Moreover $V\in C_b^4$ implies $L^\pi V$ and $(L^\pi)^2V$ are bounded.
Hence the same Dynkin-iteration argument as in
Lemma~\ref{lemma:semigroup-second-order-expansion} (with $L^a$ replaced by $L^\pi$)
yields the second-order expansion
\begin{equation}\label{eq:appendix-terminal-expansion-policy}
  \E_x^\pi[V(X_\tau)]
  =
  V(x) + \tau(L^\pi V)(x) + \mathcal{O}(\tau^2),
\end{equation}
uniformly in $x$ and uniformly over admissible $\pi$.
Multiplying by $e^{-\beta\tau}=1-\beta\tau+\mathcal{O}(\tau^2)$ gives
\begin{equation}\label{eq:appendix-terminal-discounted-expansion-policy}
  \E_x^\pi[e^{-\beta\tau}V(X_\tau)]
  =
  V(x) + \tau\big((L^\pi V)(x)-\beta V(x)\big) + \mathcal{O}(\tau^2).
\end{equation}

\textbf{Running term.}
Define the relaxed one-step reward-with-entropy
\[
  \tilde r^{(\alpha)}(x;\pi)
  :=
  \int_{\mathcal{A}}\Big(r(x,a)-\alpha\log\pi(a\mid x)\Big)\,\pi(da\mid x).
\]
Under Assumptions~\ref{assumption:dynamics-smooth} and \ref{assumption:reward}
and the policy regularity in Assumption~\ref{assumption:reward},
the function $x\mapsto \tilde r^{(\alpha)}(x;\pi)$ is bounded and $C^2$ with bounded derivatives,
uniformly over admissible $\pi$.
Applying the same second-order semigroup expansion (again with generator $L^\pi$) to
$\tilde r^{(\alpha)}(\cdot;\pi)$ implies
\begin{equation}\label{eq:appendix-running-reward-expansion-policy}
  \E_x^\pi[\tilde r^{(\alpha)}(X_t;\pi)]
  =
  \tilde r^{(\alpha)}(x;\pi) + \mathcal{O}(t),
  \qquad t\in[0,1],
\end{equation}
uniformly in $x$ and uniformly over $\pi$. Therefore,
\begin{equation}\label{eq:appendix-running-integral-expansion-policy}
\E_x^\pi\Big[\int_0^\tau e^{-\beta t}\tilde r^{(\alpha)}(X_t;\pi)\,dt\Big] 
=\int_0^\tau (1+\mathcal{O}(t))\big(\tilde r^{(\alpha)}(x;\pi)+\mathcal{O}(t)\big)\,dt
=\tau\,\tilde r^{(\alpha)}(x;\pi) + \mathcal{O}(\tau^2)
\end{equation}
uniformly in $x$ and uniformly over $\pi$.

Combining \eqref{eq:appendix-terminal-discounted-expansion-policy} and
\eqref{eq:appendix-running-integral-expansion-policy} in
\eqref{eq:appendix-Phi-policywise} gives \eqref{eq:appendix-policywise-expansion}.
Finally, note that the bracketed term in \eqref{eq:appendix-policywise-expansion}
equals
\[
  \tilde r^{(\alpha)}(x;\pi) + (L^\pi V)(x) - \beta V(x)
  =
  \int_{\mathcal{A}} q_V(x,a)\,\pi(da\mid x)
  - \alpha\int_{\mathcal{A}}\log\pi(a\mid x)\,\pi(da\mid x),
\]
which is exactly $G_\alpha(V)(x,\pi)$ in \eqref{eq:appendix-Galpha-definition}.

Now we optimize over policies $\pi$ to obtain the term $H^{(\alpha)}(V)$. By definition, $\Phi_\tau^{(\alpha)}(V)(x) = \sup_\pi \Phi_{\tau}^{\pi,(\alpha)}(V)(x)$. Hence, taking the supremum in \eqref{eq:appendix-policywise-expansion} yields
\[
  \Phi_\tau^{(\alpha)}(V)(x)
  =
  V(x) + \tau\sup_\pi G_\alpha(V)(x,\pi) + \mathcal{O}(\tau^2).
\]
Again, through a similar argument as in \cref{lemma:per-step-mismatch} that invokes \cref{lemma:entropy-KL-identity}, we have:
\[
  \sup_{\pi(\cdot\mid x)}
  \Big\{
    \int q_V(x,a)\,\pi(da\mid x)
    -\alpha\int \log\pi(a\mid x)\,\pi(da\mid x)
  \Big\}
  =
  \alpha\log\int_{\mathcal{A}} \exp\!\Big(\frac{q_V(x,a)}{\alpha}\Big)\,da
  =
  H^{(\alpha)}(V)(x).
\]
For $\alpha=0$ the same optimization reduces to $\sup_a q_V(x,a)=H^{(0)}(V)(x)$.
Thus
\[
  \Phi_\tau^{(\alpha)}(V)(x) = V(x) + \tau H^{(\alpha)}(V)(x) + \mathcal{O}(\tau^2),
\]
uniformly in $x$, which proves \eqref{eq:appendix-Phi-smooth-local-expansion}.
The equivalent form \eqref{eq:appendix-Phi-vs-Ttau-smooth} follows from the definition
$T_\tau^{(\alpha)}(V)=V+\tau H^{(\alpha)}(V)$.
\end{proof}

Now define the finite-horizon Picard operator and its Richardson variant:
\begin{align}
  (T_{\tau,u}^{(\alpha)}V)(x)
  &:= V(x)+\tau\,H_u^{(\alpha)}(V)(x),
  \label{eq:appendix-Ttau-u-definition}\\
  (T_{\tau,u}^{R,(\alpha)}V)(x)
  &:= V(x)+\tau\,\tilde H_u^{(\alpha)}(V)(x).
  \label{eq:appendix-Ttau-R-definition}
\end{align}

Here we have similar one-step mismatch as in \cref{lemma:per-step-mismatch}, but with better errors.

\begin{lemma}[One-step mismatch: finite-horizon and Richardson]\label{lemma:one-step-mismatch-u-and-Richardson}
Assume Assumptions~\ref{assumption:dynamics-smooth} and \ref{assumption:value-C4},
and fix $\alpha>0$.
Then there exist $\tau_0,u_0>0$ and $C<\infty$ such that for all
$\tau\in(0,\tau_0]$, $u\in(0,u_0]$, and all $V\in C_b^4(\mathbb{R}^d)$,
\begin{align}
  \|\Phi_\tau^{(\alpha)}(V)-T_{\tau,u}^{(\alpha)}(V)\|_\infty
  &\le C\big(\tau^2+\tau u\big),
  \label{eq:appendix-one-step-mismatch-finite-horizon}\\
  \|\Phi_\tau^{(\alpha)}(V)-T_{\tau,u}^{R,(\alpha)}(V)\|_\infty
  &\le C\big(\tau^2+\tau u^2\big).
  \label{eq:appendix-one-step-mismatch-Richardson}
\end{align}
\end{lemma}

\begin{proof}
By Lemma~\ref{lemma:Phi-smooth-local-expansion},
\[
  \|\Phi_\tau^{(\alpha)}(V)-T_\tau^{(\alpha)}(V)\|_\infty \le C_\Phi\tau^2.
\]
For the finite-horizon Picard operator,
\[
  \|T_\tau^{(\alpha)}(V)-T_{\tau,u}^{(\alpha)}(V)\|_\infty
  = \tau\,\|H^{(\alpha)}(V)-H_u^{(\alpha)}(V)\|_\infty
  \le \tau\,\|q_V-q_V^u\|_\infty,
\]
using Lemma~\ref{lemma:logsumexp-and-H-bias}. Then Lemma~\ref{lemma:qV-u-first-order-bias}
yields $\|q_V-q_V^u\|_\infty\le C_qu$, hence
\[
  \|\Phi_\tau^{(\alpha)}(V)-T_{\tau,u}^{(\alpha)}(V)\|_\infty
  \le C_\Phi\tau^2 + \tau C_qu
  \le C(\tau^2+\tau u).
\]
The Richardson bound's proof is identical through using the following inequality from Lemmas~\ref{lemma:qV-u-Richardson-bias} and ~\ref{lemma:logsumexp-and-H-bias}:
\[
  \|H^{(\alpha)}(V)-\tilde H_u^{(\alpha)}(V)\|_\infty
  \le \|\tilde q_V^u-q_V\|_\infty
  \le \tilde C_q u^2
\]
\end{proof}

\subsection{Proof of \cref{thm:q_convergence}}\label{sec:appendix-proof-q-convergence}

\qFunctionConvergence*
\begin{proof}
We present the Richardson-based statement; the non-Richardson bound follows by replacing $u^2$ by $u$.

\paragraph{Step 1: Perturbed contraction recursion.}
Let $V^*:=V^{(\alpha)}$ denote the unique fixed point of $\Phi_\tau^{(\alpha)}$.
By the contraction property in \cref{lemma:Phi-contraction},
\[
  \|\Phi_\tau^{(\alpha)}(V)-\Phi_\tau^{(\alpha)}(V^*)\|_\infty
  \le e^{-\beta\tau}\,\|V-V^*\|_\infty.
\]
For the Richardson Picard-Hamiltonian iteration $V_{k+1}:=T_{\tau,u}^{R,(\alpha)}(V_k)$,
\[
\begin{aligned}
  \|V_{k+1}-V^*\|_\infty
  &= \|T_{\tau,u}^{R,(\alpha)}(V_k)-V^*\|_\infty \\
  &\le \|\Phi_\tau^{(\alpha)}(V_k)-\Phi_\tau^{(\alpha)}(V^*)\|_\infty
      +\|\Phi_\tau^{(\alpha)}(V_k)-T_{\tau,u}^{R,(\alpha)}(V_k)\|_\infty \\
  &\le e^{-\beta\tau}\,\|V_k-V^*\|_\infty + C(\tau^2+\tau u^2),
\end{aligned}
\]
where the last step uses \eqref{eq:appendix-one-step-mismatch-Richardson}. Again, a simple inductive argument gives:
\begin{equation}\label{eq:appendix-V-bound-Richardson-final}
  \|V_k-V^*\|_\infty
  \le e^{-\beta\tau k}\|V_0-V^*\|_\infty
      +\frac{C(\tau^2+\tau u^2)}{1-e^{-\beta\tau}}.
\end{equation}

\paragraph{Step 2: $q$-convergence from value convergence.}
By the triangle inequality,
\[
  \|\tilde q_{V_k}^u-q_{V^*}\|_\infty
  \le \|\tilde q_{V_k}^u-\tilde q_{V^*}^u\|_\infty
     +\|\tilde q_{V^*}^u-q_{V^*}\|_\infty.
\]
The first term is controlled by Lemma~\ref{lemma:qV-u-Lipschitz-in-V}:
\[
  \|\tilde q_{V_k}^u-\tilde q_{V^*}^u\|_\infty
  \le \frac{6}{u}\,\|V_k-V^*\|_\infty.
\]
The second term is the modelling bias, bounded by Lemma~\ref{lemma:qV-u-Richardson-bias}:
\[
  \|\tilde q_{V^*}^u-q_{V^*}\|_\infty \le \tilde C_q u^2.
\]
Combining these yields
\begin{equation}\label{eq:appendix-q-bound-Richardson-final}
  \|\tilde q_{V_k}^u-q_{V^*}\|_\infty
  \le \frac{6}{u}\,\|V_k-V^*\|_\infty + \tilde C_q u^2.
\end{equation}
Substituting \eqref{eq:appendix-V-bound-Richardson-final} into
\eqref{eq:appendix-q-bound-Richardson-final} gives the stated $q$-bound.
\end{proof}

Finally, the bound contains the factor $1/u$ multiplying the value error. Thus, to guarantee $\tilde q_{V_k}^u\to q_{V^*}$ as $(\tau,u)\to(0,0)$, it suffices that $u \to 0$ and $\tau/u \to 0$ so that $\frac{1}{u}\cdot(\tau+u^2) \to 0$), and the errors in \cref{thm:q_convergence} go to $0$.
\section{Proof of convergence for CT-SAC (\cref{thm:algorithm-convergence} and \cref{corollary:algorithm_statistical-convergence})}
\label{sec:appendix:algorithm_convergence}
\subsection{From $(V_k,q_k)$ to a single $Q_k$}
\label{sec:appendix:population-equivalence}
We prove that the single-critic update is an algebraic re-expression of the decoupled $(V_k,q_k)$ iteration, hence inherits its convergence guarantees. Fix an algorithmic step size $\tau>0$ and a holding time $u>0$ with $\gamma:=e^{-\beta u}$. Recall that
\begin{equation}
  q_V^u(x,a)
  :=
  \frac{ \gamma\,\E_x^a[V(X_u)] - V(x)}{u} + r(x,a).
  \label{eq:appendix-qVu-definition}
\end{equation}
We will also use the shorthand
\begin{equation}
  \Delta_u V(x,a):=\gamma\,\E_x^a[V(X_u)] - V(x),
  \qquad
  q_V^u(x,a)=\frac{\Delta_u V(x,a)}{u}+r(x,a).
  \label{eq:appendix-Delta-u}
\end{equation}

\begin{lemma}[One-step $Q$-update induced by $(V,q)$]
\label{lemma:appendix-Q-from-Vq}
Assume that at iteration $k$ we have a pair $(V_k,q_k)$ and a policy $\pi_k$ such that
\begin{equation}
  V_{k+1}(x) = V_k(x)
  +\tau \mathcal{Q}^{(\alpha)}(q_k)(x) \quad \text{(Hamiltonian flow update)},
  \label{eq:appendix-V-update}
\end{equation}
and the finite-horizon approximation $q_{k+1}$:
\begin{equation}
  q_{k+1}(x,a)
  =
  \frac{\gamma\,\E_x^a[V_{k+1}(X_u)]-V_{k+1}(x)}{u}+r(x,a).
  \label{eq:appendix-q-update}
\end{equation}
Define the single critic $Q_k(x,a):=V_k(x)+q_k(x,a)$.
Then $Q_{k+1}:=V_{k+1}+q_{k+1}$ satisfies the closed-form update
\begin{equation}
\begin{aligned}
  Q_{k+1}(x,a)
  &=
  (1-\tau)Q_k(x,a)
  + \tau r(x,a)
  + \tau\,\E_{a\sim\pi_k(\cdot\mid x)}\big[\tilde Q_k(x,a)\big]  \\
  &\quad
  + \tau\,\frac{
    \gamma\,\E_{a'\sim\pi_k(\cdot\mid X_u)}\big[\tilde Q_k(X_u,a')\big]
    - \E_{a\sim\pi_k(\cdot\mid x)}\big[\tilde Q_k(x,a)\big]
  }{u},
\end{aligned}
  \label{eq:appendix-Q-update-closed}
\end{equation}
where $\tilde Q_k(x, a)$ is defined as $Q_k(x,a)-\alpha\log\pi_k(a\mid x)$ with $\pi_k(a\mid x) \sim \exp(q_k(x, a)/\alpha) \sim \exp(Q_k(x, a)/\alpha)$.

Note that this is the same equation as \cref{eq:main_Q_update}. We state the equation again here for convenience reference.
\end{lemma}
\begin{proof}
From \cref{lemma:entropy-KL-identity}, we immediately have $V_{k+1}(x) = V_k(x) +\tau \,\E_{a\sim\pi_k(\cdot\mid x)}\Big[q_k(x,a)-\alpha\log\pi_k(a\mid x)\Big]$, for $\pi_k(.|x) \sim \exp(q_k(x, .)/\alpha)$. Using $q_k=Q_k-V_k$ gives:
\begin{equation}
\begin{aligned}
  V_{k+1}(x)
  &= V_k(x)
  +\tau\,\E_{a\sim\pi_k}\Big[Q_k(x,a)-V_k(x)-\alpha\log\pi_k(a\mid x)\Big] \\
  &=
  (1-\tau)V_k(x) + \tau\,\E_{a\sim\pi_k(\cdot\mid x)}\big[\tilde Q_k(x,a)\big].
\end{aligned}
  \label{eq:appendix-V-update-tildeQ}
\end{equation}
Plugging \Cref{eq:appendix-V-update-tildeQ} into \Cref{eq:appendix-q-update} gives
\begin{equation}
\begin{aligned}
  q_{k+1}(x,a)
  &=
  \frac{1}{u}\Big(
    \gamma\,\E_x^a[V_{k+1}(X_u)]-V_{k+1}(x)
  \Big)+r(x,a) \\
  &=
  \frac{1}{u}\Big(
    \gamma(1-\tau)\E_x^a[V_k(X_u)] + \gamma\tau\,\E_x^a\!\Big[\E_{a'}[\tilde Q_k(X_u,a')]\Big] \\
  &\qquad\qquad
    -(1-\tau)V_k(x) - \tau\,\E_a[\tilde Q_k(x,a)]
  \Big)+r(x,a).
\end{aligned}
  \label{eq:appendix-q-expand}
\end{equation}
Rearranging the terms, we get (we drop $\E_x^a$ for the last term for readability):
\begin{equation}
\begin{aligned}
  q_{k+1}(x,a)
  &=
  (1-\tau)\frac{\gamma\E_x^a[V_k(X_u)]-V_k(x)}{u} + r(x,a)
  + \frac{\tau}{u}
    \Big(
      \gamma\,\E_{a'}[\tilde Q_k(X_u,a')] - \E_a[\tilde Q_k(x,a)]
    \Big) \\
  &=
  (1-\tau)\big(q_k(x,a)-r(x,a)\big) + r(x,a)
  + \frac{\tau}{u}
    \Big(
      \gamma\,\E_{a'}[\tilde Q_k(X_u,a')] - \E_a[\tilde Q_k(x,a)]
    \Big).
\end{aligned}
  \label{eq:appendix-q-rearranged}
\end{equation}
Finally, add $V_{k+1}(x)$ from \cref{eq:appendix-q-rearranged} to obtain
\begin{equation}
\begin{aligned}
  Q_{k+1}(x,a)
  &=V_{k+1}(x)+q_{k+1}(x,a) \\
  &=
  (1-\tau)V_k(x)+\tau\E_a[\tilde Q_k(x,a)]
  +(1-\tau)q_k(x,a)+\tau r(x,a) \\
  &\quad
  +\frac{\tau}{u}\Big(
      \gamma\,\E_{a'}[\tilde Q_k(X_u,a')] - \E_a[\tilde Q_k(x,a)]
    \Big) \\
  &=
  (1-\tau)Q_k(x,a)
  +\tau r(x,a)
  +\tau\E_a[\tilde Q_k(x,a)]
  +\tau\frac{\gamma\E_{a'}[\tilde Q_k(X_u,a')] - \E_a[\tilde Q_k(x,a)]}{u}.
\end{aligned}
\end{equation}
This is exactly \cref{eq:appendix-Q-update-closed} (or \cref{eq:main_Q_update}).
\end{proof}

\textbf{Note.} For implementation, it is convenient to introduce a ``fast'' update
\begin{equation}
Q_{k+1}^{\text{fast}}(s, a) = \frac{Q_{k+1}(s, a) - (1-\tau)Q_k(s, a)}{\tau}
\end{equation}
The full update can then be written as:
\begin{align}
Q_{k+1}^{\text{fast}}(s, a) &= r(s, a) + \E_a\big[\tilde{Q}_k(s, a)\big] + \frac{\gamma \E_{a'}\big[\tilde{Q}_k(s', a')\big] - \E_a\big[\tilde{Q}_k(s, a)\big]}{u} \\
Q_{k+1}(s, a) &= \tau Q_{k+1}^{\text{fast}}(s, a) + (1-\tau) Q_k(s, a) \label{eq:polyak_similar_update}
\end{align}
When $u = 1$, this reduces to a $Q$-update closely aligned with the discrete-time counterpart. Below is a deterministic version (CT-TD3) of CT-SAC (See \cref{algo:ct_td3}).

\begin{algorithm}[t]
  \caption{Continuous-time TD3 (CT-TD3)}
  \label{algo:ct_td3}
  \begin{algorithmic}[1]
    \STATE \textbf{Input:} discount rate $\beta$, Euler step $\tau>0$, actor delay $d\in\mathbb{N}$,
    exploration std $\sigma_{\mathrm{expl}}$,
    target smoothing std $\sigma_{\mathrm{targ}}$, clip $c>0$.
    \STATE Initialize twin critics $Q_{\theta_1},Q_{\theta_2}$ and deterministic actor $\mu_{\phi}$.
    \STATE Initialize target networks $Q_{\bar\theta_1}\leftarrow Q_{\theta_1}$, $Q_{\bar\theta_2}\leftarrow Q_{\theta_2}$, $\mu_{\bar\phi}\leftarrow \mu_{\phi}$.
    \FOR{$k = 0,1,2,\dots$}
      \STATE Collect transitions $(x, a, r, x', u)$, where $x=X_t$, $x'=X_{t+u}$, and $u$ can vary across samples.
      Use behavior action $a=\mu_\phi(x)+\epsilon$ with $\epsilon\sim\mathcal{N}(0,\sigma_{\mathrm{expl}}^2 I)$ (clip to action bounds if needed). Store in an off-policy replay buffer.
      \STATE Sample a mini-batch $\mathcal{B}$ of $(x,a,r,x',u)$ from the replay buffer and set $\gamma = e^{-\beta u}$ per sample.

      \STATE \textbf{Target policy smoothing:}
      \[
        \tilde a' = \mu_{\bar\phi}(x') + \epsilon',\qquad
        \epsilon'\sim \mathrm{clip}\big(\mathcal{N}(0,\sigma_{\mathrm{targ}}^2 I),-c,c\big),
      \]

      \STATE Compute clipped double-$Q$ bootstrap:
      \[
        \bar Q_{\bar\theta}(x',\tilde a') := \min\{Q_{\bar\theta_1}(x',\tilde a'),\,Q_{\bar\theta_2}(x',\tilde a')\}.
      \]

      \STATE \textbf{Critic updates:} for $i\in\{1,2\}$, form the target
      \[
      y_i
      :=
      (1-\tau)Q_{\bar\theta_i}(x, a)
      + \tau\,\bar Q_{\bar\theta}\big(x,\mu_{\bar\phi}(x)\big)
      + \tau\,r
      + \tau\,\frac{\gamma\,\bar Q_{\bar\theta}(x',\tilde a')-\bar Q_{\bar\theta}\big(x,\mu_{\bar\phi}(x)\big)}{u},
      \]
      and update $\theta_i$ by regression on $\mathcal{B}$:
      \[
        \theta_i \leftarrow \arg\min_{\theta}\;
        \E_{(x,a,r,x',u)\sim\mathcal{B}}
        \big(Q_{\theta}(x,a)-y_i\big)^2.
      \]

      \IF{$k \bmod d = 0$}
        \STATE \textbf{Delayed policy update:} update actor using critic $Q_{\theta_1}$:
        \[
          \phi \leftarrow \phi + \eta_\pi\,
          \E_{x\sim\mathcal{B}}
          \Big[\nabla_\phi \mu_\phi(x)\,
          \nabla_a Q_{\theta_1}(x,a)\big|_{a=\mu_\phi(x)}\Big].
        \]
      \ENDIF
    \ENDFOR
  \end{algorithmic}
\end{algorithm}

\textbf{Richardson update.} We now state the Richardson version of $Q$-function update. Fix a holding time $u>0$ and define $\gamma:=e^{-\beta u}$ and $\gamma_{1/2}:=e^{-\beta u/2}$. Let $X_{u/2}$ denote the state reached after time $u/2$ when starting from $(x,a)$. Define the policy-averaged soft value functional
\begin{equation}
  S_k(x)
  :=
  \E_{a\sim\pi_k(\cdot\mid x)}\!\big[\tilde Q_k^{\mathrm{R}}(x,a)\big],
  \qquad
  \tilde Q_k^{\mathrm{R}}(x,a):=Q_k^{\mathrm{R}}(x,a)-\alpha\log\pi_k(a\mid x),
  \label{eq:appendix-richardson-Sk}
\end{equation}
and the Richardson critic decomposition
\begin{equation}
  Q_k^{\mathrm{R}}(x,a)
  :=
  V_k(x) + \tilde q_{V_k}^u(x,a),
  \label{eq:appendix-richardson-Q-decomp}
\end{equation}
Then the induced single-critic update of $Q_k^{\mathrm{R}}$ has the form:
\begin{equation}
\begin{aligned}
  Q_{k+1}^{\mathrm{R}}(x,a)
  &=
  (1-\tau)Q_k^{\mathrm{R}}(x,a)
  + \tau r(x,a)
  + \tau S_k(x) \\
  &\quad
  + \frac{\tau}{u}\Big(
      4\gamma_{1/2}\,\E_x^a\!\big[S_k(X_{u/2})\big]
      - \gamma\,\E_x^a\!\big[S_k(X_{u})\big]
      - 3 S_k(x)
  \Big).
\end{aligned}
  \label{eq:appendix-richardson-Q-update-closed}
\end{equation}
Detailed derivations are given in \cref{lemma:appendix-richardson-Q-algebra}. We now continue with the algorithm convergence.

\begin{lemma}[Inductive preservation of $Q_k=V_k+q_k$]
\label{lemma:appendix-induction-Q-decomp}
Assume $(V_0,q_0)$ satisfy the finite-horizon relation $q_0=q_{V_0}^u$ from
\Cref{eq:appendix-qVu-definition} and define $Q_0:=V_0+q_0$.
Let $(V_k,q_k)$ be generated by \Cref{eq:appendix-V-update}--\Cref{eq:appendix-q-update}
and let $Q_k$ be generated by \Cref{eq:appendix-Q-update-closed}.
Then for all $k\ge 0$,
\begin{equation}
  Q_k(x,a)=V_k(x)+q_k(x,a).
  \label{eq:appendix-Q-decomp-induction}
\end{equation}
\end{lemma}

\begin{proof}
By construction, \Cref{eq:appendix-Q-decomp-induction} holds for $k=0$.
Assume it holds for some $k$. By \Cref{lemma:appendix-Q-from-Vq},
the update \Cref{eq:appendix-Q-update-closed} produced from $(V_k,q_k)$ satisfies
$Q_{k+1}=V_{k+1}+q_{k+1}$. This completes the induction.
\end{proof}

We now show that the single critic converges by combining the value and rate convergence bounds.

\TheoreticalAlgorithmConvergence*
\begin{proof}
From \cref{lemma:appendix-induction-Q-decomp}, we get:
\begin{equation}
  Q_k - Q^{(\alpha)}
  =
  (V_k - V^{(\alpha)}) + (q_k - q_{V^{(\alpha)}}),
  \label{eq:appendix-Q-minus-opt}
\end{equation}
Therefore, by the triangle inequality,
\begin{equation}
  \|Q_k-Q^{(\alpha)}\|_\infty
  \le
  \|V_k-V^{(\alpha)}\|_\infty + \|q_k-q_{V^{(\alpha)}}\|_\infty.
  \label{eq:appendix-Q-triangle}
\end{equation}
The terms are controlled by \cref{thm:value_convergence} and \cref{thm:q_convergence} (with $q_k=q_{V_k}$ in the idealized setting, or $q_k=\tilde q_{V_k}^u$ in the model-free setting). Combining the two bounds yields convergence of $Q_k$ to $Q^{(\alpha)}=V^{(\alpha)}+q_{V^{(\alpha)}}$,
with the error rates given by \cref{thm:value_convergence} and \cref{thm:q_convergence}.
\end{proof}

We now sketch a standard finite-sample extension that converts the exact convergence into a statistical bound for the learned critic $\hat Q_k$ produced by neural network fitting. We first state standard i.i.d.\ generalization bounds that we use to control the statistical error of the critic regression step.
\subsection{Learning-theoretic ingredients (i.i.d.\ regression)}
\label{sec:appendix:learning-theory}
Throughout this section, $\mathcal{X}$ denotes an input space,
$\mathcal{Y}\subset \mathbb{R}$ a bounded label space, and $\mathcal{Z}=\mathcal{X}\times \mathcal{Y}$.
Let $\mathcal{H}$ be a hypothesis class of functions $h:\mathcal{X}\to\mathcal{Y}$ and let
$l(h,(x,y))=\phi(h(x),y)$ be a loss with $\phi(\cdot,y)$ $\gamma$-Lipschitz for all $y$.
For i.i.d.\ samples $Z_i=(X_i,Y_i)\sim \mathbb{P}$, define the population and empirical risks
\begin{equation}
  e(h):=\E[l(h,Z)],\qquad
  E_n(h):=\frac{1}{n}\sum_{i=1}^n l(h,Z_i).
  \label{eq:appendix-risk-def}
\end{equation}

\begin{definition}[Rademacher complexity]
For $\mathcal{T}\subseteq \mathbb{R}^n$, its Rademacher complexity is
\begin{equation}
  \mathrm{Rad}(\mathcal{T})
  :=
  \E\bigg[\sup_{t\in\mathcal{T}} \frac{1}{n}\sum_{i=1}^n B_i t_i\bigg],
  \label{eq:appendix-rad-def}
\end{equation}
where $B_i$ are i.i.d.\ Rademacher random variables.
\end{definition}

Besides Rademacher complexity, we also define the following sets: $\mathcal{H} \circ \set{Z_1, \cdots, Z_n}: = \set{(h(X_1), \cdots, h(X_n)),\ h \in \mathcal{H}} \subseteq \mathbb{R}^n$, and $\mathcal{L}\circ\set{Z_1, \cdots, Z_n}: = \set{(l(h, Z_1), \cdots, l(h, Z_n)),\ h \in \mathcal{H}} \subseteq \mathbb{R}^n$. We state the following standard PAC-type results \cite{lecture_note}:

\begin{lemma}[Lipschitz contraction for losses]
\label{lemma:appendix-loss-lipschitz}
If $\phi(\cdot,y)$ is $\gamma$-Lipschitz for all $y$, then
\begin{equation}
  \E\Big[\sup_{h\in\mathcal{H}}(e(h)-E_n(h))\Big]
  \le
  2\,\E\big[\mathrm{Rad}(\mathcal{L}\circ\{Z_1,\dots,Z_n\})\big]
  \le
  2\gamma\,\E\big[\mathrm{Rad}(\mathcal{H}\circ\{Z_1,\dots,Z_n\})\big].
  \label{eq:appendix-lipschitz-rad}
\end{equation}
\end{lemma}

\begin{lemma}[High-probability uniform bound]
\label{lemma:appendix-generalization-iid}
Assume $l(h,Z)\in[0,c]$ almost surely. Then with probability at least $1-\delta$,
\begin{equation}
  \sup_{h\in\mathcal{H}} \big(e(h)-E_n(h)\big)
  \le
  4\,\mathrm{Rad}(\mathcal{L}\circ\{Z_1,\dots,Z_n\})
  + c\sqrt{\frac{2\log(1/\delta)}{n}}.
  \label{eq:appendix-generalization}
\end{equation}
\end{lemma}

\begin{lemma}[Rademacher complexity for bounded neural nets]
\label{lemma:appendix-rad-nn}
For a hypothesis class $\mathcal{H}$ consisting of (regularized) neural networks with bounded
weights, there exist constants $C_1,C_2>0$ such that for $n$ i.i.d.\ samples,
\begin{equation}
  \mathrm{Rad}(\mathcal{H}\circ\{Z_1,\dots,Z_n\})
  \le
  \frac{1}{\sqrt{n}}\big(C_1 + C_2\sqrt{\log d}\big),
  \label{eq:appendix-rad-nn}
\end{equation}
where $d$ is an ambient dimension parameter of the network class.
\end{lemma}

\subsection{Proof of \cref{corollary:algorithm_statistical-convergence}}
\label{sec:appendix:population-Q-convergence}

\begin{definition} With $\tilde Q(x,a):=Q(x,a)-\alpha\log\pi(a\mid x)$, and $\pi(a|x) \sim \exp(Q(x, a)/\alpha)$, we define the single-critic update operator arising from \cref{eq:main_Q_update}:
\begin{equation}
\label{eq:appendix-F-tau-u-nonrich}
\begin{aligned}
  \big(\mathcal{F}_{\tau,u}(Q)\big)(x,a)
  :=
  (1-\tau)Q(x,a)
  &+ \tau r(x,a)
  + \tau \E_{a\sim\pi(\cdot\mid x)}[\tilde Q(x,a)] \\
  & + \tau\frac{\gamma\,\E_{a'\sim\pi(\cdot\mid x')}[\tilde Q(x',a')] -\E_{a\sim\pi(\cdot\mid x)}[\tilde Q(x,a)]}{u}
\end{aligned}
\end{equation}
\end{definition}
\begin{definition}
With the Richardson q-function $\tilde q_V^u := 2q_V^{u/2}-q_V^u$, the single-critic update operator is:
\begin{equation}
\label{eq:appendix-F-tau-u-rich}
\begin{aligned}
  \big(\mathcal{F}^{\mathrm{Rich}}_{\tau,u}(Q)\big)(x,a)
  &:=
  (1-\tau)Q(x,a)
  +\tau r(x,a)
  +\tau \E_{a\sim\pi(\cdot\mid x)}[\tilde Q(x,a)] \\
  &\quad
  +\frac{\tau}{u}\Big(
      4e^{-\beta u/2}\E_{a'\sim\pi(\cdot\mid x_{u/2}')}\big[\tilde Q(x_{u/2}',a')\big]
      -e^{-\beta u}\E_{a''\sim\pi(\cdot\mid x_{u}')}\big[\tilde Q(x_{u}',a'')\big]
      -3\E_{a\sim\pi(\cdot\mid x)}[\tilde Q(x,a)]
    \Big)
\end{aligned}
\end{equation}
where $x_{u/2}'=X_{t+u/2}$ and $x_u'=X_{t+u}$.
\end{definition}

\textbf{Approximation procedure.}
At each iteration $k$, we draw $n_k$ i.i.d.\ transition samples
\begin{equation}
  (X_i,A_i,R_i,X_i',U_i)_{i=1}^{n_k},
  \label{eq:appendix-samples}
\end{equation}
from a replay distribution (treated as i.i.d.\ for analysis), and form a bounded regression target
$Y_k(x,a,r,x',u)$ corresponding to the right-hand side of the population update \cref{eq:main_Q_update}. We then fit $\hat Q_{k+1}\in\mathcal{H}_Q$ by (approximately) minimizing the empirical loss $\frac{1}{n_k}\sum_{i=1}^{n_k}\phi(\hat Q_{k+1}(X_i,A_i),Y_k(Z_i))$. Assume the output range of $\mathcal{H}_Q$ and the targets are uniformly bounded. Let $\nu$ denote the state--action sampling distribution induced by the data collection scheme
(e.g.\ replay buffer sampling), and define the $L^2(\nu)$ norm by:
\[
  \|f\|_{L^2(\nu)}:=\Big(\E_{(X,A)\sim\nu}[\,|f(X,A)|^2\,]\Big)^{1/2}.
\]
Let $Q_k$ denote the exact iterate and $\hat Q_{k+1}$ the learned iterate obtained by regression at iteration $k+1$. At each step $k$, the target $Y_k$ is exactly the value $\big(\mathcal{F}_{\tau,u}(Q)\big)$ on realized samples. As a result, thanks to \cref{lemma:appendix-loss-lipschitz,lemma:appendix-generalization-iid,lemma:appendix-rad-nn}, we can estimate the one-step statistical error (which is a random variable):
\[
  \mathrm{stat}_k:=\|\hat Q_{k+1}-\mathcal{F}_{\tau,u}(\hat Q_k)\|_{L^2(\nu)}.
\]
In this case, we have:
\begin{equation}
  \mathrm{stat}_k
  \;\lesssim\;
  \frac{1}{\sqrt{n_k}}
  \bigg(1+\sqrt{\log d}+\sqrt{\log(1/\delta)}\bigg)
  \qquad\text{with probability\ $\ge 1-\delta$},
  \label{eq:appendix-stat-rate}
\end{equation}
Note that the replay buffer sampling distribution $\nu = \nu_k$ can change at each iterations. However, for large enough $n_k$, we can easily transfer from $L^2(\nu)$ to $L^2(\nu_k)$ given that $\nu_k$ is regular enough. Hence, for simplicity, we proceed with a fixed replay buffer sampling $\nu$.

\begin{lemma}[Lispchitz of critic mapping]\label{lemma:appendix-Q-map-stability-L2}
Assume the policy $\pi_k(\cdot\mid x)$ used inside $\tilde Q_k(x,a):=Q_k(x,a)-\alpha\log\pi_k(a\mid x)$
is treated as fixed when comparing two critics (policy-evaluation step).
Then the mapping $\mathcal{F}_{\tau,u}$ is Lipschitz in $L^2(\nu)$:
\begin{equation}\label{eq:appendix-Q-map-stability-L2}
  \|\mathcal{F}_{\tau,u}(Q)-\mathcal{F}_{\tau,u}(Q')\|_{L^2(\nu)}
  \le \Big(1+\frac{\tau}{u}\Big)\|Q-Q'\|_{L^2(\nu)}.
\end{equation}
\end{lemma}

\begin{proof}
The update $\mathcal{F}_{\tau,u}$ is affine in $Q$ through terms of the form
$Q(x,a)$, $\E_{a\sim\pi(\cdot\mid x)}[Q(x,a)]$, and $\E_{a'\sim\pi(\cdot\mid x')}[Q(x',a')]$,
all multiplied by coefficients bounded by $1+\tau/u$.
Using Jensen's inequality for conditional expectations $\big\|\E[Q(X,\cdot)\mid X]\big\|_{L^2}\le \|Q(X,\cdot)\|_{L^2}$ yields the stated Lipschitz bound.
\end{proof}

\AlgorithmConvergence*
\begin{proof}
Given an $\epsilon > 0$, from \cref{thm:algorithm-convergence}, given the algorithmic error go to $0$, we can choose an iteration $L$ so that $\|Q_L-Q^{(\alpha)}\|_{L^2(\nu)} \leq \|Q_L-Q^{(\alpha)}\|_{\infty} < \epsilon/2$. Additionally, choose $n_j$ large enough, so that $\mathrm{stat^{(2)}_j} < (1/L_0^{L-1-j})\dfrac{\epsilon}{2L}$ for probability of at least $1 - \delta/k$ for $L_0 = 1 + \tau/u$.

Now let $Q_k$ be the exact iteration results and $\hat Q_k$ the learned one. Define the statistical deviation in $L^2(\nu)$:
\[
  \epsilon_k^{(2)}:=\|\hat Q_k-Q_k\|_{L^2(\nu)}.
\]
Then we have the decomposition,
\begin{equation}
\begin{aligned}
  \epsilon_{k+1}^{(2)}
  &=
  \|\hat Q_{k+1}-Q_{k+1}\|_{L^2(\nu)} \\
  &\le
  \|\hat Q_{k+1}-\mathcal{F}_{\tau,u}(\hat Q_k)\|_{L^2(\nu)}
  +
  \|\mathcal{F}_{\tau,u}(\hat Q_k)-\mathcal{F}_{\tau,u}(Q_k)\|_{L^2(\nu)}.
\end{aligned}
  \label{eq:appendix-error-split-L2}
\end{equation}
The first term is the one-step regression error $\mathrm{stat}_k$. The second term is bounded by \cref{lemma:appendix-Q-map-stability-L2}:
\[
  \|\mathcal{F}_{\tau,u}(\hat Q_k)-\mathcal{F}_{\tau,u}(Q_k)\|_{L^2(\nu)}
  \le
  \Big(1+\frac{\tau}{u}\Big)\epsilon_k^{(2)}.
\]
Hence,
\begin{equation}\label{eq:appendix-error-recursion-L2}
  \epsilon_{k+1}^{(2)}
  \le
  \mathrm{stat}_k
  +
  \Big(1+\frac{\tau}{u}\Big)\epsilon_k^{(2)}.
\end{equation}
Given the same initial point with $\epsilon_0^{(2)} = 0$, an inductive argument gives:
\begin{equation}\label{eq:appendix-error-unroll-L2}
  \epsilon_{L}^{(2)}
  \le
  \sum_{j=0}^{L-1}\Big(1+\frac{\tau}{u}\Big)^{L-1-j}\mathrm{stat}_j = \sum_{j=0}^{L-1}L_0^{L-1-j}\mathrm{stat}_j
\end{equation}
Based on the choice of $n_k$, the RHS can further be bounded by $\sum_{j=0}^{L-1} \epsilon/(2L) = \epsilon/2$ with probability of at least $1 - k(\delta/k) = 1 - \delta$. Finally, with probability of at least $1-\delta$, we have the decomposition:
\begin{equation}\label{eq:appendix-total-error-L2}
  \|\hat Q_L-Q^{(\alpha)}\|_{L^2(\nu)}
  \le
  \underbrace{\|Q_L-Q^{(\alpha)}\|_{L^2(\nu)}}_{\text{algorithmic error}}
  +
  \underbrace{\|\hat Q_L-Q_L\|_{L^2(\nu)}}_{=\epsilon_L^{(2)}\ \text{statistical error}} < \epsilon/2 + \epsilon/2 = \epsilon
\end{equation}
\end{proof}

We defer the theoretical analysis with more refined control of sample complexity and regret bound to \cref{sec:appendix:regret_bound}
\section{Random time step $U$}\label{sec:appendix:random_time}
Many environments produce irregular times $U$, and here we provide similar convergence results for (random time) $U$. The deterministic finite-horizon approximation results extend directly by conditioning on $U$ and averaging.
The key observation is that any deterministic Hamiltonian bias bound in $u$ transfers to random $U$
by replacing $u^\gamma$ with $\E[U^\gamma]$.

\begin{lemma}[Hamiltonian bias]\label{lemma:random-time-H-transfer}
Fix $\alpha\ge 0$ and a bounded test function $V$.
Let $U$ be a random variable taking values in $(0,1]$, independent of the state and control.
Assume that for some $\gamma>0$ and constant $C_H<\infty$,
\begin{equation}\label{eq:deterministic-H-bias-general}
  \|H_u^{(\alpha)}(V)-H^{(\alpha)}(V)\|_\infty \le C_H\,u^\gamma,
  \qquad \forall u\in(0,1].
\end{equation}
Define the random Hamiltonian $H_U^{(\alpha)}(V)$ by evaluation at $u=U$,
and its averaged version $\bar H_U^{(\alpha)}(V):=\E[H_U^{(\alpha)}(V)]$ (expectation over $U$ only).
Then,
\begin{equation}\label{eq:random-time-H-bias-general}
  \|\bar H_U^{(\alpha)}(V)-H^{(\alpha)}(V)\|_\infty
  \le C_H\,\E[U^\gamma].
\end{equation}
The same statement holds if $H_u^{(\alpha)}$ is replaced by the Richardson Hamiltonian $\tilde H_u^{(\alpha)}$.
\end{lemma}

\begin{proof}
By Jensen and the deterministic bound \eqref{eq:deterministic-H-bias-general},
\[
  \|\bar H_U^{(\alpha)}(V)-H^{(\alpha)}(V)\|_\infty
  = \Big\|\E\big[H_U^{(\alpha)}(V)-H^{(\alpha)}(V)\big]\Big\|_\infty
  \le \E\big[\|H_U^{(\alpha)}(V)-H^{(\alpha)}(V)\|_\infty\big]
  \le C_H\,\E[U^\gamma].
\]
The Richardson case is identical.
\end{proof}

\begin{theorem}[Picard iteration with random time step]\label{theorem:random-U-Picard}
Fix $\alpha\ge0$ and $\tau>0$.
Let $\Phi_\tau^{(\alpha)}$ be the dynamic-programming semigroup from
\eqref{eq:appendix-Phi-definition}. Consider the averaged Picard operator
\begin{equation}\label{eq:random-time-averaged-Picard-operator}
  (\widetilde T_{\tau}^{(\alpha)} V)(x)
  := V(x) + \tau\,\bar H_U^{(\alpha)}(V)(x)
  = V(x) + \tau\,\E\!\big[H_U^{(\alpha)}(V)(x)\big],
\end{equation}
and the iterates $\widetilde V_{k+1}:=\widetilde T_\tau^{(\alpha)}(\widetilde V_k)$.

Assume further that the local expansion error satisfies, for all $V \in \mathcal{D}$,
\begin{equation}\label{eq:random-time-local-DP-mismatch}
  \|\Phi_\tau^{(\alpha)}(V) - (V+\tau H^{(\alpha)}(V))\|_\infty \le C_{\mathrm{DP}}\,\tau^{p},
\end{equation}
for some $p>1$ and constant $C_{\mathrm{DP}}<\infty$.

In addition if the deterministic Hamiltonian bias bound \eqref{eq:deterministic-H-bias-general} holds
for some $\gamma>0$, then for all $k\ge0$,
\begin{equation}\label{eq:random-time-Picard-bound-general}
  \|\widetilde V_k - V^{(\alpha)}\|_\infty
  \le e^{-\rho\tau k}\,\|\widetilde V_0 - V^{(\alpha)}\|_\infty
  + \frac{C_{\mathrm{DP}}\tau^{p} + \tau C_H\,\E[U^\gamma]}{1-e^{-\rho\tau}}.
\end{equation}
\end{theorem}

\begin{proof}
For any $V$,
\[
  \|\widetilde T_\tau^{(\alpha)}(V) - (V+\tau H^{(\alpha)}(V))\|_\infty
  = \tau\|\bar H_U^{(\alpha)}(V)-H^{(\alpha)}(V)\|_\infty
  \le \tau C_H\,\E[U^\gamma]
\]
by Lemma~\ref{lemma:random-time-H-transfer}. Combining with the local semigroup mismatch \eqref{eq:random-time-local-DP-mismatch} gives
\[
  \|\widetilde T_\tau^{(\alpha)}(V) - \Phi_\tau^{(\alpha)}(V)\|_\infty
  \le C_{\mathrm{DP}}\tau^{p} + \tau C_H\,\E[U^\gamma].
\]
Again, using contraction of $\Phi_\tau^{(\alpha)}$ and the fixed-point identity
$\Phi_\tau^{(\alpha)}(V^{(\alpha)})=V^{(\alpha)}$ yields the recursion
\[
  \|\widetilde V_{k+1}-V^{(\alpha)}\|_\infty
  \le e^{-\rho\tau}\|\widetilde V_k-V^{(\alpha)}\|_\infty
     + C_{\mathrm{DP}}\tau^{p} + \tau C_H\,\E[U^\gamma],
\]
and a similar inductive argument proves \eqref{eq:random-time-Picard-bound-general}.
\end{proof}

\begin{corollary} Based on the error estimates from \cref{sec:appendix:q_convergence}, we can naturally extends to random time settings under the following 3 scenarios:
 
\textbf{(1) Finite-horizon Picard-Hamiltonian iterations with $q^u_V$ under standard assumption}. If $\|H_u^{(\alpha)}(V)-H^{(\alpha)}(V)\|_\infty \le C u^{1/2}$ and $\|\Phi_\tau^{(\alpha)}(V)-(V+\tau H^{(\alpha)}(V))\|_\infty \le C\tau^{3/2}$, then \eqref{eq:random-time-Picard-bound-general} holds with $(p,\gamma)=(3/2,1/2)$ and the random bias term becomes $\tau\E[U^{1/2}]$.

\textbf{(2) Finite-horizon Picard-Hamiltonian iterations with $q^u_V$ under smoother assumption}. If $\|H_u^{(\alpha)}(V)-H^{(\alpha)}(V)\|_\infty \le C u$ and the semigroup mismatch is $\mathcal{O}(\tau^2)$, then $(p,\gamma)=(2,1)$ and the random bias term becomes $\tau\E[U]$.

\textbf{(3) Richardson Picard-Hamiltonian iterations}. If $\|\tilde H_u^{(\alpha)}(V)-H^{(\alpha)}(V)\|_\infty \le C u^2$ and the semigroup mismatch is $\mathcal{O}(\tau^2)$, then $(p,\gamma)=(2,2)$ and the random bias term becomes $\tau\E[U^2]$.
\end{corollary}

\subsection{$q$-convergence}\label{sec:random-holding-time-q}

The Hamiltonian transfer in Lemma~\ref{lemma:random-time-H-transfer} is immediate because the bias bound
depends only on positive moments $\E[U^\gamma]$.
For the advantage-rate $q_V^u$, an additional stability condition is needed because the finite-horizon
definition contains a division by $u$, making the map $V\mapsto q_V^u$ Lipschitz with constant of order $1/u$.

\begin{lemma}[Random-time Lipschitz stability of $q_V^U$]\label{lemma:random-time-q-Lipschitz}
Fix $\beta>0$ and for $u \in (0,1]$. Let $U$ be a random variable taking values in $(0,1]$, independent of the state and control.
Then for any bounded $V,W$,
\begin{equation}\label{eq:random-time-q-Lipschitz}
  \E\big[\|q_V^U-q_W^U\|_\infty\big]
  \le 2\,\E\!\Big[\frac{1}{U}\Big]\;\|V-W\|_\infty,
\end{equation}
and
\begin{equation}\label{eq:random-time-qtilde-Lipschitz}
  \E\big[\|\tilde q_V^U-\tilde q_W^U\|_\infty\big]
  \le 10\,\E\!\Big[\frac{1}{U}\Big]\;\|V-W\|_\infty.
\end{equation}
In particular, if $U\ge u_{\min}>0$ almost surely, then the above bounds hold with
$\E[1/U]\le 1/u_{\min}$.
\end{lemma}

\begin{proof}
Fix $u\in(0,1]$. Using definitions of $q^u$ and cancelling $r(x,a)$, we have:
\[
  |q_V^u(x,a)-q_W^u(x,a)|
  \le \frac{1}{u}\Big(e^{-\beta u}\big|\E_x^a[V(X_u^a)-W(X_u^a)]\big| + |V(x)-W(x)|\Big)
  \le \frac{2}{u}\|V-W\|_\infty.
\]
Taking the supremum over $(x,a)$ yields $\|q_V^u-q_W^u\|_\infty\le \frac{2}{u}\|V-W\|_\infty$.
Evaluating at $u=U$ and taking expectation gives \eqref{eq:random-time-q-Lipschitz}.
For Richardson, apply the same bound to $u$ and $u/2$:
\[
  \|\tilde q_V^u-\tilde q_W^u\|_\infty
  \le 2\|q_V^{u/2}-q_W^{u/2}\|_\infty + \|q_V^{u}-q_W^{u}\|_\infty
  \le \Big(2\cdot\frac{4}{u}+\frac{2}{u}\Big)\|V-W\|_\infty
  = \frac{10}{u}\|V-W\|_\infty,
\]
and again evaluate at $u=U$ and take expectation to obtain \eqref{eq:random-time-qtilde-Lipschitz}.
\end{proof}

\begin{theorem}[Random holding times: averaged $q$-convergence]\label{theorem:random-time-q-convergence}
Let $V^*:=V^{(\alpha)}$ be the optimal value function.
Assume the random-time averaged Picard iteration $(\widetilde V_k)_{k\ge0}$ satisfy the value error bound
\begin{equation}\label{eq:random-time-value-bound-recall}
  \|\widetilde V_k - V^*\|_\infty
  \le e^{-\rho\tau k}\|\widetilde V_0 - V^*\|_\infty
  + \frac{C_{\mathrm{DP}}\tau^{p} + \tau C_H\,\E[U^\gamma]}{1-e^{-\rho\tau}}.
\end{equation}
If $\E[1/U]<\infty$, then the corresponding random-time rates obey
\begin{equation}\label{eq:random-time-q-conv}
  \E\big[\|q_{\widetilde V_k}^U - q(V^*)\|_\infty\big]
  \le 2\,\E\!\Big[\frac{1}{U}\Big]\;\|\widetilde V_k - V^*\|_\infty
  + C_q\,\E[U^\gamma],
\end{equation}
and, if Richardson is used,
\begin{equation}\label{eq:random-time-qtilde-conv}
  \E\big[\|\tilde q_{\widetilde V_k}^U - q(V^*)\|_\infty\big]
  \le 10\,\E\!\Big[\frac{1}{U}\Big]\;\|\widetilde V_k - V^*\|_\infty
  + C_q'\,\E[U^2].
\end{equation}
\end{theorem}

\begin{proof}
Decompose and apply triangle inequality:
\[
  \E\big[\|q_{\widetilde V_k}^U - q(V^*)\|_\infty\big]
  \le \E\big[\|q_{\widetilde V_k}^U - q_{V^*}^U\|_\infty\big]
      + \E\big[\|q_{V^*}^U - q(V^*)\|_\infty\big].
\]
The first term is controlled by Lemma~\ref{lemma:random-time-q-Lipschitz}.
The second term is the modelling bias, which follows from the deterministic bias bound
$\|q_{V^*}^u-q(V^*)\|_\infty\le C_q u^\gamma$ and averaging over $U$:
$\E[\|q_{V^*}^U-q(V^*)\|_\infty]\le C_q\E[U^\gamma]$.
The Richardson case is identical, using the Richardson modelling bias
$\|\tilde q_{V^*}^u-q(V^*)\|_\infty\le C_q' u^2$ and the Lipschitz bound
\eqref{eq:random-time-qtilde-Lipschitz}.
\end{proof}

\textbf{Note.} If $U\in[u_{\min},1]$ almost surely, then $\E[1/U]\le 1/u_{\min}$ and the $q$-bounds \eqref{eq:random-time-q-conv}--\eqref{eq:random-time-qtilde-conv} hold with explicit constants. This matches the typical irregular-timestep setting where the environment enforces a minimum timestep.

\section{Regret bound and sample complexity}
\label{sec:appendix:regret_bound}
This section provides a more refined statistical analysis of the practical single-critic learning algorithm in \cref{algo:ct_sac}. Recall that $Q_k$ denotes the analytic iteration results without statistical error, and the $\hat Q_k$ is the function produced by sample-based estimates (with statistical error). The goal is to control the cumulative gap $\sum_{k=0}^{L-1}\|\hat Q_k-Q^{(\alpha)}\|_\infty$. The main technique that allows a more refined error analysis is a geometric argument that upgrades $L^2(\nu)$ control to uniform $\|\cdot\|_\infty$ control for Lipschitz functions on a compact domain.

At iteration $k$, recall the update operator $\mathcal{F}_{\tau,u}$ \cref{eq:appendix-F-tau-u-nonrich} that defines the regression target on $\mathcal{Z}$:
\begin{equation*}
\begin{aligned}
  \big(\mathcal{F}_{\tau,u}(Q)\big)(x,a)
  :=
  (1-\tau)Q(x,a)
  &+ \tau r(x,a)
  + \tau \E_{a\sim\pi(\cdot\mid x)}[\tilde Q(x,a)] \\
  & + \tau\frac{\gamma\,\E_{a'\sim\pi(\cdot\mid x')}[\tilde Q(x',a')] -\E_{a\sim\pi(\cdot\mid x)}[\tilde Q(x,a)]}{u}
\end{aligned}
\end{equation*}

Recall the statistical error $\mathrm{stat}_k:=\|\hat Q_{k+1}-\mathcal{F}_{\tau,u}(\hat Q_k)\|_{L^2(\nu)}$. Also recall that thanks to \cref{lemma:appendix-loss-lipschitz,lemma:appendix-generalization-iid,lemma:appendix-rad-nn}, we have:
\begin{equation}
  \mathrm{stat}_k
  \;\lesssim\;
  \frac{1}{\sqrt{n_k}}
  \bigg(1+\sqrt{\log d}+\sqrt{\log(1/\delta)}\bigg)
  \qquad\text{with probability\ $\ge 1-\delta$},
\end{equation}

\textbf{Setup and assumptions}
Let $\mathcal{X}\subset \mathbb{R}^d$ ($d = d_{\mathcal{X}}$) be the state space and $\mathcal{A}\subset \mathbb{R}^{d_{\mathcal{A}}}$ the action space. Define the state--action domain with dimension $m$:
\[
  \mathcal{Z}:=\mathcal{X}\times\mathcal{A},\qquad m:=d_{\mathcal{X}}+d_{\mathcal{A}}.
\]
We write $z=(x,a)\in\mathcal{Z}$, equip $\mathcal{Z}$ with the Euclidean metric $d(z,z'):=\|z-z'\|_2$, and denote the sup norm on $\mathcal{Z}$ by $\|f\|_\infty=\sup_{z\in\mathcal{Z}}|f(z)|$.
Throughout, we fix an algorithmic step size $\tau>0$ and a holding time $u>0$, with $\gamma:=e^{-\beta u}$.

\begin{assumption}[Compact domain and boundedness]
\label{assumption:compact-domain}
We have the following assumptions. First, we assume state-action space $\mathcal{Z}=\mathcal{X}\times\mathcal{A}$ is compact and has diameter
$\mathrm{diam}(\mathcal{Z})\le R$ for some $R>0$. Additionally, assume the rewards are bounded: $|r(x,a)|\le r_{\max}$
\end{assumption}

\begin{assumption}[Coverage density lower bound]
\label{assumption:coverage-density-lower}
At each critic regression step $k$, training samples $Z_i=(X_i,A_i)$ are drawn i.i.d.\ from a distribution $\nu_k$ supported on $\mathcal{Z}$ whose density is bounded below by $\nu_{\min}>0$ with respect to Lebesgue measure on $\mathbb{R}^m$. Equivalently, for any measurable $E\subseteq \mathcal{Z}$,
\[
  \nu_k(E)\;\ge\;\nu_{\min}\,\mathrm{Vol}(E).
\]
\end{assumption}

\begin{assumption}[Lipschitz critic]
\label{assumption:lipschitz-critic-envelope}
All critics considered are uniformly bounded and Lipschitz on $\mathcal{Z}$:
there exist constants $B_Q,L_Q>0$ such that for every iterate $Q$ (population or learned),
\[
  \|Q\|_\infty\le B_Q,
  \qquad
  |Q(z)-Q(z')|\le L_Q\,\|z-z'\|_2,
  \qquad
  \forall z,z'\in\mathcal{Z}.
\]
Moreover, for each fixed $Q$, the population mapping $\mathcal{F}_{\tau,u}(Q)$ is also Lipschitz with
\[
  \|\mathcal{F}_{\tau,u}(Q)\|_\infty\le B_F,
  \qquad
  \mathrm{Lip}\big(\mathcal{F}_{\tau,u}(Q)\big)\le L_F,
\]
for constants $B_F,L_F$ that do not depend on $k$.
(The same holds for the Richardson mapping $\mathcal{F}_{\tau,u}^{\mathrm{Rich}}$.)
\end{assumption}

This \cref{assumption:lipschitz-critic-envelope} holds true by a similar argument to the proof of \cref{lemma:appendix-Q-map-stability-L2}

\subsection{A norm-transfer argument from $L^2$ to $\|\cdot\|_\infty$}\label{sec:appendix:covering}

We start with a few lemmas from a standard norm-transfer argument.

\begin{lemma}[Uniform lower bound on ball mass]
\label{lemma:appendix-ball-mass-lower}
Under \cref{assumption:coverage-density-lower}, for any $z\in\mathcal{Z}$ and any $r>0$ with
$B_r(z)\subseteq\mathcal{Z}$,
\begin{equation}
\label{eq:appendix-ball-mass-lower}
  \nu\big(B_r(z)\big)
  \;\ge\;
  \nu_{\min}\,\mathrm{Vol}\big(B_r(z)\big)
  \;=\;
  \nu_{\min}\,c_m\,r^m,
\end{equation}
where $c_m:=\mathrm{Vol}(B_1(0))$ is the volume of the unit ball in $\mathbb{R}^m$.
\end{lemma}

\begin{lemma}[$L^2(\nu)$ to sup-norm]
\label{lemma:appendix-L2-to-sup}
Let $g:\mathcal{Z}\to\mathbb{R}$ be $L_g$-Lipschitz and suppose $\|g\|_{L^2(\nu)}\le \varepsilon$.
Under \Cref{assumption:coverage-density-lower}, one has
\begin{equation}
\label{eq:appendix-L2-to-sup}
  \|g\|_\infty
  \;\le\;
  \Big(\frac{2^{m+2}L_g^m}{\nu_{\min}c_m}\Big)^{\!\frac{1}{m+2}}\,
  \varepsilon^{\frac{2}{m+2}}.
\end{equation}
\end{lemma}

\begin{proof}
Let $z^*\in\arg\max_{z\in\mathcal{Z}}|g(z)|$ and set $M:=\|g\|_\infty=|g(z^*)|$.
Choose $r:=M/(2L_g)$. By Lipschitzness, for any $z\in B_r(z^*)$,
\[
  |g(z)|\ge |g(z^*)|-L_g\|z-z^*\|_2 \ge M - L_g r = M/2.
\]
Therefore,
\[
  \varepsilon^2
  \;\ge\;
  \int_{\mathcal{Z}} g(z)^2\,\nu(dz)
  \;\ge\;
  \int_{B_r(z^*)} g(z)^2\,\nu(dz)
  \;\ge\;
  (M/2)^2\,\nu\big(B_r(z^*)\big).
\]
By \cref{lemma:appendix-ball-mass-lower},
$\nu(B_r(z^*))\ge \nu_{\min}c_m r^m=\nu_{\min}c_m (M/(2L_g))^m$.
Substituting gives
\[
  \varepsilon^2
  \;\ge\;
  \frac{M^2}{4}\,\nu_{\min}c_m\Big(\frac{M}{2L_g}\Big)^m
  \;=\;
  \frac{\nu_{\min}c_m}{2^{m+2}L_g^m}\,M^{m+2}.
\]
Rearranging yields \Cref{eq:appendix-L2-to-sup}.
\end{proof}

\textbf{Aggregate bounds.} Now we summarize the inequality ingredients for the main bound. Because of \cref{lemma:appendix-L2-to-sup}, there exists a universal constant $C_{\infty}$ so that, with probability of at least $1 - \delta$:
\begin{equation}
\label{eq:appendix-one-step-sup-rate}
  \|\hat Q_{k+1}-\mathcal{F}_{\tau, u}(\hat{Q}_k)\|_\infty \le C_{\infty}\,\mathrm{stat}_k^{\frac{2}{m+2}}
  \le 
  C_{\infty}\,
  \bigg(
    \frac{1+\sqrt{\log d}+\sqrt{\log(1/\delta)}}{\sqrt{n_k}}
  \bigg)^{\!\frac{2}{m+2}}
\end{equation}

Additionally, recall that $Q^{(\alpha)}$ is the optimal soft critic $Q^{(\alpha)}(x,a) = V^{(\alpha)}(x)\,+ q_{V^{(\alpha)}}(x,a)$. By \cref{thm:q_convergence}, there exist $\rho>0$ and $C > 0$ such that:
\begin{equation}
\label{eq:appendix-population-bias}
  \|Q_k-Q^{(\alpha)}\|_\infty
  \le
  \frac{e^{-\rho\tau k}}{u}\,\|Q_0-Q^{(\alpha)}\|_\infty
  + C\big(\tau/u + u\big),
\end{equation}

\subsection{Main bound}
\label{sec:appendix:main_bound}

Define the one-step \emph{sup-norm statistical deviation}
\begin{equation}
\label{eq:appendix-def-zeta-k}
  \zeta_k
  :=
  \big\|\hat Q_{k+1}-\mathcal{F}_{\tau,u}(\hat Q_k)\big\|_\infty,
\end{equation}

\begin{lemma}[Cumulative sup-norm bound]
\label{lemma:appendix-cumulative-supnorm-gap}
Suppose the update operator $\mathcal{F}_{\tau, u}$ is $(1 - \alpha)$-Lipschitz for $\alpha < 1$ where $\alpha$ depends on $\tau$ and $u$. Under Assumption ~\ref{assumption:compact-domain} and ~\ref{assumption:coverage-density-lower} (\cref{assumption:lipschitz-critic-envelope} is satisfied), for any $L\ge 1$,
\begin{equation}
\label{eq:appendix-cumulative-gap-bound}
\begin{aligned}
  \sum_{k=0}^{L-1}\|\hat Q_k-Q^{(\alpha)}\|_\infty
  \le
  \frac{1}{(1-e^{-\rho\tau})u}
    \|\hat Q_0-Q^{(\alpha)}\|_\infty
    + C L\big(\tau/u + u\big)
    + \frac{1}{\alpha} \Bigg(\sum_{k=0}^{L-1}\zeta_k \Bigg)
\end{aligned}
\end{equation}
\end{lemma}

\begin{proof} The bias (algorithmic) error $\norm{Q_k - Q^{(\alpha)}}_{\infty}$ makes up the first 2 terms in \cref{eq:appendix-cumulative-gap-bound}. This is clear from \cref{eq:appendix-population-bias}. Now we bound the statistical error with $E_k:= \norm{\hat Q_k - Q_k}_\infty$. We have:
\[
\begin{aligned}
  E_{k+1} = \norm{\hat Q_{k+1} - \mathcal{F}_{\tau,u}(\hat Q_k)}_{\infty} &\le
  \norm{\hat Q_{k+1}-\mathcal{F}_{\tau,u}(\hat Q_k)}_{\infty}
  +
  \norm{\mathcal{F}_{\tau,u}(\hat Q_k)-\mathcal{F}_{\tau,u}( Q_k)}_\infty \\
  &\le \zeta_k + (1 - \alpha) \norm{\hat Q_k - Q_k} = \zeta_k + (1-\alpha)E_k
\end{aligned}
\]
Hence, through an inductive argument, the running gap $\sum_{k = 0}^{L-1} E_k$ is bounded by $\sum_k \zeta_k (1 + (1-\alpha) + \cdots + (1-\alpha)^{L-1})$, leading to the third term of \cref{eq:appendix-cumulative-gap-bound}
\end{proof}

Note that the terms in the update operators (see \cref{eq:appendix-F-tau-u-nonrich}) can be rearranged so that it consists of 3 components with Lipschitz coefficients $\big(1-\tau\big)$, $\dfrac{\gamma\tau}{u}$ and $\big(\tau - \dfrac{\tau}{u}\big)$. In case if $u \geq 1$, this yields a $(1-\alpha)$ Lipchitz operator with $\alpha = (1-\gamma)\dfrac{\tau}{u}$

\begin{theorem}[Sample complexity and regret bounds]
\label{theorem:appendix-regret_bound}
Given the probability confidence level $\delta \in (0,1)$ and maximum number of iteration $L_{\texttt{max}}$. Suppose the update operator is $(1-\alpha)$-Lipschitz with $\alpha \geq C_0 \dfrac{\tau}{u}$. Then there exists ($q$-function discretization) timestep $u$, the Hamiltonian flow discretization step $\tau$, a sequence of numbers of samples $\set{n_k}$, and an universal constants $C_1$ and $C_2$ (independent of $L$, $L_{\texttt{max}}$, $\tau$ or $u$), such that for every iteration $L_{\texttt{max}}/2 \leq L \leq L_{\texttt{max}}$, with probability at least $1-\delta$, the running gap satisfies:
\begin{equation}
\label{eq:appendix-cum-gap-rate}
  \sum_{k=0}^{L-1}\|\hat Q_k-Q^{(\alpha)}\|_\infty
  \leq C_1 L^{3/4}
\end{equation}

Additionally, in both cases, the total sample size in the first $L$ iteration satisfies
\begin{equation}
\label{eq:appendix-total-sample-complexity}
  K = K(L)
  :=
  \sum_{k=0}^{L-1} n_k
  = C_2 L^{m+3}.
\end{equation}
Consequently, for any given fixed $L$, ignore the log-terms, we get $\mathrm{Regret}(K) = \mathcal{O}(K^{\frac{4m + 11}{4m + 12}})$ (\textbf{sub-linear regret} rate).
\end{theorem}

\begin{proof}
Choose $\tau = L_{\texttt{max}}^{-1/2},\,u = L_{\texttt{max}}^{-1/4}$. Choose the number of samples $n_k = (k + 1)^{m+2}$ for all $k \geq 0$.

Now we choose per-iteration confidence $\delta_k:=\delta/L$ for $k \in \overline{0, L-1}$. Through the statistical bound derived above, with probability $\ge 1-\delta$ (via union bound over $k$ using $\delta_k=\delta/L$), we have for all $k \in \overline{0, L-1}$:
\[
  \zeta_k
  \le
  \Big(\mathrm{stat}_k\Big)^{\frac{2}{m+2}}
  \le
  C_{\infty}\Bigg(
    \frac{1+\sqrt{\log d}+\sqrt{\log(L/\delta)}}{\sqrt{n_k}}
  \Bigg)^{\!\frac{2}{m+2}}.
\]

Now we analyze each error term within the large bracket of \cref{eq:appendix-cumulative-gap-bound}. First, with $n_k = (k+1)^{m+2}$, the right-hand side for the inequality of $\zeta_k$ can be upper-bounded by $C_3\big(\log (L/\delta)^{\frac{1}{m+2}}/(k+1) \big)$. Hence $\sum_{k=0}^{L-1}\zeta_k \leq C_4\big(\log (L/\delta)^{\frac{1}{m+2}} \times \log L\big)$, and the statistical part is then only multiply by $1/\alpha = u/\tau = C_5L^{1/4}$.

Second, due to the choice of $\tau$ and $u$, the bias part \cref{eq:appendix-cumulative-gap-bound} from  can be bounded by $C_6 \dfrac{1}{\tau u} + CL (\tau/u + u) = C_7 L^{3/4}$, which dominates the statistical part and yields the overall error rate in \cref{eq:appendix-cum-gap-rate}. The sample complexity follows from $\sum_{k=1}^L k^{m+2} \leq C_1 L^{m+3}$ for some universal constant $C_1$. 

Finally, regarding the regret bound, which can be defined by sum of gaps $Q^{(\alpha)}(Z) - \hat{Q}_k(Z)$ over starting points $Z$ of all episodes. Thus, up to iteration $L$ (for any $L \leq L_{\texttt{max}}$), the regret can be over-estimated by:
\begin{align*}
\mathrm{Regret} &\leq \sum_{\text{all samples } Z} \E\norm{\hat{Q}_k(Z) - Q^{(\alpha)}(Z)} \leq \sum_{k=0}^{L-1} n_k \norm{\hat{Q}_k - Q^{(\alpha)}}_{\infty} \\
&\leq \left(\max_{k \in \overline{0, L-1}} {n_k} \right) \sum_{k=0}^{L-1} \|\hat Q_k-Q^{(\alpha)}\|_\infty
  \leq C_1 L^{m+2} L^{3/4} = \mathcal{O}(L^{m+11/4})
\end{align*}
As $K = \mathcal{O}(L^{m+3})$, $\mathcal{O}(L^{m + 11/4}) = \mathcal{O}(K^{\frac{m + 11/4}{m + 3}}) = \mathcal{O}(K^{\frac{4m + 11}{4m + 12}})$, yielding the desired regret bound. Here all $C_i$'s are universal constants independent of $L$, $L_{\texttt{max}}$, $\tau$ or $u$.
\end{proof}

\textbf{Notes.} For random time step $U\in(0,1]$, independent of $(x,a)$, the only change is the bias term, which becomes a moment of $U$: $\E[\mathrm{bias}(U)]$. In particular: $\mathrm{bias}(U)=\E[U^{1/2}]$ for standard case, and $\mathrm{bias}(U)=\E[U^2]$ for Richardson-version. Additionally, an improved variant with sharper regret and complexity guarantees will be studied in subsequent work. The proof techniques developed here extend naturally to adaptive step sizes $(\tau_k)_{k\ge 0}$; we also plan to derive explicit rates for such schedules in the future.

\section{Further Theoretical Details For Main Paper}\label{sec:appendix:further_theoretical_details}
\subsection{Further explanations for the main paper}
\textbf{Variational representation for Hamiltonian.} For the Hamiltonian function, by a standard variational argument (see \cref{lemma:entropy-KL-identity}), one can equivalently write:
\begin{equation}
H^{(\alpha)}(V)(x)
=\sup_{\pi(\cdot\mid x)} \int_{\mathcal{A}} \Big(r(x,a) + (L^a V)(x) - \beta V(x) - \alpha\log\pi(a\mid x) \Big)\,\pi(da\mid x)
\label{eq:soft-hamiltonian-variational}
\end{equation}
As $\alpha\to 0$, $H^{(\alpha)}$ recovers $H^{(0)}$ and, for $\alpha>0$, the policy $\pi$ that maximize the above variational form is exactly the Boltzmann policy
\begin{equation}
 \pi^{(\alpha)}_V(a\mid x) \propto
    \exp\Big(
     \tfrac{1}{\alpha}\big(
        r(x,a) + (L^a V)(x) - \beta V(x)
     \big)
    \Big)
 \label{eq:boltzmann-policy}
\end{equation}

\textbf{Distinction between $u$ and $\tau$.} We emphasize that our framework separates two distinct discretizations: (i) the environment produces transitions with potentially irregular durations $u$, and (ii) the algorithm updates $V$ using a flow step size $\tau$, which controls value-iteration stability and convergence and is not tied to $u$. This separation is one reason the method remains stable under small or irregular time steps. Additionally, in this work, we make two concrete modeling choices for simplicity and scalability. First, we discretize the value flow \cref{eq:V-flow-H} using a single forward-Euler (Picard) step. More accurate ODE solvers (e.g., higher-order Runge--Kutta schemes) could be substituted to improve numerical stability. Second, we implement the decoupled iteration with a single critic by using the algebraic representation $Q = V + 1\cdot q$ in \cref{eq:Q-decomposition}. This choice is viewed as selecting a nominal unit time scale for combining a value-like term and an instantaneous rate term. More generally, one may consider $Q = V + c\,q$ for other scalings $c>0$, which may trade off learning signal, conditioning, and empirical performance. Exploring such optimal discretization and scaling strategies is the scope for future work.

\subsection{Discrete MDP with noise converge to SDE}
Here we explain and rigorously prove the statement we mentioned in \cref{sec:prelim:mdp-to-sde} regarding the convergence of discrete-MDP to stochastic differential equation (SDE) when the time step converges to $0$.

\begin{lemma}[Convergence of MDP to SDE]
Fix $T>0$ and let the time step $\Delta_n\to 0$. Define grid times $t_k=k\Delta_n$ and $\bar t:=t_k$ be the left endpoint mapping for $t\in[t_k,t_{k+1})$. Assume $f(\cdot,u)$ and $\sigma(\cdot,u)$ are globally Lipschitz and of linear growth in $x$, and let $X_t$ solve the SDE:
\[
dX_t = f(X_t,U_t)\,dt + \sigma(X_t,U_t)\,dW_t,\qquad X_0=x.
\]
Define (discretized) MDP process $X_k^{(n)}$ with discretized time step $\Delta_n$:
\[
X_{k+1}^{(n)} = X_k^{(n)} + f(X_k^{(n)},U_{t_k})\Delta_n + \sigma(X_k^{(n)},U_{t_k})\big(W_{t_{k+1}}-W_{t_k}\big),
\qquad X_0^{(n)}=x,
\]
and let $X_t^{(n)}:=X_k^{(n)}$ for $t\in[t_k,t_{k+1})$ (piecewise-constant interpolation), and similarly $U_{\bar t}:=U_{t_k}$.
Then there exists a constant $C=C(T,L,K,x)$ such that
\[
\mathbb E\Big[\sup_{0\le t\le T}\|X_t^{(n)}-X_t\|^2\Big]\le C\,\Delta_n.
\]
In particular, $\sup_{t\le T}\|X_t^{(n)}-X_t\|\to 0$ in $L^2$ (and hence in probability).
\end{lemma}

\begin{proof}
First by the proof of \cref{lemma:moment-bound}, $\mathbb E\|X_t-X_s\|^2 \le C|t-s|$. Therefore, for any $u\in[0,T]$ with $\bar u$ its left gridpoint, $|u-\bar u|\le \Delta_n$ and $\mathbb E\|X_u-X_{\bar u}\|^2 \le C\,\Delta_n$. Now, the piecewise-constant Euler process can be written as
\[
X_t^{(n)} = x + \int_0^t f(X^{(n)}_{\bar s},U_{\bar s})\,ds
           + \int_0^t \sigma(X^{(n)}_{\bar s},U_{\bar s})\,dW_s.
\]
Subtracting the SDE gives the error process $e_t:=X_t^{(n)}-X_t$:
\[
e_t
= \int_0^t \big[f(X^{(n)}_{\bar s},U_{\bar s})-f(X_s,U_s)\big]\,ds
 +\int_0^t \big[\sigma(X^{(n)}_{\bar s},U_{\bar s})-\sigma(X_s,U_s)\big]\,dW_s.
\]
Let $E_t:=\sup_{0\le r\le t}\|e_r\|$.

\textbf{Drift term.} With $g_s=f(X^{(n)}_{\bar s},U_{\bar s})-f(X_s,U_s)$, using Lipschitzness, we get:
\[
\|g_s\| \le L\|X^{(n)}_{\bar s}-X_s\|
\le L\big(\|X^{(n)}_{\bar s}-X_{\bar s}\|+\|X_{\bar s}-X_s\|\big)
\le L\big(E_s+\|X_{\bar s}-X_s\|\big).
\]
As a result,
\begin{equation}\label{eq:drift-bound}
\mathbb E\Big[\sup_{0\le r\le t}\Big\|\int_0^r g_s\,ds\Big\|^2\Big]
\le C\int_0^t \mathbb E[E_s^2]\,ds + C\int_0^t \mathbb E\|X_s-X_{\bar s}\|^2\,ds.
\end{equation}

\textbf{Diffusion term.} By the Burkholder--Davis--Gundy inequality (for $p=2$),
\[
\mathbb E\Big[\sup_{0\le r\le t}\Big\|\int_0^r h_s\,dW_s\Big\|^2\Big]
\le C_{\mathrm{BDG}}\;\mathbb E\int_0^t \|h_s\|_F^2\,ds.
\]
With $h_s=\sigma(X^{(n)}_{\bar s},U_{\bar s})-\sigma(X_s,U_s)$ and Lipschitzness,
\[
\|h_s\|_F \le L\|X^{(n)}_{\bar s}-X_s\|
\le L\big(E_s+\|X_{\bar s}-X_s\|\big),
\]
so that
\begin{equation}\label{eq:diff-bound}
\mathbb E\Big[\sup_{0\le r\le t}\Big\|\int_0^r h_s\,dW_s\Big\|^2\Big]
\le C\int_0^t \mathbb E[E_s^2]\,ds + C\int_0^t \mathbb E\|X_s-X_{\bar s}\|^2\,ds.
\end{equation}

Now combining \eqref{eq:drift-bound}--\eqref{eq:diff-bound} together with triangle inequality,
\[
\mathbb E[E_t^2]
\le C\int_0^t \mathbb E[E_s^2]\,ds + C\int_0^t \mathbb E\|X_s-X_{\bar s}\|^2\,ds.
\]
As proved earlier, $\mathbb E\|X_s-X_{\bar s}\|^2\le C\Delta_n$ for all $s\le T$, so that
\[
\mathbb E[E_t^2]\le C\int_0^t \mathbb E[E_s^2]\,ds + C\,t\,\Delta_n.
\]
Finally, Gronwall's lemma gives $\mathbb E[E_T^2]\le C\,\Delta_n$, proving the claim with the error go to $0$ as $\Delta_n \to 0$.
\end{proof}

Note that instead of $\big(W_{t_{k+1}}-W_{t_k}\big)$, we can replace it by i.i.d normal random variable $\xi_{t_k}$, and all the inequalities derived in the proof still holds in expectation as they has the same distribution due to the property of the Brownian motion.

\subsection{Exact Richardson single-critic update equation.}
\begin{lemma}[Algebraic derivation of the Richardson single-critic update]
\label{lemma:appendix-richardson-Q-algebra}
Fix $u>0$ and define $\gamma:=e^{-\beta u}$ and $\gamma_{1/2}:=e^{-\beta u/2}$. Let $V_k$ be updated by the entropy-regularized expectation step
\begin{equation}
  V_{k+1}(x)
  =
  (1-\tau)V_k(x) + \tau S_k(x),
  \qquad
  S_k(x):=\E_{a\sim\pi_k(\cdot\mid x)}\!\big[\tilde Q_k^{\mathrm{R}}(x,a)\big],
  \label{eq:appendix-V-update-richardson}
\end{equation}
where $Q_k^{\mathrm{R}}(x,a):=V_k(x)+\tilde q_{V_k}^u(x,a)$ and $\tilde Q_k^{\mathrm{R}}(x,a):=Q_k^{\mathrm{R}}(x,a)-\alpha\log\pi_k(a\mid x)$. Then the induced single-critic iteration satisfies the closed-form update:
\begin{equation}
\begin{aligned}
  Q_{k+1}^{\mathrm{R}}(x,a)
  &=
  (1-\tau)Q_k^{\mathrm{R}}(x,a)
  + \tau r(x,a)
  + \tau S_k(x) \\
  &\quad
  + \frac{\tau}{u}\Big(
      4\gamma_{1/2}\,\E_x^a\!\big[S_k(X_{u/2})\big]
      - \gamma\,\E_x^a\!\big[S_k(X_{u})\big]
      - 3 S_k(x)
  \Big).
\end{aligned}
  \label{eq:appendix-richardson-Q-update-closed-reprove}
\end{equation}
\end{lemma}

\begin{proof}
We start with definition of $q_V^{u/2}$:
\[
q_V^{u/2}(x,a)
=
\frac{\gamma_{1/2}\E_x^a[V(X_{u/2})]-V(x)}{u/2}+r(x,a)
=
\frac{2(\gamma_{1/2}\E_x^a[V(X_{u/2})]-V(x))}{u}+r(x,a).
\]
Hence
\begin{equation}
\begin{aligned}
  \tilde q_V^u(x,a)
  &=2q_V^{u/2}(x,a)-q_V^u(x,a)\\
  &=
  \frac{4\gamma_{1/2}\E_x^a[V(X_{u/2})]-\gamma\E_x^a[V(X_u)]-3V(x)}{u}
  + r(x,a).
\end{aligned}
  \label{eq:appendix-tildeq-expanded}
\end{equation}

Now apply \Cref{eq:appendix-tildeq-expanded} to $V_{k+1}$:
\begin{equation}
\begin{aligned}
  \tilde q_{V_{k+1}}^u(x,a)
  &=
  \frac{
    4\gamma_{1/2}\E_x^a[V_{k+1}(X_{u/2})]
    -\gamma \E_x^a[V_{k+1}(X_u)]
    -3V_{k+1}(x)
  }{u}
  + r(x,a).
\end{aligned}
  \label{eq:appendix-tildeqVk1-start}
\end{equation}
Substitute the value update \Cref{eq:appendix-V-update-richardson} inside each term.
Since the update is pointwise in the state, for any random state $Y$ we have
$V_{k+1}(Y)=(1-\tau)V_k(Y)+\tau S_k(Y)$, and therefore
\[
\E_x^a[V_{k+1}(X_h)]
=
(1-\tau)\E_x^a[V_k(X_h)]
+
\tau\E_x^a[S_k(X_h)],
\qquad h\in\{u/2,u\}.
\]
Also $V_{k+1}(x)=(1-\tau)V_k(x)+\tau S_k(x)$.
Plugging these into \cref{eq:appendix-tildeqVk1-start} and collecting the corresponding terms gives:
\begin{equation}
\begin{aligned}
  \tilde q_{V_{k+1}}^u(x,a)
  &=
  (1-\tau)
  \frac{
    4\gamma_{1/2}\E_x^a[V_k(X_{u/2})]
    -\gamma\E_x^a[V_k(X_u)]
    -3V_k(x)
  }{u}
  \\
  &\quad
  +
  \tau
  \frac{
    4\gamma_{1/2}\E_x^a[S_k(X_{u/2})]
    -\gamma\E_x^a[S_k(X_u)]
    -3S_k(x)
  }{u}
  + r(x,a).
\end{aligned}
  \label{eq:appendix-tildeqVk1-split}
\end{equation}
The first fraction is exactly $\tilde q_{V_k}^u(x,a)-r(x,a)$ by \cref{eq:appendix-tildeq-expanded}.
Thus
\begin{equation}
\begin{aligned}
  \tilde q_{V_{k+1}}^u(x,a)
  &=
  (1-\tau)\big(\tilde q_{V_k}^u(x,a)-r(x,a)\big)
  + r(x,a)\\
  &\quad
  + \frac{\tau}{u}\Big(
      4\gamma_{1/2}\E_x^a[S_k(X_{u/2})]
      -\gamma\E_x^a[S_k(X_u)]
      -3S_k(x)
  \Big)\\
  &=
  (1-\tau)\tilde q_{V_k}^u(x,a)
  + \tau r(x,a)
  + \frac{\tau}{u}\Big(
      4\gamma_{1/2}\E_x^a[S_k(X_{u/2})]
      -\gamma\E_x^a[S_k(X_u)]
      -3S_k(x)
  \Big).
\end{aligned}
  \label{eq:appendix-tildeqVk1-final}
\end{equation}

Finally, add the value update $V_{k+1}(x)=(1-\tau)V_k(x)+\tau S_k(x)$ to both sides, we get
\begin{align*}
  Q_{k+1}^{\mathrm{R}}(x,a)
  &=
  (1-\tau)V_k(x)+\tau S_k(x)
  + (1-\tau)\tilde q_{V_k}^u(x,a)
  + \tau r(x,a)
  + \frac{\tau}{u}\Big(\cdots\Big)\\
  &=
  (1-\tau)\big(V_k(x)+\tilde q_{V_k}^u(x,a)\big)
  + \tau r(x,a)
  + \tau S_k(x)
  + \frac{\tau}{u}\Big(\cdots\Big)\\
  &=
  (1-\tau)Q_k^{\mathrm{R}}(x,a)
  + \tau r(x,a)
  + \tau S_k(x)
  + \frac{\tau}{u}\Big(
      4\gamma_{1/2}\E_x^a[S_k(X_{u/2})]
      -\gamma\E_x^a[S_k(X_u)]
      -3S_k(x)
  \Big),
\end{align*}
which is exactly \Cref{eq:appendix-richardson-Q-update-closed-reprove}.
\end{proof}

\subsection{Decoupled continuous-time policy evaluation}
For a fixed policy $\pi$, standard stochastic control arguments imply that the value function $V^{(\alpha)}(\cdot;\pi)$ in \eqref{eq:value-relaxed} satisfies the linear elliptic PDE:
\begin{equation}
\beta V^{(\alpha)}(x;\pi) - (L^\pi V^{(\alpha)}(\cdot;\pi))(x) - \tilde r^{(\alpha)}(x,\pi) = 0 \label{eq:policy-eval-pde}
\end{equation}

Fix a Markov policy $\pi(a\mid x)$ and temperature $\alpha\ge 0$.
For a value function $V\in\mathcal{D}\subset C_b^2(\mathbb{R}^d)$, define the \emph{instantaneous $q$-rate} for $(\pi,\alpha)$ as
\begin{equation}
  q^{(\pi,\alpha)}_V(x,a) := r(x,a) - \alpha \log \pi(a\mid x) + (L^a V)(x) - \beta V(x).
  \label{eq:q-pi-alpha}
\end{equation}

This corresponds to the infinitesimal temporal-difference
\begin{equation}
  q^{(\pi,\alpha)}_V(x,a)
  = \lim_{\Delta t\to 0} \frac{e^{-\beta\Delta t}\mathbb{E}[V(X_{\Delta t})] - V(x)}{\Delta t}  + r(x,a) - \alpha\log\pi(a\mid x)
\end{equation}
This policy-dependent $q$-function $q^{(\pi,\alpha)}_V$ can be estimated from short rollouts by:
\begin{equation}
  q_V^{u, (\pi,\alpha)}(x,a) := \frac{e^{-\beta u}V(X_{t+u})
        - V(x)}{u} + r(x,a) - \alpha\log\pi(a\mid x)
  \label{eq:q-estimator}
\end{equation}

Given $q^{(\pi,\alpha)}_V$, the policy-evaluation PDE satisfies
\[
  (L^\pi V)(x) + \tilde r^{(\alpha)}(x,\pi) - \beta V(x) = 0
\]
And can be equivalently written as:
\begin{equation}
  \int_{\mathcal{A}} q^{(\pi,\alpha)}_V(x,a)\,\pi(da\mid x) \;=\; 0,
  \qquad \forall x.
  \label{eq:q-eval-pde}
\end{equation}
Following a similar idea as in \cref{sec:formulation}, we obtain the following flow update for $V$ under a fixed $\pi$:
\begin{equation}
  \frac{d}{d\tau} V_\tau(x)
  \;=\;
  \int_{\mathcal{A}} q^{(\pi,\alpha)}_{V_\tau}(x,a)\,\pi(da\mid x).
  \label{eq:q-flow-eval}
\end{equation}
A single explicit Euler step of~\eqref{eq:q-flow-eval} with step size $\tau>0$ is
\begin{equation}
  T_\tau^{(\pi,\alpha)}(V)(x)
  := V(x)
     + \tau \int_{\mathcal{A}} q^{(\pi,\alpha)}_V(x,a)\,\pi(da\mid x),
  \label{eq:Ttau-pi}
\end{equation}

This yields an analogous Picard–Hamiltonian iteration for policy evaluation. The corresponding theoretical guarantees follow by applying the same argument as in the original setting, with only notational changes.
\section{Further experiment details}\label{sec:appendix:further_experiment_details}

\subsection{Visualization on the trading task}
We provide additional (qualitative) visualizations for the trading task in \cref{fig:irregular_trading}. These plots complement the quantitative metrics by illustrating representative trajectories under the same evaluation market paths.

\begin{figure}[ht]
    \centering

    \begin{subfigure}{\textwidth}
        \centering
        \includegraphics[width=0.6\textwidth]{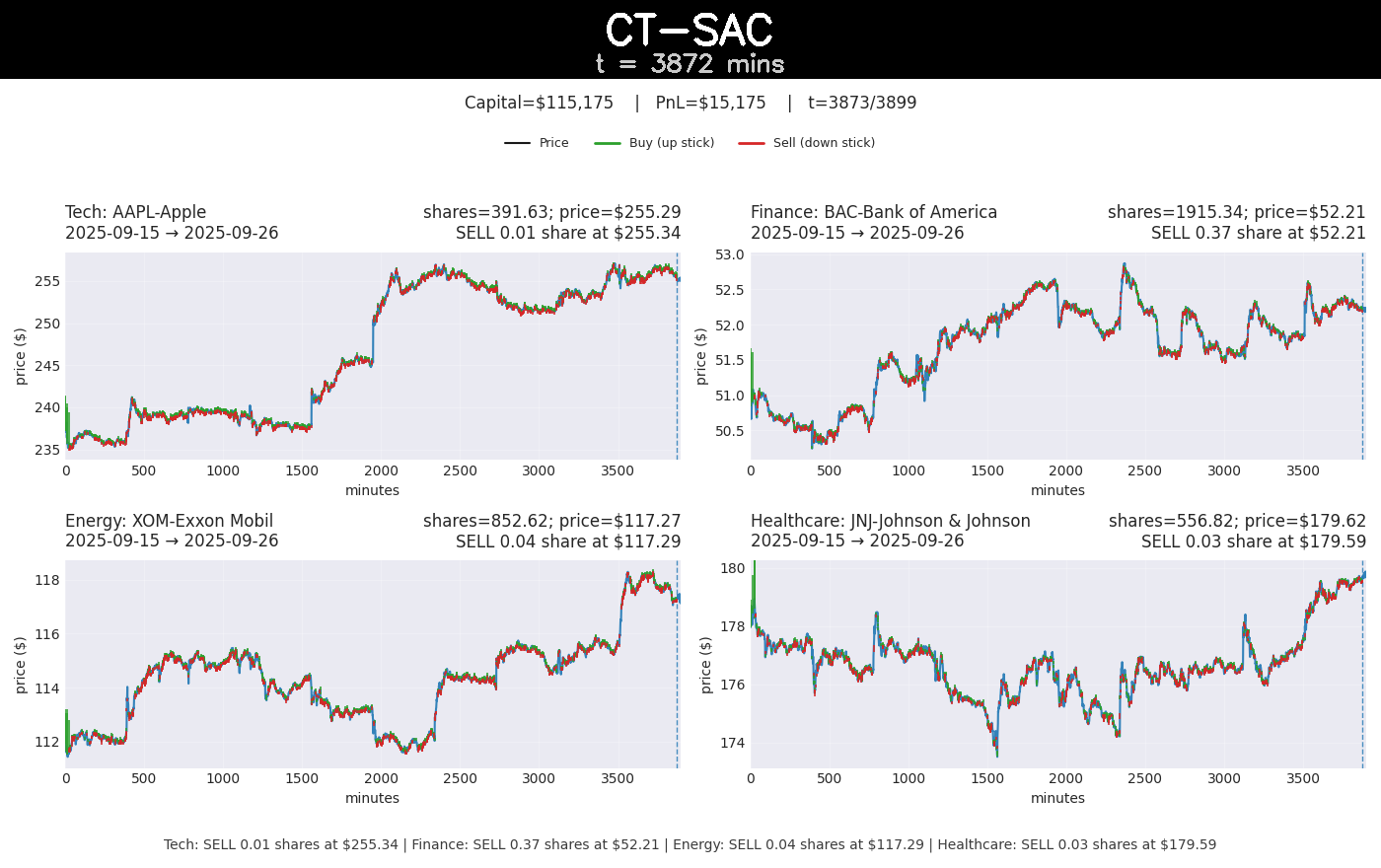}
        \caption{CT-SAC on trading task}
        \label{fig:trading_ct_sac}
    \end{subfigure}

    \vspace{6pt}

    \begin{subfigure}{\textwidth}
        \centering
        \includegraphics[width=0.6\textwidth]{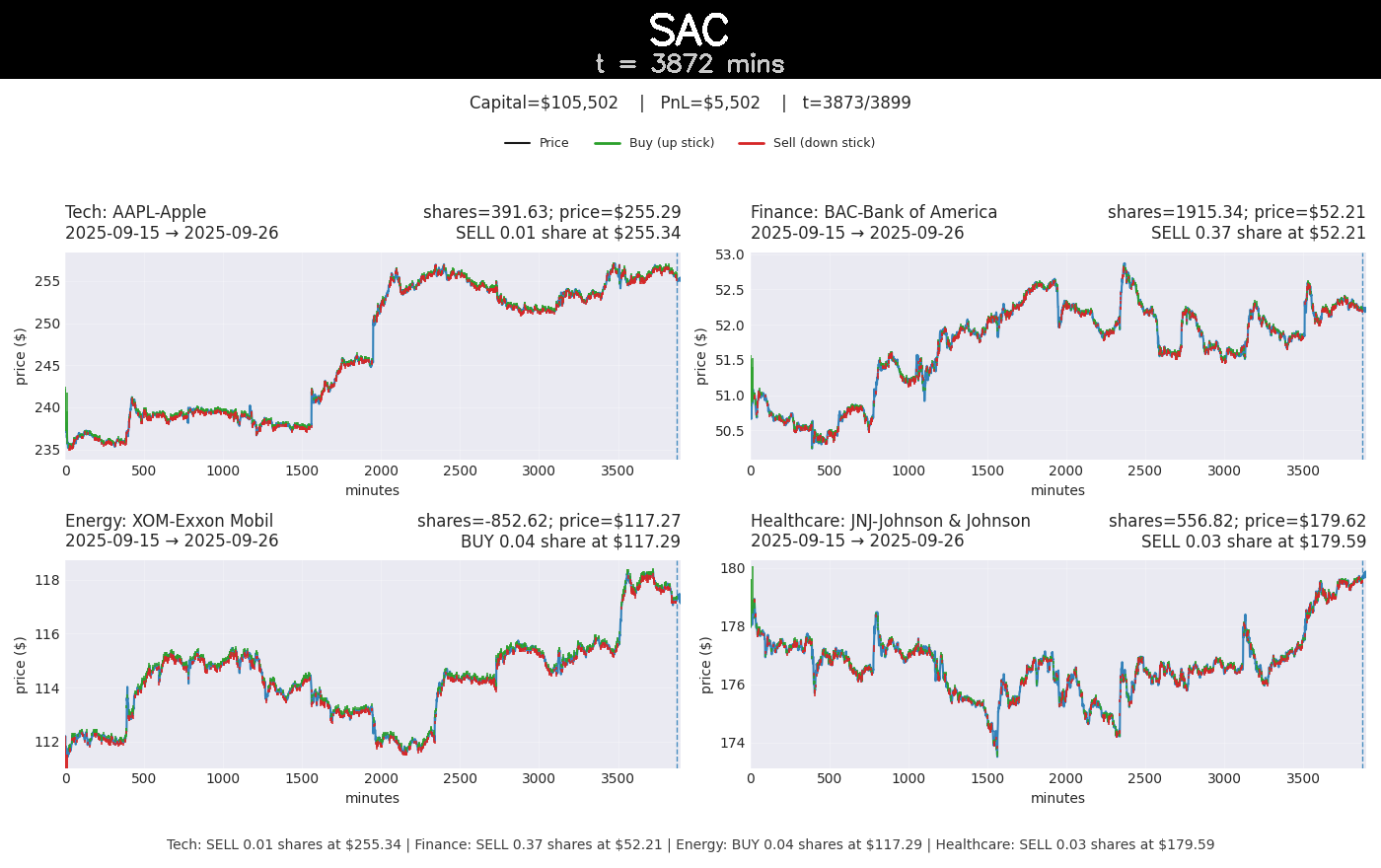}
        \caption{SAC on trading task}
        \label{fig:trading_sac}
    \end{subfigure}

    \vspace{6pt}

    \begin{subfigure}{\textwidth}
        \centering
        \includegraphics[width=0.6\textwidth]{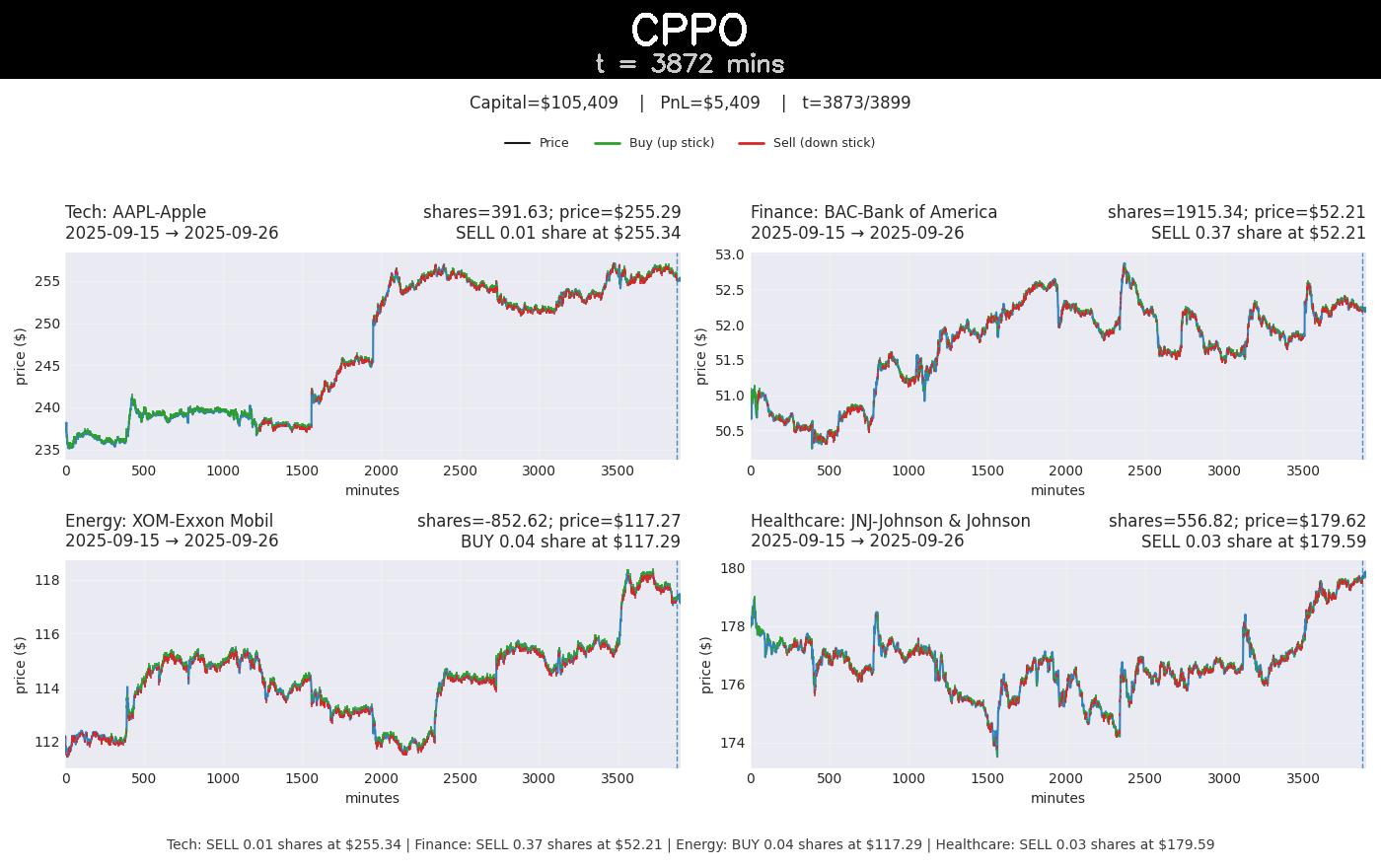}
        \caption{PPO on trading task}
        \label{fig:trading_cppo}
    \end{subfigure}

    \caption{CT-SAC vs.\ two competitive baselines CT-SAC and CPPO on trading task with two-week trading episode under irregular time steps. During these 2-weeks, CT-SAC made a profit of \$15,000 while other two methods only make about \$5,500.}
    \label{fig:irregular_trading}
\end{figure}

\subsection{Visualization on control tasks}
We provide additional qualitative results on the control tasks under two evaluation settings.

\textbf{Irregular-time evaluation.}
We first evaluate on a held-out irregular-time setting and compare CT-SAC against (i) continuous-time baselines and (ii) discrete-time baselines. See \cref{fig:irregular_ctsac_vs_continuous_last,fig:irregular_ctsac_vs_discrete_last} for the corresponding images.

\begin{figure*}[!t]
    \centering

    \begin{subfigure}{\textwidth}
        \centering
        \includegraphics[width=0.85\textwidth]{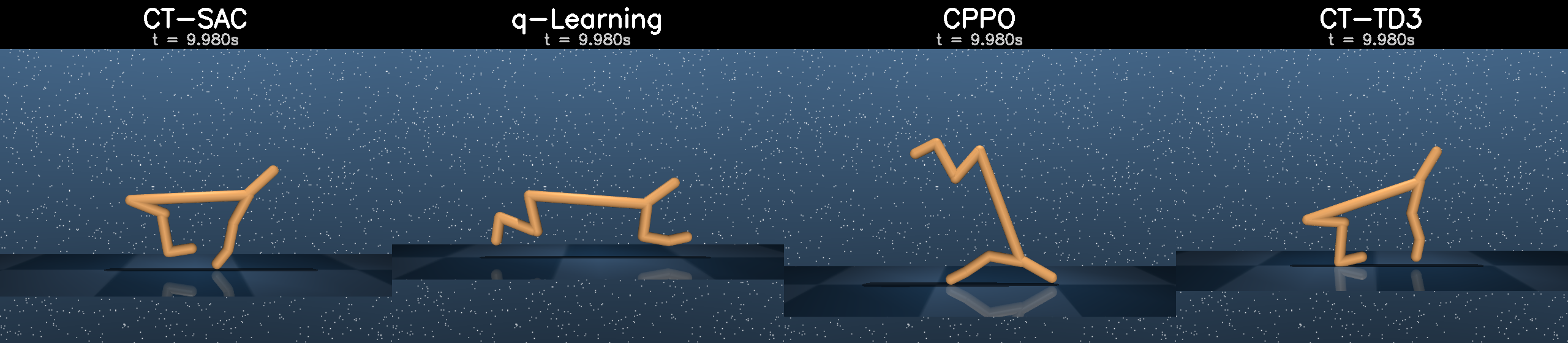}
        \caption{Cheetah}
        \label{fig:irregular_cont_cheetah}
    \end{subfigure}

    \vspace{6pt}

    \begin{subfigure}{\textwidth}
        \centering
        \includegraphics[width=0.85\textwidth]{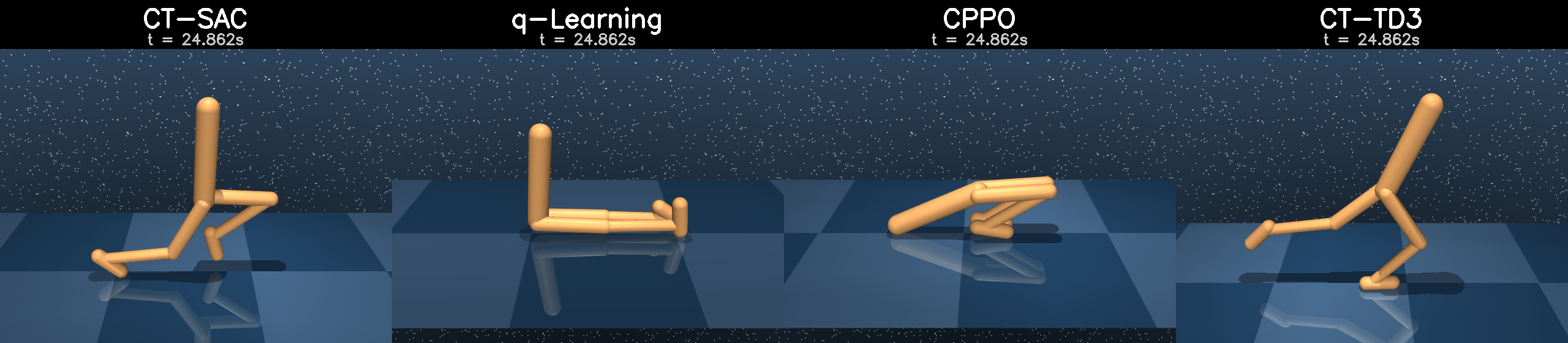}
        \caption{Walker}
        \label{fig:irregular_cont_walker}
    \end{subfigure}

    \vspace{6pt}

    \begin{subfigure}{\textwidth}
        \centering
        \includegraphics[width=0.85\textwidth]{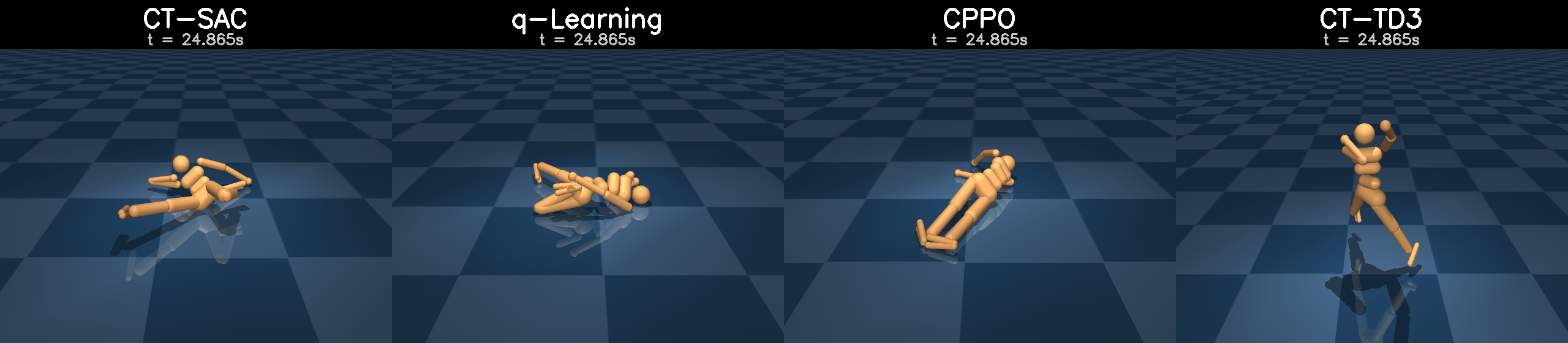}
        \caption{Humanoid}
        \label{fig:irregular_cont_humanoid}
    \end{subfigure}

    \vspace{6pt}

    \begin{subfigure}{\textwidth}
        \centering
        \includegraphics[width=0.85\textwidth]{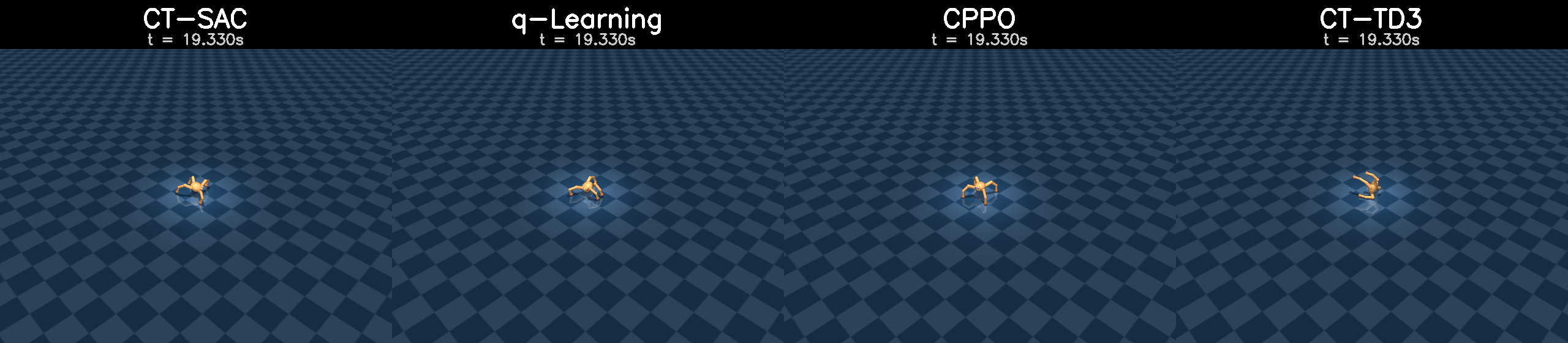}
        \caption{Quadruped}
        \label{fig:irregular_cont_quadruped}
    \end{subfigure}

    \caption{CT-SAC vs.\ continuous-time baselines under irregular time steps.}
    \label{fig:irregular_ctsac_vs_continuous_last}
\end{figure*}

\begin{figure*}[!t]
    \centering

    \begin{subfigure}{\textwidth}
        \centering
        \includegraphics[width=0.85\textwidth]{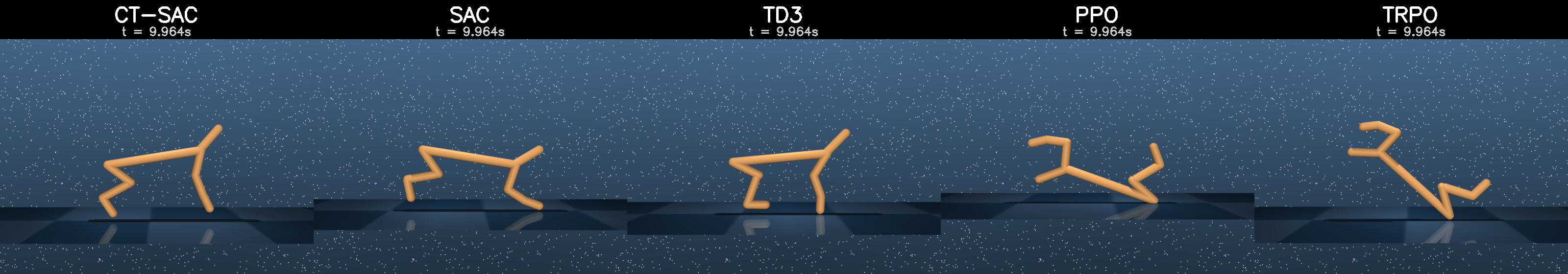}
        \caption{Cheetah}
        \label{fig:irregular_discrete_cheetah}
    \end{subfigure}

    \vspace{6pt}

    \begin{subfigure}{\textwidth}
        \centering
        \includegraphics[width=0.85\textwidth]{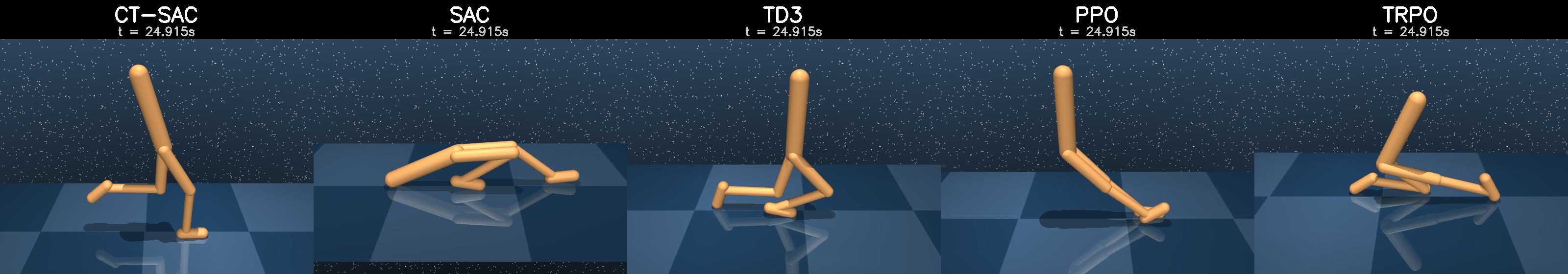}
        \caption{Walker}
        \label{fig:irregular_discrete_walker}
    \end{subfigure}

    \vspace{6pt}

    \begin{subfigure}{\textwidth}
        \centering
        \includegraphics[width=0.85\textwidth]{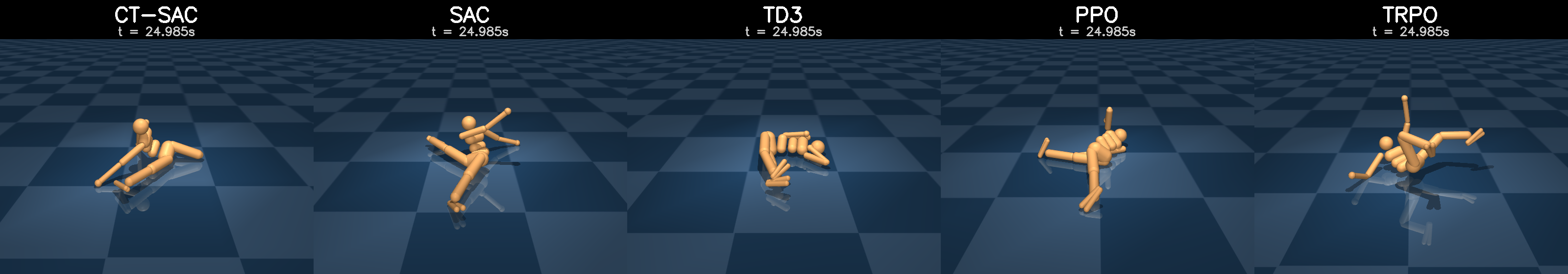}
        \caption{Humanoid}
        \label{fig:irregular_discrete_humanoid}
    \end{subfigure}

    \vspace{6pt}

    \begin{subfigure}{\textwidth}
        \centering
        \includegraphics[width=0.85\textwidth]{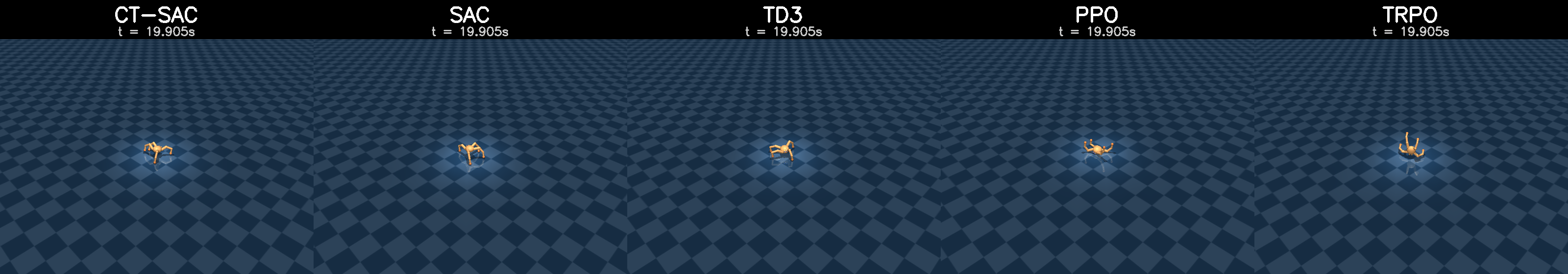}
        \caption{Quadruped}
        \label{fig:irregular_discrete_quadruped}
    \end{subfigure}

    \caption{CT-SAC vs.\ discrete-time baselines under irregular time steps.}
    \label{fig:irregular_ctsac_vs_discrete_last}
\end{figure*}

\textbf{Regular-time evaluation.}
We also evaluate the same policies under the standard regular-time setting (which differs from the irregular-time training distribution) and again compare against continuous-time and discrete-time baselines. See \cref{fig:regular_ctsac_vs_continuous_last,fig:regular_ctsac_vs_discrete_last} for qualitative results.

\begin{figure*}[!t]
    \centering

    \begin{subfigure}{\textwidth}
        \centering
        \includegraphics[width=0.85\textwidth]{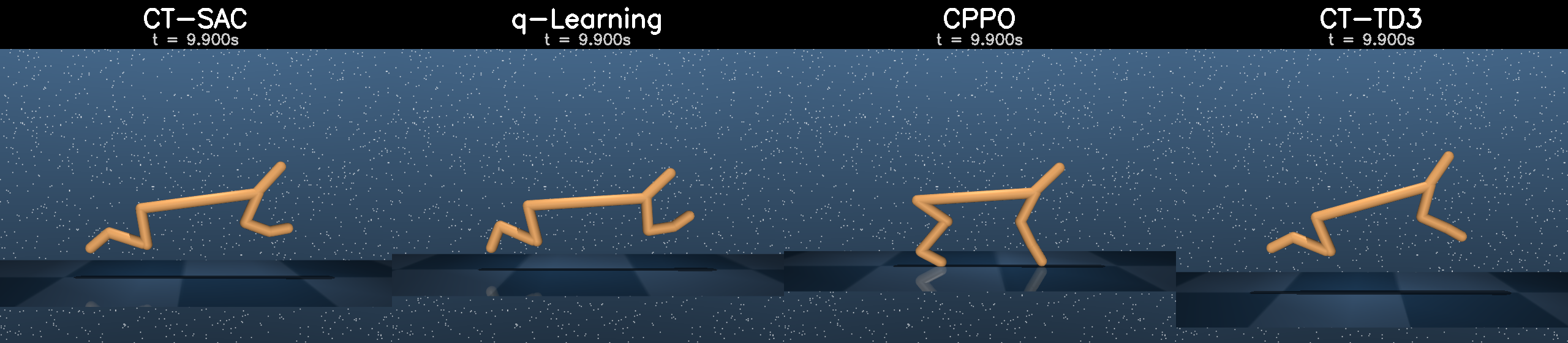}
        \caption{Cheetah}
        \label{fig:regular_cont_cheetah}
    \end{subfigure}

    \vspace{6pt}

    \begin{subfigure}{\textwidth}
        \centering
        \includegraphics[width=0.85\textwidth]{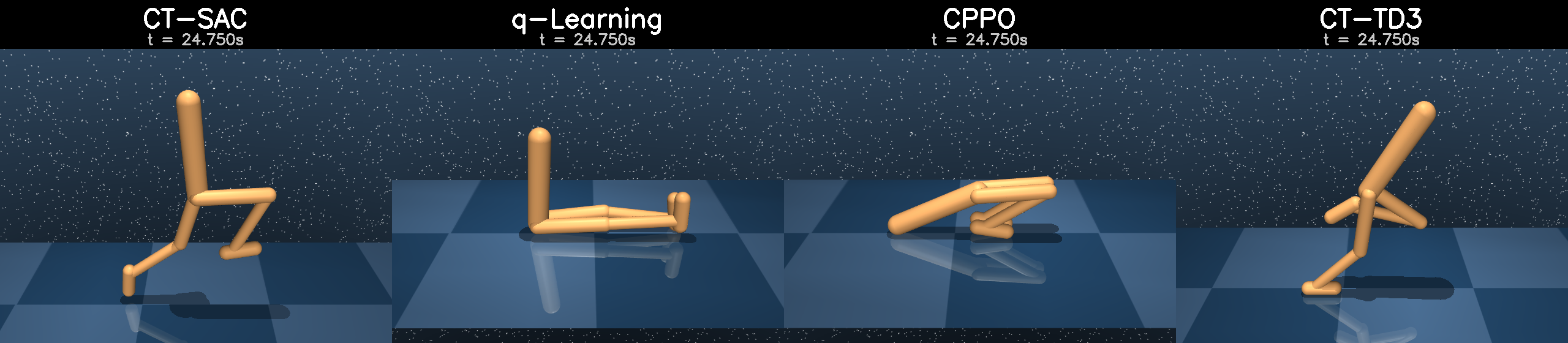}
        \caption{Walker}
        \label{fig:regular_cont_walker}
    \end{subfigure}

    \vspace{6pt}

    \begin{subfigure}{\textwidth}
        \centering
        \includegraphics[width=0.85\textwidth]{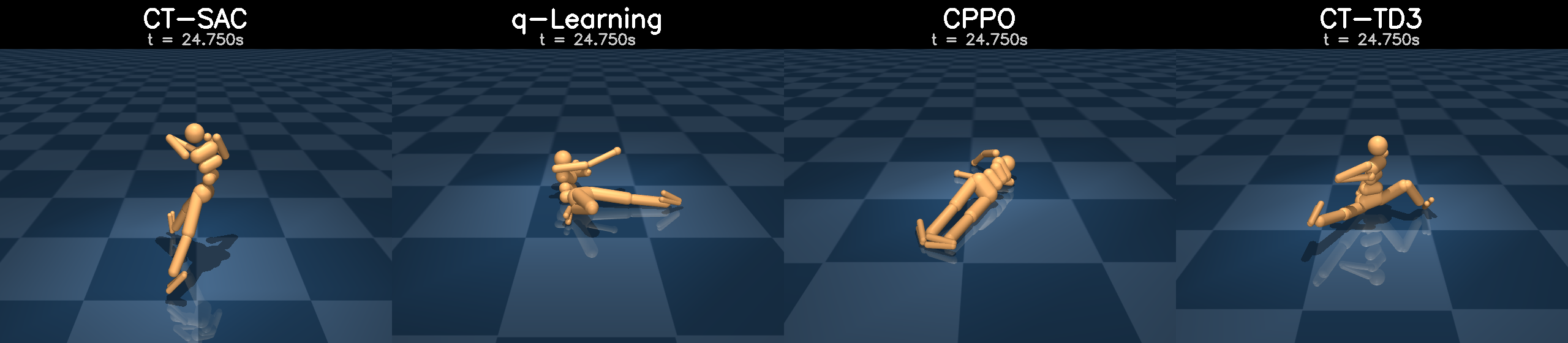}
        \caption{Humanoid}
        \label{fig:regular_cont_humanoid}
    \end{subfigure}

    \vspace{6pt}

    \begin{subfigure}{\textwidth}
        \centering
        \includegraphics[width=0.85\textwidth]{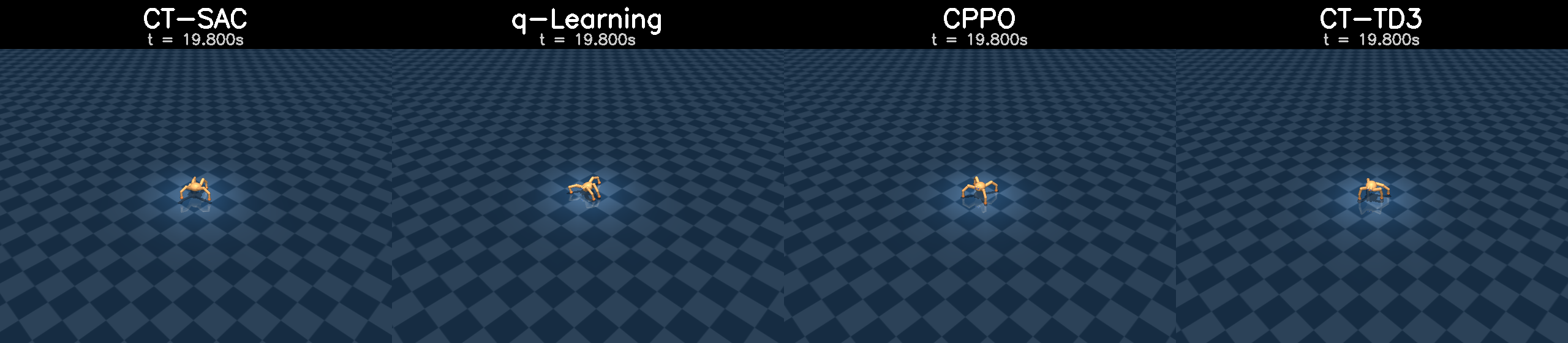}
        \caption{Quadruped}
        \label{fig:regular_cont_quadruped}
    \end{subfigure}

    \caption{CT-SAC vs.\ continuous-time baselines under \textbf{regular} time steps. Note that the model is trained under irregular time settings.}
    \label{fig:regular_ctsac_vs_continuous_last}
\end{figure*}

\begin{figure*}[!t]
    \centering

    \begin{subfigure}{\textwidth}
        \centering
        \includegraphics[width=0.85\textwidth]{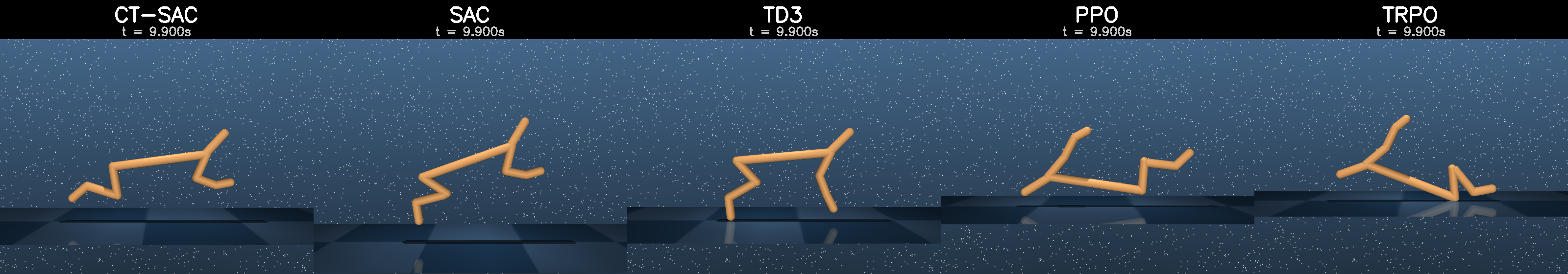}
        \caption{Cheetah}
        \label{fig:regular_discrete_cheetah}
    \end{subfigure}

    \vspace{6pt}

    \begin{subfigure}{\textwidth}
        \centering
        \includegraphics[width=0.85\textwidth]{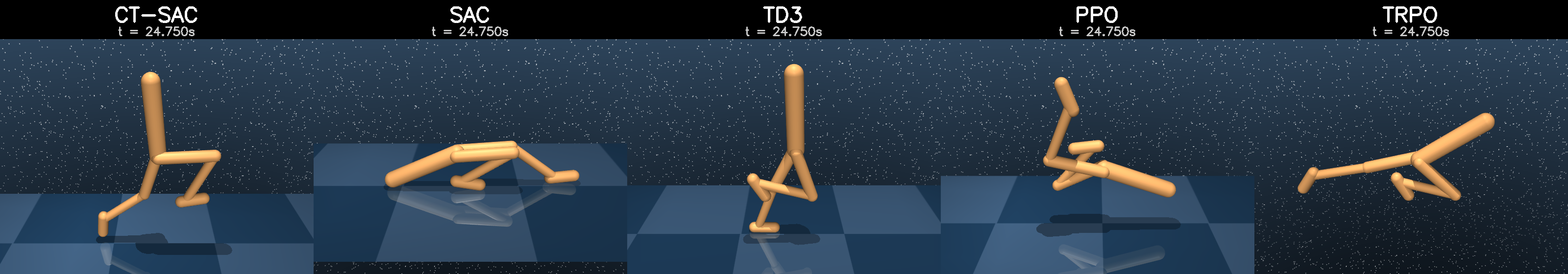}
        \caption{Walker}
        \label{fig:regular_discrete_walker}
    \end{subfigure}

    \vspace{6pt}

    \begin{subfigure}{\textwidth}
        \centering
        \includegraphics[width=0.85\textwidth]{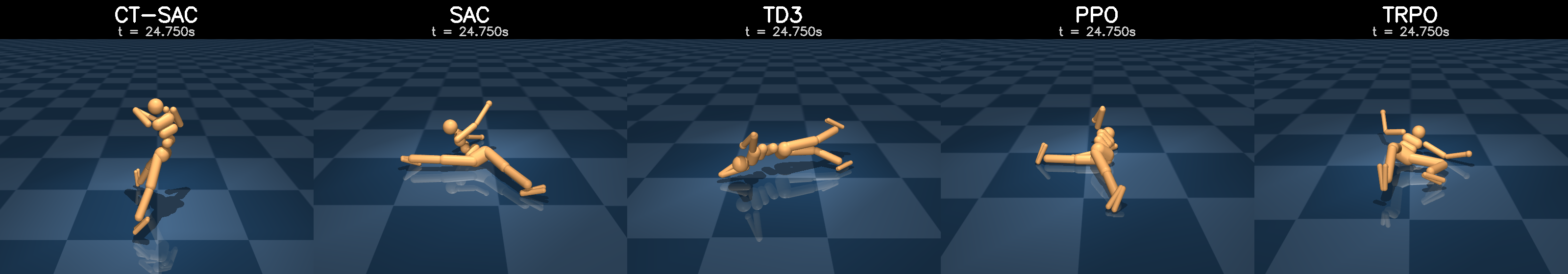}
        \caption{Humanoid}
        \label{fig:regular_discrete_humanoid}
    \end{subfigure}

    \vspace{6pt}

    \begin{subfigure}{\textwidth}
        \centering
        \includegraphics[width=0.85\textwidth]{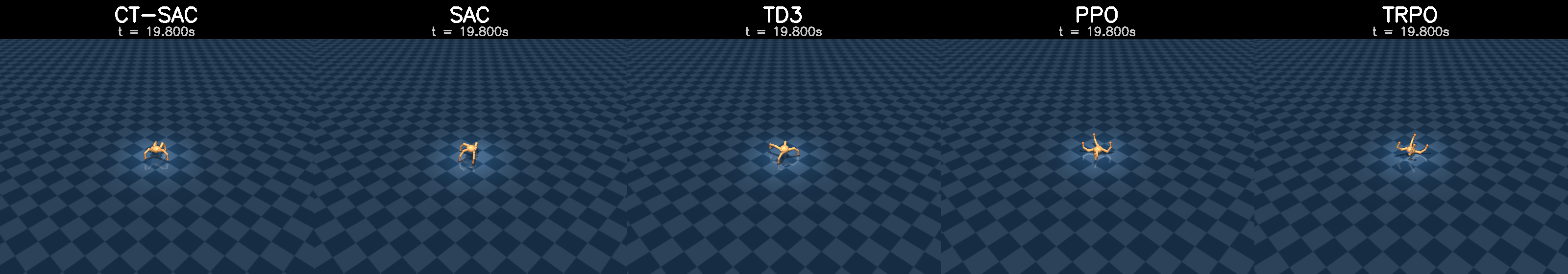}
        \caption{Quadruped}
        \label{fig:regular_discrete_quadruped}
    \end{subfigure}

    \caption{CT-SAC vs.\ discrete-time baselines under \textbf{regular} time steps. Note that the model is trained under irregular time settings.}
    \label{fig:regular_ctsac_vs_discrete_last}
\end{figure*}

\subsection{Compute resources}
\label{sec:appendix:compute_resources}

All experiments were executed on High-Performance Computing infrastructure. We used a \emph{mixture} of CPU-only and GPU-accelerated nodes with different hardware characteristics in order to support both large-scale sweeps and computationally intensive deep RL training. Each run used a single node. We report details in \cref{table:compute_resources}.

\begin{table}[t]
\centering
\caption{Compute resources used in our experiments. ``Cores'' and ``Memory'' are per node.}
\label{table:compute_resources}
\small
\setlength{\tabcolsep}{4pt}
\renewcommand{\arraystretch}{1.12}
\begin{tabular}{l l l l}
\toprule
\textbf{Node class} & \textbf{CPU vendor (cores)} & \textbf{GPU vendor (count $\times$ memory)} & \textbf{Host memory} \\
\midrule
CPU-only (high core) & NVIDIA (144) & -- & 237 GB \\
GPU node (single GPU) & NVIDIA (72)  & NVIDIA (1$\times$96 GB) & 116 GB \\
\addlinespace
CPU-only (general) & AMD (128) & -- & 256 GB \\
GPU node (multi-GPU) & AMD (128) & NVIDIA (3$\times$40 GB) & 256 GB \\
\addlinespace
CPU-only (baseline) & Intel (48) & -- & 192 GB \\
CPU-only (HBM) & Intel (112) & -- & 128 GB (HBM) \\
CPU-only (large memory) & Intel (80) & -- & 4 TB (NVDIMM) \\
\bottomrule
\end{tabular}
\end{table}

\subsection{Ablation-study hyperparameters}
For completeness, we list the hyperparameter configurations used in the ablation study.
Table~\ref{table:hp_dict_ctsac_sac}, \ref{table:hp_dict_cttd3_td3}, \ref{table:hp_dict_cppo_ppo}, and \ref{table:hp_dict_trpo_q} provide the configuration dictionary for each algorithm group, with \textbf{top/second/third} referring to the ranked settings used in the ablation tables. Due to space limit, we only list differing hyperparameters for these configurations. Please refer to the codebase to get a complete set of hyperparameters. 

% -------------------------
% CT-SAC + SAC
% -------------------------
\begin{table*}[t]
\centering
\caption{Hyperparameter dictionary (CT-SAC vs SAC). For each environment, \textbf{top/second/third} correspond to the ranked configurations used in Table~\ref{table:ablation_table}.}
\label{table:hp_dict_ctsac_sac}
\scriptsize
\setlength{\tabcolsep}{4pt}
\renewcommand{\arraystretch}{1.12}
\resizebox{\textwidth}{!}{%
\begin{tabular}{l p{0.15\textwidth} p{0.15\textwidth} p{0.15\textwidth} p{0.15\textwidth} p{0.15\textwidth} p{0.15\textwidth}}
\toprule
\textbf{Task} & \multicolumn{3}{c}{\textbf{CT-SAC}} & \multicolumn{3}{c}{\textbf{SAC}} \\
\cmidrule(lr){2-4}\cmidrule(lr){5-7}
& \textbf{top} & \textbf{second} & \textbf{third} & \textbf{top} & \textbf{second} & \textbf{third} \\
\midrule
\textbf{Cheetah} &
\texttt{lr=0.00073, tau=0.04} &
\texttt{lr=0.00073, tau=0.02} &
\texttt{lr=0.00146, tau=0.04} &
\texttt{lr=0.00073, train\_freq=4, gradient\_steps=4} &
\texttt{lr=0.00073, train\_freq=1, gradient\_steps=1} &
\texttt{lr=0.00146, train\_freq=4, gradient\_steps=4} \\
\addlinespace
\textbf{Walker} &
\texttt{lr=0.00073, train\_freq=1, gradient\_steps=1, tau=0.02} &
\texttt{lr=0.00073, train\_freq=4, gradient\_steps=4, tau=0.02} &
\texttt{lr=0.00146, train\_freq=4, gradient\_steps=4, tau=0.01} &
\texttt{lr=0.00146, train\_freq=1, gradient\_steps=1} &
\texttt{lr=0.00073, train\_freq=4, gradient\_steps=4} &
\texttt{lr=0.00292, train\_freq=1, gradient\_steps=1} \\
\addlinespace
\textbf{Humanoid} &
\texttt{train\_freq=1, gradient\_steps=1, tau=0.01} &
\texttt{train\_freq=4, gradient\_steps=4, tau=0.005} &
\texttt{train\_freq=1, gradient\_steps=1, tau=0.005} &
\texttt{log\_std\_init=-0.5, lr=0.0003, tau=0.0025} &
\texttt{log\_std\_init=0.5, lr=0.0003, tau=0.005} &
\texttt{log\_std\_init=-0.5, lr=0.0012, tau=0.005} \\
\addlinespace
\textbf{Quadruped} &
\texttt{train\_freq=4, gradient\_steps=4, tau=0.02} &
\texttt{train\_freq=1, gradient\_steps=1, tau=0.04} &
\texttt{train\_freq=1, gradient\_steps=1, tau=0.02} &
\texttt{lr=0.000182} &
\texttt{lr=0.00292} &
\texttt{lr=0.00146} \\
\addlinespace
\textbf{Trading} &
\texttt{lr=0.001, train\_freq=4, gradient\_steps=4, tau=0.02} &
\texttt{lr=0.0005, train\_freq=1, gradient\_steps=1, tau=0.02} &
\texttt{lr=0.001, train\_freq=1, gradient\_steps=1, tau=0.005} &
\texttt{log\_std\_init=-3.0, lr=0.002, train\_freq=1, gradient\_steps=1} &
\texttt{log\_std\_init=-1.0, lr=0.002, train\_freq=4, gradient\_steps=4} &
\texttt{log\_std\_init=-3.0, lr=0.001, train\_freq=1, gradient\_steps=1} \\
\bottomrule
\end{tabular}}

\vspace{8pt}
\footnotesize
\textit{Abbreviation:} lr=learning\_rate; tau=soft-update rate; train\_freq=update frequency; gradient\_steps=number of gradient steps per update; log\_std\_init=initial policy log-std.
\end{table*}

% -------------------------
% CT-TD3 + TD3
% -------------------------
\begin{table*}[t]
\centering
\caption{Hyperparameter dictionary (CT-TD3 vs TD3). For each environment, \textbf{top/second/third} correspond to the ranked configurations used in Table~\ref{table:ablation_table}.}
\label{table:hp_dict_cttd3_td3}
\scriptsize
\setlength{\tabcolsep}{4pt}
\renewcommand{\arraystretch}{1.12}
\resizebox{\textwidth}{!}{%
\begin{tabular}{l p{0.15\textwidth} p{0.15\textwidth} p{0.15\textwidth} p{0.15\textwidth} p{0.15\textwidth} p{0.15\textwidth}}
\toprule
\textbf{Task} & \multicolumn{3}{c}{\textbf{CT-TD3}} & \multicolumn{3}{c}{\textbf{TD3}} \\
\cmidrule(lr){2-4}\cmidrule(lr){5-7}
& \textbf{top} & \textbf{second} & \textbf{third} & \textbf{top} & \textbf{second} & \textbf{third} \\
\midrule
\textbf{Cheetah} &
\texttt{noise\_sigma=0.1, policy\_noise=0.4, noise\_clip=1.0} &
\texttt{noise\_sigma=0.2, policy\_noise=0.2, noise\_clip=0.5} &
\texttt{noise\_sigma=0.1, policy\_noise=0.2, noise\_clip=0.3} &
\texttt{lr=0.00073, policy\_noise=0.3, noise\_clip=0.6} &
\texttt{lr=0.00073, policy\_noise=0.1, noise\_clip=0.2} &
\texttt{lr=0.00146, policy\_noise=0.1, noise\_clip=0.2} \\
\addlinespace
\textbf{Walker} &
\texttt{lr=0.000365, policy\_noise=0.2, noise\_clip=0.5} &
\texttt{lr=0.00073, policy\_noise=0.2, noise\_clip=0.5} &
\texttt{lr=0.000365, policy\_noise=0.1, noise\_clip=0.2} &
\texttt{lr=0.001, policy\_noise=0.1, noise\_clip=0.2} &
\texttt{lr=0.0006, policy\_noise=0.1, noise\_clip=0.2} &
\texttt{lr=0.0003, policy\_noise=0.1, noise\_clip=0.2} \\
\addlinespace
\textbf{Humanoid} &
\texttt{policy\_noise=0.2, noise\_clip=1.0} &
\texttt{policy\_noise=0.2, noise\_clip=0.5} &
\texttt{policy\_noise=0.1, noise\_clip=0.2} &
\texttt{lr=0.0003, policy\_noise=0.1, noise\_clip=0.2} &
\texttt{lr=0.0003, policy\_noise=0.2, noise\_clip=0.5} &
\texttt{lr=0.0006, policy\_noise=0.2, noise\_clip=0.5} \\
\addlinespace
\textbf{Quadruped} &
\texttt{noise\_sigma=0.1, policy\_noise=0.4, noise\_clip=1.0} &
\texttt{noise\_sigma=0.1, policy\_noise=0.2, noise\_clip=0.5} &
\texttt{noise\_sigma=0.2, policy\_noise=0.2, noise\_clip=0.5} &
\texttt{lr=0.00073, policy\_noise=0.4, noise\_clip=1.0} &
\texttt{lr=0.0003, policy\_noise=0.2, noise\_clip=0.5} &
\texttt{lr=0.00073, policy\_noise=0.2, noise\_clip=0.5} \\
\addlinespace
\textbf{Trading} &
\texttt{noise\_sigma=0.1, policy\_noise=0.4, noise\_clip=1.0} &
\texttt{noise\_sigma=0.1, policy\_noise=0.1, noise\_clip=0.2} &
\texttt{noise\_sigma=0.2, policy\_noise=0.2, noise\_clip=0.5} &
\texttt{lr=0.001, policy\_noise=0.2, noise\_clip=0.3} &
\texttt{lr=0.002, policy\_noise=0.2, noise\_clip=0.5} &
\texttt{lr=0.001, policy\_noise=0.3, noise\_clip=0.6} \\
\bottomrule
\end{tabular}}

\vspace{8pt}
\footnotesize
\textit{Abbreviation:} lr=learning\_rate; policy\_noise=noise std for target action; noise\_clip=clip range for that noise; noise\_sigma=exploration/target noise scale.
\end{table*}

% -------------------------
% CPPO + PPO
% -------------------------
\begin{table*}[t]
\centering
\caption{Hyperparameter dictionary (CPPO vs PPO). For each environment, \textbf{top/second/third} correspond to the ranked configurations used in Table~\ref{table:ablation_table}.}
\label{table:hp_dict_cppo_ppo}
\scriptsize
\setlength{\tabcolsep}{4pt}
\renewcommand{\arraystretch}{1.12}
\resizebox{\textwidth}{!}{%
\begin{tabular}{l p{0.15\textwidth} p{0.15\textwidth} p{0.15\textwidth} p{0.15\textwidth} p{0.15\textwidth} p{0.15\textwidth}}
\toprule
\textbf{Task} & \multicolumn{3}{c}{\textbf{CPPO}} & \multicolumn{3}{c}{\textbf{PPO}} \\
\cmidrule(lr){2-4}\cmidrule(lr){5-7}
& \textbf{top} & \textbf{second} & \textbf{third} & \textbf{top} & \textbf{second} & \textbf{third} \\
\midrule
\textbf{Cheetah} &
\texttt{lr=0.00073, alpha=0.004, clip\_ratio=False, clip\_range=0.2, kl\_coef\_init=1.0} &
\texttt{lr=0.00073, alpha=0.0004, clip\_ratio=True, clip\_range=0.1, kl\_coef\_init=0.3} &
\texttt{lr=0.00146, alpha=0.0004, clip\_ratio=False, clip\_range=0.2, kl\_coef\_init=1.0} &
\texttt{lr=0.000365, clip\_range=0.2, vf\_coef=0.5, gae\_lambda=0.95} &
\texttt{lr=0.000365, clip\_range=0.1, vf\_coef=0.5, gae\_lambda=0.95} &
\texttt{lr=0.000365, clip\_range=0.2, vf\_coef=1.0, gae\_lambda=0.8} \\
\addlinespace
\textbf{Walker} &
\texttt{lr=0.00073, alpha=0.004, clip\_ratio=False, clip\_range=0.2, kl\_coef\_init=1.0} &
\texttt{lr=0.000365, alpha=0.004, clip\_ratio=False, clip\_range=0.2, kl\_coef\_init=1.0} &
\texttt{lr=0.00146, alpha=0.004, clip\_ratio=False, clip\_range=0.2, kl\_coef\_init=1.0} &
\texttt{lr=0.000365, clip\_range=0.2, vf\_coef=1.0, gae\_lambda=0.95} &
\texttt{lr=0.000365, clip\_range=0.2, vf\_coef=0.5, gae\_lambda=0.95} &
\texttt{lr=0.000365, clip\_range=0.2, vf\_coef=0.5, gae\_lambda=0.8} \\
\addlinespace
\textbf{Humanoid} &
\texttt{log\_std\_init=-0.5, alpha=0.004} &
\texttt{log\_std\_init=-1.0, alpha=0.004} &
\texttt{log\_std\_init=-1.0, alpha=0.0004} &
\texttt{lr=0.0003, clip\_range=0.2, vf\_coef=0.5} &
\texttt{lr=0.0003, clip\_range=0.3, vf\_coef=0.5} &
\texttt{lr=0.0006, clip\_range=0.2, vf\_coef=0.5} \\
\addlinespace
\textbf{Quadruped} &
\texttt{log\_std\_init=-3.0, clip\_ratio=False, clip\_range=0.3} &
\texttt{log\_std\_init=-3.0, clip\_ratio=False, clip\_range=0.2} &
\texttt{log\_std\_init=-1.0, clip\_ratio=False, clip\_range=0.2} &
\texttt{lr=0.000365, clip\_range=0.2, vf\_coef=0.5, gae\_lambda=0.95} &
\texttt{lr=0.000182, clip\_range=0.2, vf\_coef=0.5, gae\_lambda=0.95} &
\texttt{lr=0.000365, clip\_range=0.2, vf\_coef=0.5, gae\_lambda=0.8} \\
\addlinespace
\textbf{Trading} &
\texttt{alpha=0.0004, sqrt\_kl=0.02} &
\texttt{alpha=0.04, sqrt\_kl=0.02} &
\texttt{alpha=0.0004, sqrt\_kl=0.01} &
\texttt{lr=0.002, clip\_range=0.2, clip\_range\_vf=0.2, vf\_coef=0.5} &
\texttt{lr=0.001, clip\_range=0.3, clip\_range\_vf=0.3, vf\_coef=0.5} &
\texttt{lr=0.001, clip\_range=0.2, clip\_range\_vf=0.2, vf\_coef=1.0} \\
\bottomrule
\end{tabular}}

\vspace{8pt}
\footnotesize
\textit{Abbreviation:} lr=learning\_rate; alpha=entropy coefficient; clip\_range=PPO clipping range; clip\_range\_vf=value-function clipping; vf\_coef=value loss coefficient; gae\_lambda=GAE $\lambda$; log\_std\_init=initial policy log-std; sqrt\_kl=target $\sqrt{\mathrm{KL}}$ constraint (CPPO).
\end{table*}

% -------------------------
% TRPO + q-learning
% -------------------------
\begin{table*}[t]
\centering
\caption{Hyperparameter dictionary (TRPO vs q-learning). For each environment, \textbf{top/second/third} correspond to the ranked configurations used in Table~\ref{table:ablation_table}.}
\label{table:hp_dict_trpo_q}
\scriptsize
\setlength{\tabcolsep}{4pt}
\renewcommand{\arraystretch}{1.12}
\resizebox{\textwidth}{!}{%
\begin{tabular}{l p{0.15\textwidth} p{0.15\textwidth} p{0.15\textwidth} p{0.15\textwidth} p{0.15\textwidth} p{0.15\textwidth}}
\toprule
\textbf{Task} & \multicolumn{3}{c}{\textbf{TRPO}} & \multicolumn{3}{c}{\textbf{q-learning}} \\
\cmidrule(lr){2-4}\cmidrule(lr){5-7}
& \textbf{top} & \textbf{second} & \textbf{third} & \textbf{top} & \textbf{second} & \textbf{third} \\
\midrule
\textbf{Cheetah} &
\texttt{lr=0.000365, gae\_lambda=0.95, cg\_damping=0.1} &
\texttt{lr=0.000182, gae\_lambda=0.95, cg\_damping=0.1} &
\texttt{lr=0.000365, gae\_lambda=0.8, cg\_damping=0.1} &
\texttt{train\_freq=1, gradient\_steps=1, max\_grad\_norm=0.5} &
\texttt{train\_freq=4, gradient\_steps=4, max\_grad\_norm=0.5} &
\texttt{train\_freq=4, gradient\_steps=4, max\_grad\_norm=1.0} \\
\addlinespace
\textbf{Walker} &
\texttt{lr=0.000365, gae\_lambda=0.95, cg\_damping=0.1} &
\texttt{lr=0.000182, gae\_lambda=0.95, cg\_damping=0.1} &
\texttt{lr=0.000365, gae\_lambda=0.8, cg\_damping=0.1} &
\texttt{log\_std\_init=-3, lr=0.000365} &
\texttt{log\_std\_init=-1, lr=0.000365} &
\texttt{log\_std\_init=-3, lr=0.000182} \\
\addlinespace
\textbf{Humanoid} &
\texttt{lr=0.0006, target\_kl=0.01} &
\texttt{lr=0.0003, target\_kl=0.01} &
\texttt{lr=0.001, target\_kl=0.01} &
\texttt{log\_std\_init=-1, lr=0.000365} &
\texttt{log\_std\_init=-3, lr=0.000365} &
\texttt{log\_std\_init=-1, lr=0.000182} \\
\addlinespace
\textbf{Quadruped} &
\texttt{lr=0.000365, gae\_lambda=0.95, cg\_damping=0.1} &
\texttt{lr=0.000182, gae\_lambda=0.8, cg\_damping=0.1} &
\texttt{lr=0.000365, gae\_lambda=0.8, cg\_damping=0.1} &
\texttt{lr=0.000365, log\_std\_init=-3} &
\texttt{lr=0.00073, log\_std\_init=-3} &
\texttt{lr=0.00146, log\_std\_init=-3} \\
\addlinespace
\textbf{Trading} &
\texttt{lr=0.001, target\_kl=0.005} &
\texttt{lr=0.002, target\_kl=0.01} &
\texttt{lr=0.0005, target\_kl=0.01} &
\texttt{lr=0.004, max\_grad\_norm=1.0, log\_std\_init=-3} &
\texttt{lr=0.002, max\_grad\_norm=0.5, log\_std\_init=-3} &
\texttt{lr=0.001, max\_grad\_norm=1.0, log\_std\_init=-1} \\
\bottomrule
\end{tabular}}

\vspace{8pt}
\footnotesize
\textit{Abbreviation:} lr=learning\_rate; gae\_lambda=GAE $\lambda$; target\_kl=TRPO KL constraint; cg\_damping=conjugate-gradient damping; max\_grad\_norm=gradient clipping threshold; train\_freq=update frequency; gradient\_steps=number of gradient steps per update; log\_std\_init=initial policy log-std.
\end{table*}

\section{Benchmarking algorithms}
\label{sec:appendix:benchmarks}
\subsection{Coupled-style $q$-learning}
As the step size $u\to 0$, the drift contribution scales as $O(u)$ while the stochastic increment is $O(\sqrt{u})$, so a squared TD residual can be dominated by martingale variation rather than the Bellman drift. To address this mismatch, \citet{Jia2022-yg} characterize optimality via a martingale condition: for any start time $t_0$, the process:
\[
M_t
:= e^{-\beta t}V(X_t)
+ \int_{t_0}^t e^{-\beta s}\big(r(X_s,a_s)-q(X_s,a_s)\big)\,ds
\]
must be a martingale. Two practical relaxations follow from this characterization: (i) a \emph{martingale-loss} objective built from squared pathwise residuals (involving an outer integral over $t$ and an inner integral over future $s$), and (ii) \emph{martingale orthogonality} conditions of the form
$\E\!\int_0^T \omega_t\big(dV(X_t)+((r-q)-\beta V)\,dt\big)=0$
for suitable test processes $\omega_t$. To implement both criteria, we consider an (irregular) grid $t_0<t_1<\cdots<t_K=T$ with $\Delta_{t_i}:=t_{i+1}-t_i$, and samples $(x_{t_i},a_{t_i},r_{t_i})$. Define the discounted martingale increment over $[t_k,T]$:
\[
G_{t_k:T}(V,q)
:= e^{-\beta (T-t_k)}h(x_{t_K}) - V(x_{t_k})
   + \sum_{i=k}^{K-1} e^{-\beta(t_i-t_k)}\big[r_{t_i}-q(x_{t_i},a_{t_i})\big]\Delta_{t_i}.
\]
Two commonly used enforcement criteria are:
\begin{enumerate}
\item \textbf{Martingale loss (pathwise squared residual)} (See \cref{algo:coupled_q_martingale_loss}): 
\[
\mathcal L_{\mathrm{mg}}(V,q) := \sum_{k=0}^{K-1} G_{t_k:T}(V,q)^2\,\Delta_{t_k}.
\]
\item \textbf{Martingale orthogonality (test-function TD):} define the one-step residual
\[
\delta_k
:= V(x_{t_{k+1}})-V(x_{t_k})
 + \big(r_{t_k}-q(x_{t_k},a_{t_k})-\beta V(x_{t_k})\big)\Delta_{t_k},
\]
and enforce $\E[\omega(x_{t_k},a_{t_k})\,\delta_k]=0$ for a class of test functions $\omega$ (value gradient in this case), yielding simple SGD-style updates (See \cref{algo:coupled_q_orthogonality}).
\end{enumerate}

\begin{algorithm}[t]
  \caption{Coupled-style $q$-learning (Martingale loss)}
  \label{algo:coupled_q_martingale_loss}
  \begin{algorithmic}[1]
    \STATE \textbf{Input:} discount $\beta$, terminal time $T=t_K$, learning rates $\eta_V,\eta_q$.
    \STATE Initialize parameters $\theta$ for value $V_\theta$ and $\psi$ for advantage-rate $q_\psi$.
    \FOR{$k_{\mathrm{iter}} = 0,1,2,\dots$}
      \STATE Collect rollouts under a behavior policy, producing irregular grid samples
      $\{(x_{t_i},a_{t_i},r_{t_i},\Delta_{t_i})\}_{i=0}^{K-1}$ and terminal payoff $h(x_{t_K})$.
      \STATE Sample a mini-batch of trajectories (or segments) from the dataset/buffer.
      \FOR{each trajectory in the mini-batch}
        \FOR{$k=0,1,\dots,K-1$}
          \STATE Compute
          \[
          G_{t_k:T}(V_\theta,q_\psi)
          :=
          e^{-\beta (T-t_k)}h(x_{t_K}) - V_\theta(x_{t_k})
          + \sum_{i=k}^{K-1} e^{-\beta(t_i-t_k)}\big[r_{t_i}-q_\psi(x_{t_i},a_{t_i})\big]\Delta_{t_i}.
          \]
        \ENDFOR
        \STATE Form loss $\mathcal L_{\mathrm{mg}}(\theta,\psi) := \sum_{k=0}^{K-1} G_{t_k:T}(V_\theta,q_\psi)^2\,\Delta_{t_k}$.
        \STATE Update critics (SGD): $\theta \leftarrow \theta - \eta_V \nabla_\theta \mathcal L_{\mathrm{mg}}(\theta,\psi)$,
        $\psi \leftarrow \psi - \eta_q \nabla_\psi \mathcal L_{\mathrm{mg}}(\theta,\psi)$.
      \ENDFOR
    \ENDFOR
  \end{algorithmic}
\end{algorithm}

\begin{algorithm}[t]
  \caption{Coupled-style $q$-learning (Martingale orthogonality)}
  \label{algo:coupled_q_orthogonality}
  \begin{algorithmic}[1]
    \STATE \textbf{Input:} discount $\beta$, learning rates $\eta_V,\eta_q$.
    \STATE Initialize parameters $\theta$ for value $V_\theta$ and $\psi$ for advantage-rate $q_\psi$.
    \FOR{$k_{\mathrm{iter}} = 0,1,2,\dots$}
      \STATE Collect transitions with irregular increments $(x,a,r,x',\Delta)$ and store in a buffer.
      \STATE Sample a mini-batch $\mathcal B$ of transitions from the buffer.
      \FOR{each $(x,a,r,x',\Delta)$ in $\mathcal B$}
        \STATE Compute one-step residual
        \[
          \delta
          := V_\theta(x') - V_\theta(x)
             + \big(r - q_\psi(x,a) - \beta V_\theta(x)\big)\Delta.
        \]
        \STATE Update value (test function $\omega_1=\nabla_\theta V_\theta$):
        $\theta \leftarrow \theta + \eta_V\, \nabla_\theta V_\theta(x)\,\delta$.
        \STATE Update advantage-rate (test function $\omega_2=\nabla_\psi q_\psi$):
        $\psi \leftarrow \psi + \eta_q\, \nabla_\psi q_\psi(x,a)\,\delta$.
      \ENDFOR
    \ENDFOR
  \end{algorithmic}
\end{algorithm}

\subsection{Policy gradient}

We denote a regularizer by $p(x,a;\pi)$ (for entropy regularization, $p(x,a;\pi)=-\log\pi(a|x)$) and the performance
$\eta(\pi):=\int_{\mathcal S} V(x;\pi)\,\mu(dx)$. Below are two continuous-time policy gradient where CPPO (see \cref{algo:cppo}) is an upgraded version over the (vanilla) CPG (see \cref{algo:cpg}). For CPPO, due to its instability, we use $1.5$ instead of $2$ for penalty adaptation.

\begin{algorithm}[t]
  \caption{Continuous-time policy gradient (CPG)}
  \label{algo:cpg}
  \begin{algorithmic}[1]
    \STATE \textbf{Input:} discount $\beta$, temperature $\alpha$, learning rates $\eta_V,\eta_\pi$.
    \STATE \textbf{Models:} value $V_\phi$, policy $\pi_\theta$.
    \FOR{$k = 0,1,2,\dots$}
      \STATE Collect transitions $(x,a,r,x',\Delta)$ from rollouts under current policy $\pi_\theta$, where $\Delta=t'-t$ can vary.
      \STATE \textbf{Critic update (martingale orthogonality):}
      \[
      \delta_V := V_\phi(x')-V_\phi(x) + \Delta\big(r(x,a)+\alpha\,p(x,a;\pi_\theta)\big) - \beta V_\phi(x)\Delta,
      \]
      and perform SGD: $\phi \leftarrow \phi + \eta_V\,\nabla_\phi V_\phi(x)\,\delta_V$.
      \STATE \textbf{Advantage-rate estimate (one-step / GAE-style):}
      \[
      \hat q(x,a) := r + \frac{e^{-\beta \Delta}V_\phi(x') - V_\phi(x)}{\Delta}.
      \]
      \STATE \textbf{Policy gradient step:} update $\theta$ using a mini-batch estimator of
      \[
      \nabla_\theta \eta(\theta)
      \approx \frac{1}{\beta}\,\nabla_\theta \log \pi_\theta(a|x)\,\big(\hat q(x,a)+\alpha\,p(x,a;\pi_\theta)\big)
      \;+\;\frac{\alpha}{\beta}\,\nabla_\theta p(x,a;\pi_\theta).
      \]
      \STATE $\theta \leftarrow \theta + \eta_\pi\,\widehat{\nabla_\theta \eta(\theta)}$.
    \ENDFOR
  \end{algorithmic}
\end{algorithm}

\begin{algorithm}[t]
  \caption{Continuous-time proximal policy optimization (CPPO, KL-penalty)}
  \label{algo:cppo}
  \begin{algorithmic}[1]
    \STATE \textbf{Input:} discount $\beta$, temperature $\alpha$, target KL $\delta_{\mathrm{KL}}$, tolerance $\epsilon$, initial penalty $C_{\mathrm{pen}}$, actor epochs $E$.
    \STATE \textbf{Models:} value $V_\phi$, policy $\pi_\theta$.
    \FOR{$k = 0,1,2,\dots$}
      \STATE Collect a batch $\mathcal B=\{(x,a,r,x',\Delta)\}$ from rollouts under current policy $\pi_{\theta_k}$.
      \STATE Update critic $V_\phi$ using the CPG martingale-orthogonality step (\cref{algo:cpg}, line 5).
      \STATE Compute advantage-rate estimates $\hat q(x,a)$ on $\mathcal B$ (\cref{algo:cpg}, line 6) and (optionally) normalize.
      \STATE Set $\theta \leftarrow \theta_k$.
      \FOR{epoch $e=1,\dots,E$}
        \STATE For mini-batches from $\mathcal B$, compute importance ratio
        $w_\theta(x,a):=\pi_\theta(a|x)/\pi_{\theta_k}(a|x)$.
        \STATE Define penalized surrogate objective
        \[
        J(\theta)
        := \E_{(x,a)\sim \mathcal B}\big[w_\theta(x,a)\big(\hat q(x,a)+\alpha\,p(x,a;\pi_\theta)\big)\big]
           - C_{\mathrm{pen}}\,\bar D_{\mathrm{KL}}(\theta\|\theta_k),
        \]
        where $\bar D_{\mathrm{KL}}(\theta\|\theta_k):=\E_{x\sim\mathcal B}\!\left[\sqrt{D_{\mathrm{KL}}(\pi_{\theta_k}(\cdot|x)\,\|\,\pi_\theta(\cdot|x))}\right]$.
        \STATE Update actor by gradient ascent on $J(\theta)$.
      \ENDFOR
      \STATE Set $\theta_{k+1}\leftarrow \theta$.
      \STATE \textbf{Penalty adaptation:} if $\bar D_{\mathrm{KL}}(\theta_{k+1}\|\theta_k)\ge (1+\epsilon)\delta_{\mathrm{KL}}$, set $C_{\mathrm{pen}}\leftarrow 2C_{\mathrm{pen}}$;
      else if $\bar D_{\mathrm{KL}}(\theta_{k+1}\|\theta_k)\le \delta_{\mathrm{KL}}/(1+\epsilon)$, set $C_{\mathrm{pen}}\leftarrow \tfrac12 C_{\mathrm{pen}}$.
    \ENDFOR
  \end{algorithmic}
\end{algorithm}

\end{document}